\documentclass{article}
\usepackage[main, final]{neurips_2026}

\usepackage{microtype}
\usepackage{graphicx}
\usepackage{subcaption}
\usepackage{booktabs}
\usepackage{paralist}
\usepackage{minted}
\usepackage{acronym}
\usepackage{algorithmicx}
\usepackage[noend]{algpseudocode}
\usepackage{algorithm}
\usepackage{amsmath}
\usepackage{amssymb}
\usepackage{mathtools}
\usepackage{centernot}
\usepackage{amsthm}
\usepackage[hypertexnames=false]{hyperref}
\usepackage[capitalize,noabbrev]{cleveref}
\usepackage{wrapfig}

\usepackage{thm-restate}
\theoremstyle{plain}

\newtheorem{lemma}{Lemma}[section]

\newtheorem{definition}{Definition}[section]
\newtheorem{example}{Example}[section]

\Crefname{section}{Sec.}{Secs.}
\Crefname{appendix}{App.}{Apps.}
\Crefname{algorithm}{Alg.}{Algs.}
\Crefname{figure}{Fig.}{Figs.}
\Crefname{table}{Tab.}{Tabs.}
\Crefname{definition}{Def.}{Defs.}
\Crefname{theorem}{Thm.}{Thms.}
\Crefname{lemma}{Lem.}{Lems.}
\Crefname{corollary}{Cor.}{Cors.}
\Crefname{example}{Ex.}{Exs.}

\newcommand{\indep}{\perp\kern-6pt\perp}
\newcommand{\dep}{\centernot{\perp\kern-6pt\perp}}
\newcommand{\gap}{\centernot{-}}
\newcommand{\starleft}{*\kern-5pt}
\newcommand{\starright}{\kern-5pt*}
\acrodef{BK}{background knowledge}

\title{Integrating Background Knowledge for Scalable Causal Discovery}

\author{
  Mátyás Schubert \\
  University of Amsterdam \\
  \And
  Theofanis Aslanidis\thanks{Work performed while at University of Amsterdam.} \\
  University of Amsterdam \\
  \And
  Tom Claassen \\
  Radboud University Nijmegen \\
  \And
  Sara Magliacane \\
  Saarland University \\
  University of Amsterdam \\
}

\begin{document}

\maketitle

\begin{abstract}

Expert background knowledge is often available in practical applications of causal discovery.
Such constraints on the true causal graph can help causal discovery in terms of identifiability of causal effects and accuracy of the learned structure, but also in reducing the space of candidate causal graphs.
As causal discovery can become computationally expensive for large number of variables, it is crucial to utilize background knowledge effectively  \emph{during} the causal discovery process.
However, most current methods only use background knowledge in a postprocessing step after causal discovery to refine the learned graph.
In this work, we develop a framework for utilizing background knowledge during the causal discovery process, focusing especially on scalable causal discovery methods that recover only a subset of the whole graph.
We implement our framework for multiple algorithms and empirically show that utilizing background knowledge can both reduce computational requirements and increase the quality of the learned structures.

\end{abstract}

\section{Introduction}
\label{sec:introduction}

Causal graphs are a popular tool to represent the causal data generating process underlying observations \citep{pearl2009causality}.
Causal discovery aims to identify causal relations between causal variables by identifying (parts of) the causal graph from data \citep{glymour2019review}.
These causal relations can then be used for causal effect estimation \citep{perkovic2018complete}, to deepen our understanding of natural systems \citep{runge2019inferring} and for policy-making \citep{maasch2025discrimination}.
However, causal discovery can become computationally demanding when applied to high-dimensional settings, which can limit its applicability in practical problems.

One way to address this issue is to adapt the objective of causal discovery and focus on the downstream task for which the learned causal relations will be used.
For example, if the goal is to estimate the causal effect of a treatment on an outcome, then a large portion of the causal graph may be irrelevant.
Motivated by this, several \emph{scalable} causal discovery methods focus on learning only a set of relevant causal relations for a given causal inference task \citep{watson2022causal, maasch2024local, schubert2025snap}.
A popular approach to achieve this is local causal discovery, which identifies only the local structure of a target variable \citep{wang2014discovering, gupta2023local, schubert2026local}. 
While these methods improve the computational efficiency and hence applicability of causal discovery methods to hundreds and thousands of variables, they rarely allow for the integration of expert knowledge. 

In many real-world settings, experts may have background knowledge on the presence or lack of causal relations between variables, e.g., in the form of direct \citep{perkovic2017interpreting} or ancestral relations \citep{JMLR:v26:23-0624}, or tiers arising in temporal studies \citep{bang2023we,pmlr-v108-andrews20a}.
Background knowledge can also be extracted from experiments \citep{magliacane2016ancestral, NEURIPS2019_5ee56059}, or even large language models \citep{wan2025large, feng2025reliability}.
Such background knowledge can improve the accuracy and identifiability of learned causal relations \citep{9095242, shen2020challenges, foraita2024longitudinal}.
However, with rare exceptions, e.g., tiered background knowledge \citep{bang2023we,pmlr-v108-andrews20a}, this knowledge is usually only applied a posteriori \emph{after} causal discovery, and is not leveraged to reduce the search space and improve the scalability of causal discovery \citep{perkovic2017interpreting,zheng2026local}.
While there exist implementations that integrate limited background knowledge during causal discovery for a specific global method, PC \citep{spirtes2000causation}, we show that the current algorithms in \texttt{pcalg} \citep{kalisch2012pcalg}, \texttt{causal-learn} \citep{zheng2024causal} and \texttt{tetrad} \citep{ramsey2018tetrad} are unsound.
Another exception are logic-based methods, which allow for rich background knowledge in the form of logic rules and its seamless integration in the search process, but they can only scale up to tens of variables \citep{hyttinen2014,JMLR:v16:triantafillou15a,mooij2020jci}.

In this work we develop sound methods to utilize background knowledge \emph{during} causal discovery, to improve not only the accuracy and identifiability of the learned causal relations, but also the scalability of several current methods.
We focus on background knowledge about direct causal relations, i.e., the existence and orientation of specific edges in the true causal graph, and show how it can be used to avoid unnecessary computations.
Although integrating background knowledge sounds superficially intuitive, we show that there are some interesting counterintuitive cases that can lead to unsound conclusions if we integrate it naively.
We implement our insights in global methods PC \citep{spirtes2000causation} and SNAP \citep{schubert2025snap}, making them more scalable. Then, we also integrate them into local methods
MB-by-MB \citep{wang2014discovering}, LDECC \citep{gupta2023local} and LOAD \citep{schubert2026local}, further increasing their computational efficiency.
We prove that our extensions are sound in terms of propagating the background knowledge, and that all, except LDECC, are also complete.
Finally, we show the benefits of integrating background knowledge into the causal discovery process in a thorough empirical study on simulated and realistic data, evaluating both scalability (time and number of CI tests) and accuracy of the estimated causal effects.

\section{Background}
\label{sec:background}

We consider graphs $G = (\mathbf{V}, \mathbf{E})$ with nodes $\mathbf{V}$ and edges $\mathbf{E} \subseteq \mathbf{V} \times \mathbf{V}$.
Edges can be undirected as $X - Y$, or directed as $X \to Y$.
Undirected graphs contain only undirected edges, directed graphs contain only directed edges, and mixed graphs may contain both.
We say that a node is adjacent to another if they are connected by an edge and denote the set of nodes adjacent to $X$ in graph $G$ as $Adj_G(X)$.
In case of an undirected edge $X - Y$, we call the nodes siblings and denote the set of siblings of $X$ as $Sib_G(X)$.
If an edge is directed as $X \to Y$ then we say that $X$ is a parent of $Y$ and $Y$ is a child of $X$, and denote the corresponding sets as $Pa_G(Y)$ and $Ch_G(X)$ respectively.
The Markov blanket $MB_G(X)$ of $X$ consists of its parents, children and the parents of its children.

A path between two nodes $X$ and $Y$ is a sequence of distinct adjacent nodes with one end in $X$ and the other in $Y$.
If all edges on a path between $X$ and $Y$ are directed towards $Y$ then it is a directed path from $X$ to $Y$.
If there is a directed path from $X$ to $Y$ in a graph $G$ then we say that $X$ is an ancestor of $Y$ and $Y$ is a descendant of $X$, and denote the set of ancestors and descendants as $An_G(Y)$ and $De_G(X)$ respectively.
We consider all nodes to be their own ancestors and descendants.
We define adjacencies, siblings, parents, children, ancestors, and descendants on sets of nodes by taking the union, e.g., the nodes adjacent to a set of nodes $\mathbf{S}$ are $Adj_G(\mathbf{S}) = \cup_{X \in \mathbf{S}}Adj_G(X)$.

A directed path from a node to itself is a cycle.
A directed acyclic graph (DAG) has only directed edges and no cycles.
A causal DAG $D$ can be used to represent the data generating process of a distribution $p$ over some causal variables $\mathbf{V}$, where if a variable $X$ has a direct causal effect on another variable $Y$ then $X\to Y$ in $D$.
We focus on observational distributions $p$ and assume that $p$ is Markov and faithful to $D$ \citep{spirtes2000causation}, thus every conditional independence (CI) $X \indep Y | \mathbf{S}$ in $p$ is equivalent to a d-separation $X \perp Y | \mathbf{S}$ in $D$.
Furthermore, we assume no latent confounders or selection bias.

Under these assumptions, constraint-based causal discovery methods utilize CI tests to identify the set of causal graphs, including the true graph $D$, that imply the same set of conditional independences observed in $p$.
This set of graphs is called the Markov equivalence class (MEC) of $D$ can be represented by a unique mixed graph called the \emph{completed partially directed acyclic graph (CPDAG)} \citep{meek1995causal}.
The CPDAG and all DAGs in the MEC share the same adjacencies, and an edge is oriented in the CPDAG if and only if it is oriented the same way in all DAGs of the MEC.
Otherwise, if there exist DAGs in the MEC with differing orientations of an edge, then it is undirected in the CPDAG.

Due to these differing orientations, some causal relations may hold in some DAGs in the MEC and not in others.
Thus, we define \emph{possible} causal relations on CPDAGs.
If $X$ is a parent of $Y$ in at least one DAG in the CPDAG $G$, then we say that $X$ is a possible parent of $Y$ and denote it as $X \in \text{PossPa}_G(Y)$.
We define possible children, ancestors and descendants analogously.
In a CPDAG $G$, $X$ is a possible ancestor of $Y$, i.e., $X \in \text{PossAn}_G(Y)$, iff there exists a \emph{possibly directed path} from $X$ to $Y$, i.e., a path without edges directed towards $X$. 

If there is expert \acf{BK} $\mathcal{B}$ available about the true graph $D$ that is not implied by the observed distribution $p$, then the MEC of $D$ can be further restricted.
In this paper, we focus on \acl{BK} $\mathcal{B}$ that provides information about direct causal relations in the causal DAG $D$.
We consider three types of \ac{BK}:
If $X-Y \in \mathcal{B}$, then $X$ and $Y$ are \textit{adjacent} in $D$.
If $X \to Y \in \mathcal{B}$, then $X$ and $Y$ are adjacent in $D$ and the edge between them is \textit{oriented} from $X$ to $Y$.
If $X \gap Y \in \mathcal{B}$, then there is a \textit{gap} between $X$ and $Y$ in $D$, i.e., they are non-adjacent.
We use $(X,Y) \in \mathcal{B}$ to denote that there is any \ac{BK} about the pair, and $X *\kern-4pt-\kern-4pt* Y \in \mathcal{B}$ to denote that $X-Y, X \to Y$ or $X \gets Y$ is in $\mathcal{B}$.
For all our theoretical results, we assume that $\mathcal{B}$ is \textit{consistent} with the true causal graph $D$ i.e., it never contradicts it, but we also provide an ablation for imperfect knowledge in \Cref{sec:bk_error}.

\citet{perkovic2017interpreting} show how orientation \ac{BK} can be used to obtain a \emph{maximally oriented partially directed acyclic graph (MPDAG)}, that represents exactly the set of DAGs that adhere to both the observed distribution and the \ac{BK} $\mathcal{B}$.
A CPDAG can be restricted into an MPDAG by applying all orientation \ac{BK} on it, and exhaustively orienting all edges using Meek's rules \citep{meek1995causal}. There might still remain undirected edges in the MPDAG.
They also show that possible ancestral relations in MPDAGs appear differently than in CPDAGs and define \emph{b-possible ancestors} in \Cref{def:b-possan}.

\begin{definition}[Def.~3.1 and 3.2 in \citep{perkovic2017interpreting}]\label{def:b-possan}
    $X$ is a b-possible ancestor of $Y$ in MPDAG $G$, denoted as $X \in b$-$PossAn_G(Y)$, iff there is a b-possibly causal path from $X$ to $Y$, i.e., a path $p = \langle X = V_0, \dots,V_k=Y \rangle$ such that exists no edge $V_i \gets V_j$ for $0 \leq i < j \leq k$ in $G$.
\end{definition}

\section{Integrating BK into Causal Discovery}
\label{sec:method}

Most previous work has explored on how to apply \ac{BK} \emph{after} causal discovery to constrain the equivalence class of the true DAG \citep{perkovic2017interpreting} or its local structures \citep{zheng2026local}.
We instead aim to use \acl{BK} also \emph{during} the discovery process to reduce the computational requirements required to identify it, by avoiding unnecessary CI tests.
Our core approach is based on three general principles.

First, if $X *\kern-4pt-\kern-4pt* Y \in \mathcal{B}$, then we know that $X$ and $Y$ are adjacent in the true DAG $D$ and so they cannot be separated by any CI test.
Thus, we can skip all CI tests for these pairs of variables.
This is already implemented in several causal discovery libraries, such as \texttt{pcalg} \citep{kalisch2012pcalg} and \texttt{causal-learn} \citep{zheng2024causal}.
Second, we can use $\mathcal{B}$ to limit the search space of candidate separating sets.
As shown by Lemma 3.3.9 in \cite{spirtes2000causation}, if $X$ and $Y$ are not adjacent, then they can be d-separated by the parents of $X$ or $Y$ in the underlying causal graph.
Hence, when trying to separate $X$ from $Y$, we can prune the candidate separating sets by considering only nodes for which $\mathcal{B}$ does not contain information that would disqualify them as possible parents of $X$.
A similar strategy is implemented in \texttt{tetrad} \citep{ramsey2018tetrad}.

The third principle considers integrating \ac{BK} on gaps $X \gap Y \in \mathcal{B}$.
Similarly to adjacencies, gaps inform us about non-adjacency without performing any CI tests.
Intuition suggests that if the existence of an edge is already determined by a gap in the \ac{BK}, then it is unnecessary to test for it.
Thus, several popular implementations of PC \citep{spirtes2000causation} simply remove all edges corresponding to known gaps at the beginning of the discovery process \citep{kalisch2012pcalg, ramsey2018tetrad, zheng2024causal, ankan2024pgmpy}.
However, this intuition is wrong, as naively removing edges purely due to \ac{BK} without verifying it with a CI tests can lead to incorrect outputs even with consistent \ac{BK} and oracle CI tests, as we show in \Cref{sec:unsound}.
In particular, a gap known in advance by \ac{BK} still needs a corresponding separating set to be recorded for correct orientations.
We address this issue by delaying, instead of fully skipping, CI tests for known gaps.
This delay generally allows us to find separating sets from a smaller set of candidate sets due to adjacencies by the end of discovery process.
Our strategy of handling gaps is not only sound, but it can also automatically deal with cases when 
the \ac{BK} and CI test results contradict each other, which could potentially lead to a missing separating set and hence propagate the error to other parts of the graph.

We first show these three principles in the seminal PC algorithm \citep{spirtes2000causation}.
Then, we build on our algorithms to integrate \ac{BK} into other methods including local methods focusing on scalability.

\subsection{PC-BK}
\label{sec:pc}

The PC algorithm \citep{spirtes2000causation} is a global causal discovery algorithm that aims to identify the full CPDAG of the true causal DAG $D$.
The PC algorithm initializes a fully connected and undirected skeleton and refines it in three phases.
The skeleton phase iterates over $i \in 0 \dots |\mathbf{V}|-2$ orders, where each step at order $i$ aims to separate variables using separating sets of size $i$.
Then, based on the final skeleton and collected separating sets, PC orients v-structures as shown in \Cref{alg:pc-v}.
Finally, it orients as many edges as possible using Meek's orientations rules in \Cref{alg:meek}.
\begin{wrapfigure}{r}{0.5\textwidth}
\centering
\includegraphics[width=\linewidth]{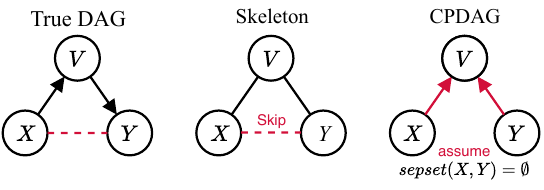}
\caption{
Skipping a CI test between $X$ and $Y$ due to a known gap $\mathcal{B} = X \gap Y$ and assuming that the separating set between $X$ and $Y$ is empty, leads to erroneously orienting an unshielded triple $X-V-Y$ as a v-structure in the CPDAG.
}
\label{fig:unsound_pc_example}
\end{wrapfigure}

While there exist PC implementations that integrate limited background knowledge during causal discovery, which we discuss in detail in \Cref{sec:unsound}, the current algorithms in \texttt{pcalg} \citep{kalisch2012pcalg}, \texttt{causal-learn} \citep{zheng2024causal} and \texttt{tetrad} \citep{ramsey2018tetrad} are unsound, while \texttt{pgmpy} \citep{ankan2024pgmpy} crashes due to conflicting information. In particular, using gap BK before causal discovery can lead to skipping CI tests, which can potentially result in incorrect orientations. This can result in assuming the wrong separating sets, as done in \texttt{pcalg} \citep{kalisch2012pcalg} or \texttt{causal-learn} \citep{zheng2024causal}, and therefore orienting an incorrect v-structure as we show in \Cref{fig:unsound_pc_example}. In this specific example, \texttt{tetrad} \citep{ramsey2018tetrad} would correctly not orient this v-structure, because if a separating set is not defined (e.g. due to skipping a CI test), it is not used in orienting v-structures. On the other hand, if the true DAG is a v-structure $X \to V \leftarrow Y$, \texttt{tetrad} would also not orient this unshielded triple, leading to a wrong CPDAG.
We provide additional details, code examples to reproduce incorrect outputs, and empirical comparisons against these methods in \Cref{sec:unsound}.
Our empirical results show that a higher rate of \ac{BK} only increases the graph estimation accuracy of our approach, while it actually decreases it for the algorithms provided by these libraries on most settings, including oracle d-separation CI tests.

Besides the unsoundness issues, current PC implementations also do not support all types of direct background knowledge we consider, i.e., gaps, adjacencies and orientations. To address these issues, we propose PC-BK, described in \Cref{alg:pc_bk}.
Unlike other implementations of PC, we do not remove edges that are known gaps from the skeleton immediately at the beginning of the algorithm.
Instead, PC-BK performs the skeleton discovery in two passes at lines \ref{line:pc:skeleton} and \ref{line:pc:find_missing_sepsets}.
The first pass, as implemented in \Cref{alg:skeleton_bk}, performs the same amount of CI tests as if the gaps were removed from the skeleton in the beginning.
It skips all pairs of variables with available \ac{BK}, including known gaps, at line~\ref{line:skel:skip}.
For every other pair, we limit the separating sets at line~\ref{line:skel:limit} to contain only variables that are still possible parents according to the current skeleton $\hat{G}$ and the given \ac{BK} $\mathcal{B}$ by defining the function $\text{PossPa}_{\hat{G}}(X, \mathcal{B}) = Adj_{\hat{G}}(X)\setminus \{V \in \mathbf{V} | X \to V \in \mathcal{B} \vee X \gap V \in \mathcal{B}\}$.

\begin{figure}[t]
\centering
\begin{minipage}{0.48\textwidth}
    \begin{algorithm}[H]
    \caption{SkelStep-BK}
    \footnotesize{
    \label{alg:skeleton_bk}
    \begin{algorithmic}[1]
        \Require{skeleton $\hat{G}$, order $i$ and \ac{BK} $\mathcal{B}$}
        \For{$X \in \hat{G}, Y \in Adj_{\hat{G}}(X)$ s.t. $(X,Y)  \notin \mathcal{B}$}\label{line:skel:skip}   
            \For{$\mathbf{S} \subseteq \text{PossPa}_{\hat{G}}(X, \mathcal{B}) \setminus \{Y\} $ s.t. $|\mathbf{S}| = i$}\label{line:skel:limit} 
                \If{$X \indep Y | \mathbf{S}$}
                    \State{Delete $X - Y$ from $\hat{G}$}
                    \State{$sepset(\{X,Y\}) \gets \mathbf{S}$}
                    \State{\textbf{break}}
                \EndIf
            \EndFor
        \EndFor
        \Ensure{Skeleton $\hat{G}$, separating sets $sepset$}
    \end{algorithmic}
    }
    \end{algorithm} 
    \vspace{-7mm}
    \begin{algorithm}[H]
    \caption{MissingSepsets-BK}
    \footnotesize{
    \label{alg:missing_sepsets}
    \begin{algorithmic}[1]
        \Require{Skeleton $\hat{G}$, order $i$ and \ac{BK} $\mathcal{B}$}
        \For{$(X \gap Y)  \in \mathcal{B}$}
            \For{$\mathbf{S} \subseteq \text{PossPa}_{\hat{G}}(X, \mathcal{B}) \setminus \{Y\}$ s.t. $|\mathbf{S}| = i$}\label{line:pc:limit_cand_sepsets_find_missing}
                    \If{$X \indep Y | \mathbf{S}$}
                        \State{Delete edge $X - Y$ from $\hat{G}$}
                        \State{$sepset(\{X,Y\}) \gets \mathbf{S}$}
                        \State{\textbf{break}}
                    \EndIf
                \EndFor
        \EndFor
        \Ensure{Skeleton $\hat{G}$, separating sets $sepset$}
    \end{algorithmic}
    }
    \end{algorithm}
\end{minipage}
\hfill
\begin{minipage}{0.48\textwidth}
\begin{algorithm}[H]
\caption{PC-BK}
\footnotesize{
\label{alg:pc_bk}
\begin{algorithmic}[1]
    \Require{Variables $\mathbf{V}$ and \ac{BK} $\mathcal{B}$}
    \State{$\hat{G} \gets$ fully connected undirected graph over $\mathbf{V}$}
    \State{$sepset \gets \emptyset$}
    \For{$i \in 0 \dots |\mathbf{V}|-2$}\label{line:pc:skel_start}
        \State{$\hat{G}, sepset \gets \texttt{SkelStep-BK}(\hat{G}, i, \mathcal{B})$}\label{line:pc:skeleton}
    \EndFor\label{line:pc:skel_end}
    
    \For{$i \in 0 \dots |\mathbf{V}|-2$}
        \State{$\hat{G}, sepset \gets \texttt{MissingSepsets-BK}(\hat{G}, i, \mathcal{B})$ }\label{line:pc:find_missing_sepsets}
    \EndFor\label{line:pc:missing_end}
    
    \State{$\hat{G} \gets \texttt{VstructPC}(\hat{G}, sepset)$ (\Cref{alg:pc-v})}\label{line:pc:vstruc}
    \State{$\hat{G} \gets$ Apply all orientations in $\mathcal{B}$}
    \State{$\hat{G} \gets \texttt{Meek}(\hat{G})$ (\Cref{alg:meek})}\label{line:pc:post_processing}
    \Ensure{MPDAG $\hat{G}$}
\end{algorithmic}
}
\end{algorithm}
\end{minipage}
\end{figure}

As discussed earlier, simply removing edges corresponding to known gaps can lead to unsoundness, if the corresponding separating sets are undefined.
Thus, we perform a second pass in the skeleton search specifically to find separating sets corresponding to all known gaps in $\mathcal{B}$ in \Cref{alg:missing_sepsets}.
Delaying this search after the first pass allows us to find separating sets for gaps from a smaller set of candidate sets due to a sparser skeleton and hence fewer possible parents.
Once the missing separating sets are identified, the v-structures can be oriented as usual.
If we cannot find any separating set for $X \gap Y \in \mathcal{B}$, possibly due to noisy finite sample data, we keep the edge between $X$ and $Y$ in the skeleton.
This ensures that even if $\mathcal{B}$ is inconsistent with the data, our algorithm does not fail.
Finally, orientations in $\mathcal{B}$ can be utilized as a post-processing step via orientation rules (\Cref{alg:meek}) as in \citep{meek1995causal, perkovic2017interpreting} at line~\ref{line:pc:post_processing}.
We show that \Cref{alg:pc_bk} is sound and complete in \Cref{thm:pc_bk} and prove it in \Cref{sec:proofs:pc}.

\begin{restatable}[]{theorem}{pcbk}
\label{thm:pc_bk}
Given oracle CI tests and consistent \ac{BK}, PC-BK outputs the true MPDAG.
\end{restatable}

We show how the same principles and subroutines of PC-BK can be integrated into other causal discovery methods, with algorithm-specific adaptations as needed.

\subsection{SNAP-BK}
\label{sec:snap}

\begin{figure}[t]
\centering
\begin{minipage}{0.48\textwidth}
\begin{algorithm}[H]
\caption{SNAP($k$)-BK}
\footnotesize{
\label{alg:snap_k_bk}
\begin{algorithmic}[1]
    \Require{ variables $\mathbf{V}$, targets $\mathbf{T} \subseteq \mathbf{V}$, max order $k$, \ac{BK} $\mathcal{B}$}
    \State{$\hat{U} \gets$ fully connected undirected graph over $\hat{\mathbf{V}} = \mathbf{V}$}
    \State{$seps \gets \emptyset$}
    \For{$i \in 0 \dots k$}
        \State{$\hat{U} \gets$ induced subgraph of $\hat{U}$ over all $\hat{\mathbf{V}}$}
        \State{$\hat{U}, seps \gets \texttt{SkelStep-BK}(\hat{U}, i, \mathcal{B})$ (\Cref{alg:skeleton_bk})}\label{line:snap:skelstep}
        \State{$\hat{U}, seps \kern-1pt\gets\kern-2pt \texttt{MissingSepsets-BK}(\hat{U}, i, \mathcal{B})$ (\Cref{alg:missing_sepsets})}\label{line:snap:missing}
        \If{$i = 0$}
            \State{$\hat{G} \gets \texttt{VstructPC}(\hat{U}, sepset)$ (\Cref{alg:pc-v})}\label{line:snap:pc_v}
        \Else
            \State{$\hat{G}, seps \gets \texttt{VstructRFCI}(\hat{U}, seps)$ (\Cref{alg:rfci-v})}\label{line:snap:rfci_v}
            \State{$\hat{U} \gets$ skeleton of $\hat{G}$}
        \EndIf
        \State{$\hat{G} \gets$ Apply all orientations in $\mathcal{B}$}\label{line:snap:bk_in_loop}
        \State{$\hat{\mathbf{V}} \gets$ $V \in \hat{\mathbf{V}}$ with poss. dir. paths to $T \in \mathbf{T}$ in $\hat{G}$}\label{line:snap:prune}
    \EndFor\label{line:snap:loop_end}
    \State{$\hat{G} \gets$ induced sub-graph of $\hat{G}$ over $\hat{\mathbf{V}}$}
    \Ensure{Remaining variables $\hat{\mathbf{V}}$, graph $\hat{G}$}
\end{algorithmic}
}
\end{algorithm}
\vspace{-8mm}
\begin{algorithm}[H]
\caption{SNAP($\infty$)-BK}
\footnotesize{
\label{alg:snap_inf_bk}
\begin{algorithmic}[1]
    \Require{variables $\mathbf{V}$, targets $\mathbf{T} \subseteq \mathbf{V}$, \ac{BK} $\mathcal{B}$}
    \State{$\hat{\mathbf{V}}, \hat{G} \gets \texttt{SNAP($k$)-BK}(\mathbf{V}, \mathbf{T}, |\mathbf{V} - 2|, \mathcal{B})$} (\Cref{alg:snap_k_bk})
    \State{$\hat{G} \gets \texttt{Meek}(\hat{G})$ (\Cref{alg:meek})}
    \State{$\hat{G} \gets$ induced subgraph of $\hat{G}$ over all $V \in \mathbf{V}$ with a b-possibly directed path to any $T \in \mathbf{T}$ in $\hat{G}$}\label{line:snap:final_prune}
    \Ensure{MPDAG $\hat{G}$}
\end{algorithmic}
}
\end{algorithm}
\end{minipage}
\hfill
\begin{minipage}{0.48\textwidth}
\begin{algorithm}[H]
\caption{MB-by-MB-BK}
\label{alg:mb_by_mb_bk}
\footnotesize{
\begin{algorithmic}[1]
    \Require{Variables $\mathbf{V}$, target $T \in \mathbf{V}$, \ac{BK} $\mathcal{B}$}
    \State $sepset, MB, L, \text{DoneList} \gets \emptyset, \emptyset, \emptyset, \emptyset$
    \State $\text{WaitList} = [T]$
    \State $\hat{G} = (\mathbf{V}, \emptyset)$
    \Repeat
        \State $X \gets$ take a node from the head of WaitList
        \State $MB(X) \gets$ GrowShrink-BK($\mathbf{V}, X, \mathcal{B}$) (\Cref{alg:gs_bk})\label{line:mb_by_mb:gs}
        \State WaitList $\gets$ Add $[MB(X) \setminus \text{DoneList} \setminus \text{WaitList}]$
        \State DoneList $\gets \text{DoneList} \cup X$

        \If{$MB^+(X) \subseteq MB^+(X')$ for some $X' \in \text{DoneList}$}
            \State Set $L_X$ to the subgraph of $L_{X'}$ over $MB^+(X)$
        \ElsIf{$MB(X) \subseteq \text{DoneList}$}
            \State Set $L_X$ to the subgraph of $G$ over $MB^+(X)$
        \Else
             \State{$L_X \gets$ PC-BK($MB(X)^+, \mathcal{B}$) (with $PossPa_{\hat{G}}^*(\cdot, \mathcal{B})$)}\label{line:mb_by_mb:pc}
        \EndIf
        \State{$G \gets$ Add edges and v-structures incl. $X$ in $L_X$}\label{line:mb_by_mb:update}
        \State{$\hat{G} \gets$ Apply all orientations in $\mathcal{B}$}\label{line:mb_by_mb:orient}
        \State{$\hat{G} \gets \texttt{OrientUndirectedEdges}(\hat{G})$ (\Cref{alg:meek_mb_by_mb})}
        \State{WaitList $\gets$ Remove all nodes whose paths to $T$ in $G$ are blocked by directed edges}
    \Until{WaitList is empty}
    \Ensure{$Pa_{\hat{G}}(T)$, $Ch_{\hat{G}}(T)$, $Sib_{\hat{G}}(T)$}
\end{algorithmic}
}
\end{algorithm}
\end{minipage}
\end{figure}

\citet{schubert2025snap} show that for a set of targets $\mathbf{T}$, only their possible ancestors are required to identify their causal relations and efficient adjustment sets.
They propose the SNAP($k$) algorithm to prune all non-ancestors identified over $k$ iterations, and extend it to SNAP($\infty$), a stand-alone causal discovery algorithm that is sound and complete over the possible ancestors of targets.

We implement SNAP($k$)-BK and SNAP($\infty$)-BK in \Cref{alg:snap_k_bk} and \Cref{alg:snap_inf_bk} respectively.
SNAP($k$)-BK employs a similar skeleton search as PC-BK, but iteratively orients v-structures and prunes the identified non-ancestors of the targets at every step.
Hence, the two subroutines \Cref{alg:skeleton_bk} and \Cref{alg:missing_sepsets} can be easily integrated into SNAP($k$)-BK at every order at lines \ref{line:snap:skelstep} and \ref{line:snap:missing}.
However, limiting candidate separating sets and orienting v-structures at early steps can lead to SNAP($k$) incorrectly pruning ancestors of the targets, as shown in \Cref{example:snap_bk_vstruct_counter}.
We solve this by orienting v-structures using the RFCI orientation rules \citep{colombo2012learning} (\Cref{alg:rfci-v}) already at order $i=1$ instead of at $i=2$ as in the original paper.
To use \ac{BK} even further, we apply all orientations in $\mathcal{B}$ at line~\ref{line:snap:bk_in_loop} on the graph $\hat{G}$ at every iteration.

After this, the original SNAP($k$) algorithm prunes variables at each iteration that are not part of the possibly ancestral set containing the targets.
Analogously, SNAP($k$)-BK prunes variables that are not in the \emph{b-possibly ancestral set} $\mathbf{V}^* \subseteq \mathbf{V}$, which is defined as any set closed under b-possible ancestry, i.e., such that $\text{b-PossAn}_G(\mathbf{V}^*) \subseteq \mathbf{V}^*$. 
In \Cref{lem:possdirpath} we show that if a node $V$ is in a b-possibly ancestral set, then all other nodes with a possibly directed path to $V$ are also in the set.
Hence we keep the pruning step at line~\ref{line:snap:prune} the same as in the original paper.

We provide two theoretical results, showing that our extensions integrate \ac{BK} correctly and efficiently into the two SNAP methods. All proofs for SNAP are in \cref{sec:proofs:snap}.
We first show that the remaining nodes at every iteration of SNAP$(k)$ form a b-possibly ancestral set that contains $\mathbf{T}$:

\begin{restatable}[]{theorem}{snapkbk}
\label{thm:snap_k_bk}
Given oracle CI tests and consistent \ac{BK}, at each step $i=0, \dots, k$ of SNAP$(k)$-BK for the remaining variables $\hat{\mathbf{V}}$ that are output by SNAP, it holds that $\mathbf{T} \subseteq \hat{\mathbf{V}}$ and $\text{b-PossAn}_G(\hat{\mathbf{V}}) \subseteq \hat{\mathbf{V}}$.
\end{restatable}

Consequently, SNAP($\infty$)-BK is sound and complete over the b-possible ancestors of the targets.

\begin{restatable}[]{corollary}{snapinfbk}
\label{cor:snap_inf_bk}
Given true MPDAG $G$, oracle CI tests and consistent \ac{BK}, let $\hat{G}$ be the output graph of SNAP$(\infty)$-BK for targets $\mathbf{T}$. Then, SNAP$(\infty)$-BK is sound and complete over the b-possible ancestors $\mathbf{T}$, i.e.
$
\hat{G} = G|_{\text{b-PossAn}(\mathbf{T})}, 
$
and $\hat{G}$ is an MPDAG.
\end{restatable}

We now turn to local methods that identify and orient edges around target variables without discovering the whole CPDAG.
To find the variables adjacent to a target $T$, many local methods begin by finding its Markov blanket $MB(T)$ using existing algorithms as a sub-routine.
We show how to utilize \ac{BK} for Markov blanket discovery in \Cref{sec:mb_discovery}, where we provide two extensions for GrowShrink \citep{margaritis1999bayesian} and Total Conditioning \citep{pellet2008using} that we prove are sound and complete.   

As local methods aim to orient edges as soon as possible by exploiting independence patterns, we find that they generally do benefit from searching for separating sets corresponding to known gaps without delay.
Thus, when applying \Cref{alg:skeleton_bk} in local methods, we swap the condition $(X,Y) \notin \mathcal{B}$ at line~\ref{line:skel:skip}, which includes known gaps, with $(X *\kern-4pt-\kern-4pt* Y) \notin \mathcal{B}$, which only considers known adjacencies.

\subsection{MB-by-MB-BK}
\label{sec:mb_by_mb}

\begin{figure}[t]
\begin{subfigure}{.32\linewidth}
    \centering
    \includegraphics[width=.8\linewidth]{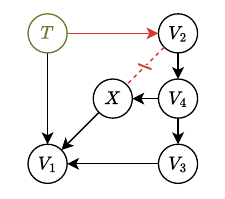}
    \caption{True DAG/MPDAG.}
    \label{fig:mb_by_mb_pc_dag}
\end{subfigure}
\begin{subfigure}{.32\linewidth}
    \centering
    \includegraphics[width=.8\linewidth]{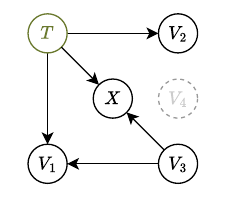}
    \caption{Naive MB-by-MB-BK.}
    \label{fig:mb_by_mb_naive}
\end{subfigure}
\begin{subfigure}{.32\linewidth}
    \centering
    \includegraphics[width=.8\linewidth]{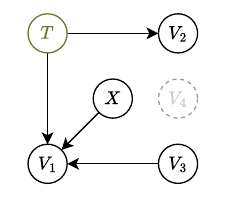}
    \caption{Our MB-by-MB-BK.}
    \label{fig:mb_by_mb_correct}
\end{subfigure}
\caption{On the DAG/MPDAG shown in \Cref{fig:mb_by_mb_pc_dag} with $\mathcal{B} = \{T \to V_2, V_2 \gap X\}$ denoted in red, a naively implemented MB-by-MB-BK incorrectly identifies $X$ as a child of target $T$ (\Cref{fig:mb_by_mb_naive}).
Our method recovers the correct relations (\Cref{fig:mb_by_mb_correct}).
Details are explained in \Cref{example:mb_by_mb}.
}
\label{fig:mb_by_mb_bk_example}
\end{figure}

MB-by-MB \citep{wang2014discovering} is a local causal discovery method that iteratively discovers the Markov blankets $MB(X)$ and local structures around variables $X$, starting from the target $X=T$, and combines these local structures to identify the parents, children and siblings of $T$ without discovering the global structure.
MB-by-MB can directly use one of \emph{our extended versions of MB discovery} described in \cref{sec:mb_discovery}, e.g., Grow-Shrink-BK in \Cref{alg:gs_bk} or Total Conditioning-BK in \Cref{alg:tc_bk}.
However, naively applying PC-BK instead of the original PC over the learned Markov blankets can lead to erroneous results on the local structures, as shown in
\Cref{fig:mb_by_mb_bk_example}. In this example, restricting the candidate separating sets only to possible parents prevents us from using $V_2$ to separate target $T$ and node $X$ even though it is the only available separator in $MB(T)$.
In general, this occurs when discovering the local structure of $MB(X)$, but none of the possible parents of a node $V \neq X$ are in $MB(X)$.
Thus, we run PC-BK to discover the local structure at line~\ref{line:mb_by_mb:pc} with an adjusted definition of possible parents that only uses \ac{BK} around $X$, but ignores it for every other $X \neq V \in MB(X)$, as
\begin{equation}\label{eq:mb_by_mb_posspa}
PossPa_{\hat{G}}^*(V, \mathcal{B}) =
\begin{cases}
    \text{PossPa}_{\hat{G}}(V, \mathcal{B}) & \text{if}\ V=X\\
    Adj_{\hat{G}}(V) & \text{otherwise}
\end{cases}.
\end{equation}

Similarly to SNAP($k$)-BK, MB-by-MB-BK can also apply orientations in $\mathcal{B}$ at every iteration to terminate quicker. We implement MB-by-MB-BK in \Cref{alg:mb_by_mb_bk} and show that it is sound and complete in \Cref{thm:mb_by_mb_bk} with proof in \cref{sec:proofs:mb_by_mb}. 

\begin{restatable}[]{theorem}{mbbybbk}
\label{thm:mb_by_mb_bk}
Given oracle CI tests and consistent \ac{BK}, MB-by-MB-BK is sound and complete in identifying the parents, children and siblings of the target.
\end{restatable}

\subsection{LDECC-BK}
\label{sec:ldecc}
LDECC \citep{gupta2023local} is another local causal discovery algorithm with the same objective as MB-by-MB.
We implement LDECC-BK in \Cref{alg:ldecc_bk}.
It begins by finding the Markov blanket of the target $T$ and then identifies its neighbors which are not already known, as shown in \Cref{alg:localpc}.
Then, it identifies some of the children and parents of the target using the Markov blanket and the provided \ac{BK}.

After this, the algorithm iterates through the same CI tests that \Cref{alg:skeleton_bk} would perform, but tries to orient edges early using \textit{eager collider checks}.
This means that the same way as earlier \ac{BK} algorithms, LDECC-BK can skip any CI test for pairs known to be adjacent by $\mathcal{B}$, and restrict candidate separating sets to subsets of possible parents.
The rest of the algorithm is identical to lines 7 to 22 of the original algorithm, where LDECC identifies the parents, children and siblings of the target.
However, by utilizing orientation \ac{BK} during discovery to identify parents at line~\ref{line:ldecc:pa}, LDECC-BK may identify additional children of the target through non-collider checks at line~19 of the original algorithm, that simple post-processing would not find.

LDECC-BK relies purely on CI relations to orient edges instead of propagating orientation rules.
Thus, as shown in \Cref{example:ldecc_not_complete}, it is unable to utilize orientation \ac{BK} that is outside of the Markov blanket of the target and not implied by the data.
This can lead to parents or children getting misclassified as siblings.
We found no way to propagate orientations without significantly changing the algorithm. 
Thus, we only show the soundness of LDECC-BK in \Cref{cor:ldecc_bk}, which follows from the soundness and completeness of Grow-Shrink-BK in \Cref{thm:gs_bk} and Thm.~5 of \citep{gupta2023local}.

\begin{restatable}[]{corollary}{ldeccbk}
\label{cor:ldecc_bk}
Given oracle CI tests and consistent \ac{BK}, LDECC-BK is sound and complete in identifying adjacencies and sound in identifying parents and children of the target.
\end{restatable}

\begin{figure}[t]
\centering
\begin{minipage}{0.5\textwidth}
\begin{algorithm}[H]
\caption{LDECC-BK}
\label{alg:ldecc_bk}
\footnotesize{
\begin{algorithmic}[1]
    \Require{Variables $\mathbf{V}$, target $T \in \mathbf{V}$, \ac{BK} $\mathcal{B}$}
    \State{$MB \gets \texttt{Grow-Shrink-BK}(\mathbf{V}, T, \mathcal{B})$} (\Cref{alg:gs_bk})\label{line:ldecc:gs}
    \State{$Ne, sepset \gets \texttt{LocalPC-BK}(T, MB, \mathcal{B})$ (\Cref{alg:localpc})}\label{line:ldecc:localpc}
    \State{$Ch \gets \texttt{UCChildren}(MB, Ne, sepset)$ (Fig.~4 of \citep{gupta2023local})}\label{line:ldecc:uccchildren}
    \State{$Ch \gets Ch \cup \{V \in \mathbf{V} | T \to V \in \mathcal{B}\}$}
    \State{$Pa\gets\{V \in \mathbf{V} | V \to T \in \mathcal{B}\}$}\label{line:ldecc:pa}
    \State{$U \gets$ fully connected undirected graph over $\mathbf{V}$}
    \For{$A \indep B | \mathbf{S} \in \texttt{SkelStep-BK}(U,0..,\mathcal{B})$, $A,B \neq T$}\label{line:ldecc:fullpc}
        \State{Identical to lines 7 to 22 of Fig.~5 of \citep{gupta2023local})}
    \EndFor
    \State{$Sib \gets Ne \setminus(Pa \cup Ch)$}
    \Ensure{$Pa$, $Ch$, $Sib$}
\end{algorithmic}
}
\end{algorithm}
\end{minipage}
\hfill
\begin{minipage}{0.47\textwidth}
\begin{algorithm}[H]
\caption{LOAD-BK}
\label{alg:load_bk}
\footnotesize{
\begin{algorithmic}[1]
    \Require{Variables $\mathbf{V}$, targets $X,Y \in \mathbf{V}$, \ac{BK} $\mathcal{B}$}
    \If{$X \to Y \in \mathcal{B}$}
        \State{$T,O \gets X,Y$}
    \ElsIf{$Y \to X \in \mathcal{B}$}
        \State{$T,O \gets Y,X$}
    \Else
        \State{Step 1 of LOAD (Alg.~3. in \citep{schubert2026local})}
    \EndIf
    \State{Steps 2 and 3 of LOAD (Alg.~3. in \citep{schubert2026local})}
    \State{Steps 4 of LOAD with \ac{BK} on adjacencies and orientations over nodes not projected out (Alg.~3. in \citep{schubert2026local})}
    \Ensure{Causal relation, identifiability, adjustment sets}
\end{algorithmic}
}
\end{algorithm}

\end{minipage}
\end{figure}

\subsection{LOAD-BK}
LOAD \citep{schubert2026local} uses local causal discovery algorithms, such as MB-by-MB, to first discover the causal relations of a pair of targets, and then the optimal adjustment set in a modified forbidden projection, to estimate the causal effect between them.
Thus, LOAD-BK can utilize our framework indirectly through MB-by-MB-BK.
However, \ac{BK} on gaps might not hold anymore under the modified forbidden projection.
Thus, we utilize \ac{BK} only on adjacencies and orientations in the last step of LOAD-BK.
Furthermore, given a target pair $X$ and $Y$, if $X \to Y$ or $Y \to X$ is in $\mathcal{B}$, then LOAD-BK can skip its first step that aims to determine the causal relation between the targets, as shown in \Cref{alg:load_bk}.
We show that LOAD-BK is sound and complete \Cref{thm:load_bk} and provide proof in \Cref{sec:proofs:load}, for which we extend the modified forbidden projection of \citet{schubert2026local} to MPDAGs.

\begin{restatable}[]{theorem}{loadbk}
\label{thm:load_bk}
Given oracle CI tests and consistent \ac{BK}, LOAD-BK is sound and complete in finding optimal adjustment set for a pair of target variables.
\end{restatable}

\section{Related Work}
\label{sec:related_work}
We focus on \ac{BK} about direct causal relations, i.e., whether an edge exists and how it is oriented in the true causal DAG.
\citet{meek1995causal} has done pioneering work on orientation \ac{BK}, and developed the orientation rules also utilized in this work, shown in \Cref{alg:meek}, where the fourth rule is specifically designed for extra orientations arising from \ac{BK}.
This theory was further developed by \citet{perkovic2017interpreting} who show how to maximally orient edges and obtain a MPDAG from an available CPDAG.

\citet{JMLR:v26:23-0624} consider background knowledge on the direction of causal relations such as orientations, and also ancestral and non-ancestral relations, but assume an already given CPDAG.
\citet{zheng2026local} explore how to use \ac{BK} on orientations, as well as (non-) ancestral relations to refine the local structures discovered by the MB-by-MB algorithm, which may allow MB-by-MB to terminate earlier, but they do not explicitly avoid unnecessary CI tests like \Cref{alg:mb_by_mb_bk}.
We discuss how ancestral \ac{BK} can be utilized by our approach in \Cref{sec:ancestral_bk}.

A special type of non-ancestral \ac{BK} is given by ordered tiers of variables arising in, e.g., longitudinal data, which can be represented by tiered MPDAGs \citep{bang2023we}.
Tiered \ac{BK} is quite different from the pairwise \ac{BK} that we consider, as it provides knowledge on (non-)ancestral relations, rather than direct ones.
\citet{pmlr-v108-andrews20a} integrate tiered BK during causal discovery with latent variables by processing tiers one-by-one, discovering only the edges within each tier and into it from earlier tiers.
\citet{bang2024improving} develop the tiered PC (tPC) algorithm to discover tiered MPDAGs and exploit tiered \ac{BK} during the discovery process by excluding conditioning sets containing future-tier variables, which stems from the same insight by which we exclude all variables that are not possible parents.
Tiered \ac{BK} can be also utilized by causal discovery algorithms allowing for latent variables and multiple datasets \citep{pmlr-v108-andrews20a, bang2025constraint} as well as in score-based causal discovery \citep{larsen2025score}.

In the causally insufficient case, \citet{WANG2023103964} consider local \ac{BK}, which fully orients all edges adjacent to any node, and introduce additional orientation rules.
\citet{venkateswaran2024towards} extend this to general orientation \ac{BK}, and provide a generalization for the orientation rules of \citet{meek1995causal} for latent variables.
\citet{wang2025polynomialdelay} uses \ac{BK} to efficiently enumerate all graphs in the equivalence class.
These works utilize \ac{BK} after causal discovery instead of during the process.

Logic-based causal discovery methods, e.g., \citep{claassen2011logical,hyttinen2014,JMLR:v16:triantafillou15a,magliacane2016ancestral,mooij2020jci}, can seamlessly integrate rich background knowledge in the form of logical rules or constraints on causal edges or paths in the causal discovery process, including in challenging settings like latent confounding, selection bias, cycles or data from multiple contexts. In particular, \citet{hyttinen2014, magliacane2016ancestral, mooij2020jci} formulate causal discovery as a constraint optimization problem, where weights can be added to each constraint, which allows to model noisy or imperfect knowledge.
Since these methods are based on constraints, they can leverage \ac{BK} during the discovery process to reduce the search space. On the other hand, these methods can only scale up to tens of variables, as opposed to the methods we consider, which can scale up to hundreds or even a few thousand variables.

\citet{pmlr-v186-gonzales22a} utilize imperfect experts that provide a corresponding confidence probability for each orientation \ac{BK}.
In particular, their algorithm employs a similar strategy to tPC, where descendants of a variable according to the \ac{BK} are excluded from separating sets.
\citet{chen2025mitigating} also consider possibly imperfect orientation \ac{BK} and develop methods to detect and self-correct when they lead to structural inconsistencies during causal discovery.

There are several proposals on how to integrate large language models (LLMs) into causal discovery as imperfect experts \citep{wan2025large}.
\citet{long2023causal} and \citet{takayama2025integrating} use LLMs to refine the learned graph after the discovery process. \citet{kampani2024llm} and \citet{10872923} use LLMs for prior knowledge to initialize causal discovery, but learn a DAG through heuristics or continuous optimization, instead of a MPDAG with theoretical guarantees.
\citet{hiremath2025guess2graph} and \citet{NEURIPS2025_b26dfe70} use \ac{BK} during causal discovery to prioritize edge deletions and additions, instead of avoiding them.

\section{Experiments}
\label{sec:experiments}

We evaluate the effects of utilizing \ac{BK} for all methods described in \Cref{sec:method}.
For all algorithms except PC, we randomly sample a treatment-outcome target pair where the treatment is an ancestor of the outcome in the true graph.
Since we focus on \ac{BK} about direct causal relations, we do not provide this ancestral relation as \ac{BK}.
Following \citet{schubert2026local}, we extend MB-by-MB-BK and LDECC-BK to also determine the ancestral relation of the targets, denoted as MB-by-MB$^+$-BK and LDECC$^+$-BK.
As baselines, we compare our approach to the post-processing alternative.
For PC and SNAP($\infty$) we use \citet{perkovic2017interpreting}, for MB-by-MB$^+$ and LOAD we use \citep{zheng2026local}, and for LDECC$^+$ we add or remove parents and children to its output based on the \ac{BK}.
We provide additional discussion and results on post-processing in \Cref{sec:post_proc}.
We publish our code at \url{https://github.com/matyasch/bk}.

We report the number of CI tests as an implementation independent measure of computational efficiency, while execution time provides a practical metric.
Since not all methods output a causal graph, we only report structural metrics for PC-BK and SNAP($\infty$)-BK in \Cref{sec:shd} and instead measure the intervention distance \citep{schubert2025snap}, which compares the true causal effect between the targets and the causal effect estimated based on the output of the algorithm using 1000 new data points \citep{gradu2022valid}.
This aligns better with the purpose of many of these algorithms, i.e., downstream causal effect estimation.
We repeat each experiment 100 times and average results over 90 with the best and worst 5 results removed.
We provide additional details on libraries and hardware used for all experiments in \Cref{sec:experimental_details}.

\paragraph{Synthetic Data.}
We evaluate all algorithms on data sampled from random Erdős–Rényi causal models with 1000 nodes, expected degree of 3, and maximum degree of 10.
We evaluate on three different data domains and corresponding CI tests:
we use oracle d-separation CI tests to demonstrate the theoretical advantages of our framework, and also run experiments with 10000 samples of linear Gaussian and binary data.
For the linear Gaussian data we sample edge weights from the range of $[-3, 0.5] \cup [0.5,3]$ and use standard Gaussian deviation, and employ Fisher-Z CI tests.
The binary data is generated from a conditional probability table based on the causal graph with uniform probabilities, and we use $G^2$ CI tests.
We use a significance level of $\alpha = 0.01$ or all CI tests.

For each experiment, we sample \ac{BK} between a fraction of all variable pairs, according to the \emph{rate} of \ac{BK}.
If two variables are non-adjacent and selected for \ac{BK}, then we register it as a known gap.
Otherwise, it can be registered in $\mathcal{B}$ either as an adjacency or as an orientation with uniform probability.
This means that a \ac{BK} rate of 1 provides information about all variable pairs, but not necessarily all orientations, and hence may not identify the true causal DAG.
We do not show results for LDECC with Fisher-Z CI tests since its running times exceeded our threshold of $24$ hours.

In \Cref{fig:over_bk} we show the difference in each metric between our extended methods and each original method combined with postprocessing.
A lower value than zero indicates a decrease in the number of tests, execution time or distance between the true and estimated causal effects, while a higher value than zero indicates an increase.
The results demonstrate that our strategy of utilizing \ac{BK} can improve both computational requirements compared to post-processing, while maintaining the same output quality over most data types and algorithms.
The first two rows show the number of CI tests and running time, and indicate that \ac{BK} can consistently reduce the cost of all causal discovery algorithms (see \Cref{fig:bk_type_snap} for SNAP($\infty$)-BK and \Cref{fig:bk_type_load} for LOAD-BK).
These improvements are the most beneficial for PC-BK and LDECC$^+$-BK compared to post-processing.

The third row of \Cref{fig:over_bk} shows our results for the intervention distance, where our approach maintains the performance of post-processing.
In line with our theoretical proofs, the prediction accuracy of our approach is identical to post-processing in the oracle d-separation setting for all methods except LDECC, which slightly improves.
Orientation \ac{BK} can constrain the equivalence class of the true graph, and thus the set of possible adjustment sets, and sometimes turn a causal effect fully identifiable.
This benefit can be clearly observed from the third row of \Cref{fig:over_bk} with oracle d-separation CI tests, where all methods obtain more accurate estimates of the true causal effect.
Since LDECC-BK is not complete, its intervention distance is slightly higher than for MB-by-MB-BK.
However, it achieves a lower distance compared to the post-processing baseline as it can identify more children during discovery.
On linear Gaussian data, similar benefits can be observed for PC-BK, SNAP($\infty$)-BK and LDECC$^+$-BK.
On binary data, the intervention distance is quite low throughout \ac{BK} rates.

\begin{figure*}[t]
    \centering
    \includegraphics[width=\linewidth]{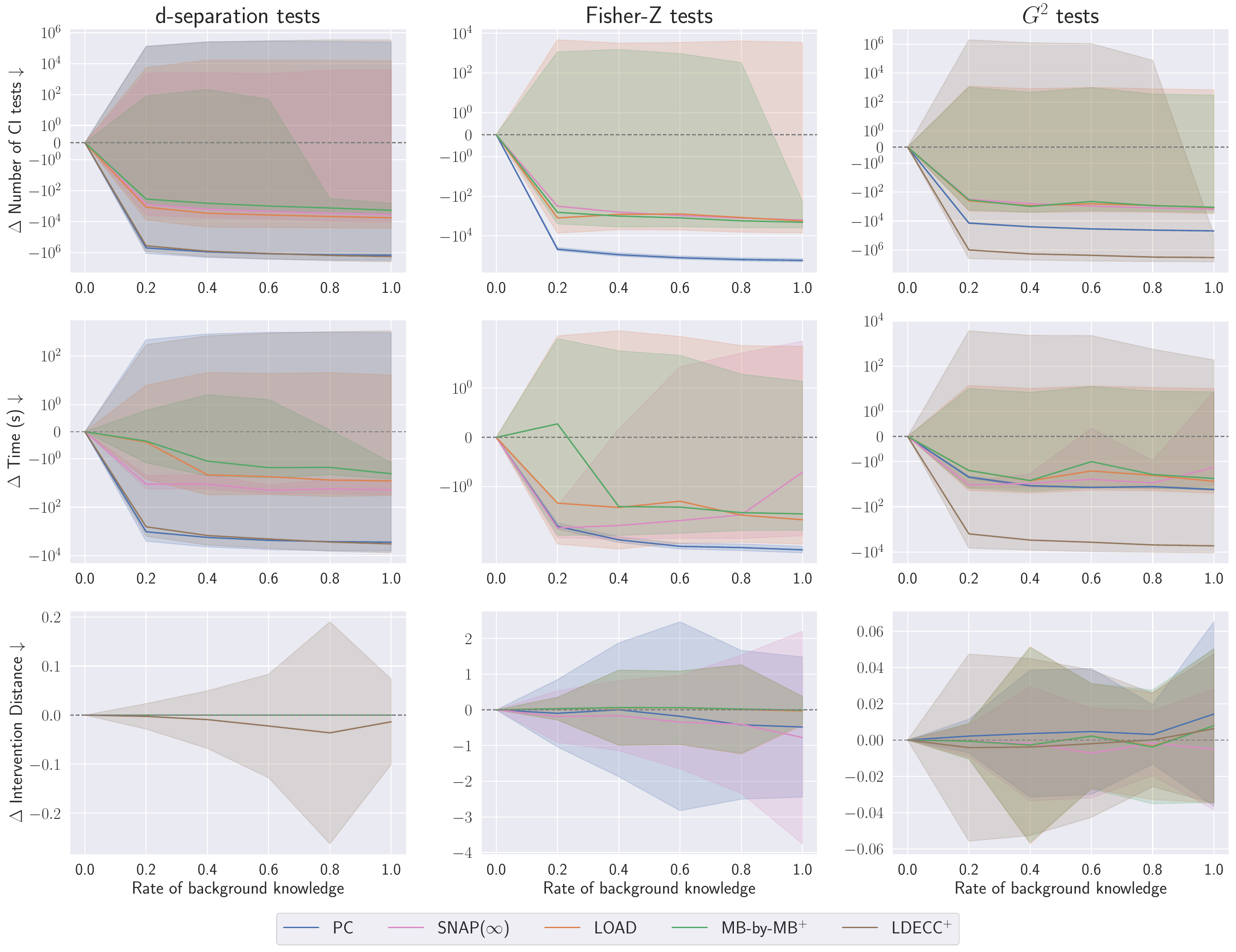}
    \caption{Difference between our methods and baselines combined with postprocessing over rate of \ac{BK} between variable pairs, where lower means better.
    The shadow area denotes standard deviation.}
    \label{fig:over_bk}
\end{figure*}

\paragraph{Ablations.}

In \Cref{sec:bk_types_abl} we study which type of BK benefits our methods the most by providing each type separately. 
Our results show that PC-BK and SNAP($\infty$)-BK benefit much more from known gaps to reduce computational costs than local methods, which benefit the most from orientation \ac{BK}.
As expected, orientation \ac{BK} improves the intervention distance for all algorithms and data modalities.

We consider \emph{imperfect \ac{BK}} in \Cref{sec:bk_error} by introducing errors into $10\%$ of adjacencies and orientations, and $5\%$ of gaps to keep errors balanced between \ac{BK} types, and show the results in \Cref{fig:fixed_error}.
While MB-by-MB$^+$-BK, and LDECC$^+$-BK still maintain decreasing computational costs as \ac{BK} increases despite errors, SNAP($\infty$)-BK becomes increasingly slow at higher \ac{BK} rates.
Intervention distances increase with more imperfect \ac{BK} due to the growing number of incorrect orientations, accompanied by larger variance.
We also vary the error rate from $10\%$, $20\%$ and $30\%$ for a fixed \ac{BK} rate of $0.2$.
Our results in \Cref{fig:bk_over_error} show that for most methods the computational cost and intervention distance increase with higher error rates, although MB-by-MB$^+$-BK, which gets slightly cheaper.
\Cref{fig:post_proc_noisy} shows that our approach remains comparable against post-processing under erroneous \ac{BK}. 

\paragraph{Realistic Data.}
We also evaluate our approach on data generated by the MAGIC-NIAB, ECOLI70, ARTH150 and ANDES networks from the \texttt{bnlearn} repository \citep{scutari2010bnlearn}.
Our results in \Cref{sec:semi-synth} show the same benefits of integrating \ac{BK} as for synthetic data.
The computational costs of most methods improve over all settings, especially on linear Gaussian data.
The main exception is SNAP($\infty$) which shows a constant cost over all rates of \ac{BK}
Without any background knowledge, the PC(-BK) and LDECC$^+$(-BK) algorithms took prohibitively long to run on the ARTH150 network, demonstrating that utilizing \ac{BK} during discovery can be necessary in some scenarios.

\section{Conclusions}
\label{sec:conclusions}

In this work, we develop methods to integrate background knowledge about direct causal relations \emph{during} the causal discovery process. This allows us to improve not only the accuracy and identifiability of the learned causal relations, but also the scalability of current methods. We implement our insights in two global causal discovery methods, three local causal discovery methods and two Markov blanket discovery algorithms from literature. 
We prove that our extensions are sound in terms of propagating the background knowledge, and that all, except LDECC-BK, are also complete. Our empirical results show that integrating background knowledge can reduce the computational cost of these methods in several settings, with gains that depend on the method, CI test and \ac{BK} type.
We plan to extend our results to settings with latent confounding and selection bias in future work.

\newpage
\section*{Acknowledgements}
We thank SURF (\url{www.surf.nl}) for the support in using the National Supercomputer Snellius.

\section*{Funding Transparency Statement}
The authors did not receive third-party support in relation to this project and declare no competing interests. In particular, while TA works at Uber, this work was performed during his Master thesis at the University of Amsterdam. Similarly, while SM had previously worked part-time for the MIT-IBM Watson AI Lab, this work is not related to the topics and projects she contributed to there.

\bibliographystyle{unsrtnat}
\bibliography{bibliography}

@book{pearl2009causality,
  title={Causality},
  author={Pearl, Judea},
  year={2009},
  publisher={Cambridge university press}
}

@ARTICLE{gradu2022valid,
title = {Valid Inference After Causal Discovery},
author = {Gradu, Paula and Zrnic, Tijana and Wang, Yixin and Jordan, Michael I.},
year = {2025},
journal = {Journal of the American Statistical Association},
volume = {120},
number = {550},
pages = {1127-1138},
url = {https://EconPapers.repec.org/RePEc:taf:jnlasa:v:120:y:2025:i:550:p:1127-1138}
}

@article{maasch2025discrimination,
	author = {Maasch, Jacqueline and Gan, Kyra and Chen, Violet and Orfanoudaki, Agni and Akpinar, Nil-Jana and Wang, Fei},
	doi = {10.1609/aaai.v39i18.34130},
	journal = {Proceedings of the AAAI Conference on Artificial Intelligence},
	month = {Apr.},
	number = {18},
	pages = {19349-19357},
	title = {Local Causal Discovery for Structural Evidence of Direct Discrimination},
	url = {https://ojs.aaai.org/index.php/AAAI/article/view/34130},
	volume = {39},
	year = {2025}
}

@book{lauritzen1996graphical,
  title={Graphical models},
  author={Lauritzen, Steffen L},
  volume={17},
  year={1996},
  publisher={Clarendon Press}
}

@book{spirtes2000causation,
  author    = {Peter Spirtes and Clark Glymour and Richard Scheines},
  title     = {Causation, Prediction, and Search},
  edition   = {2nd},
  publisher = {MIT Press},
  year      = {2000}
}

@article{colombo2012learning,
  title={Learning high-dimensional directed acyclic graphs with latent and selection variables},
  author={Colombo, Diego and Maathuis, Marloes H and Kalisch, Markus and Richardson, Thomas S},
  journal={The Annals of Statistics},
  pages={294--321},
  year={2012},
  publisher={JSTOR}
}

@article{perkovic2018complete,
  author  = {Emilija Perkovi\'c and Johannes Textor and Markus Kalisch and Marloes H. Maathuis},
  title   = {Complete Graphical Characterization and Construction of Adjustment Sets in Markov Equivalence Classes of Ancestral Graphs},
  journal = {Journal of Machine Learning Research},
  year    = {2018},
  volume  = {18},
  number  = {220},
  pages   = {1--62},
  url     = {http://jmlr.org/papers/v18/16-319.html}
}

@article{glymour2019review,
	author = {Glymour, Clark and Zhang, Kun and Spirtes, Peter},
	doi = {10.3389/fgene.2019.00524},
	issn = {1664-8021},
	journal = {Frontiers in Genetics},
	title = {Review of Causal Discovery Methods Based on Graphical Models},
	volume = {10},
	year = {2019},
}

@article{guo2023variable,
  title={Variable elimination, graph reduction and the efficient g-formula},
  author={Guo, F Richard and Perkovi{\'c}, Emilija and Rotnitzky, Andrea},
  journal={Biometrika},
  volume={110},
  number={3},
  pages={739--761},
  year={2023},
  publisher={Oxford University Press}
}

@InProceedings{schubert2025snap,
  title = 	 {SNAP: Sequential Non-Ancestor Pruning for Targeted Causal Effect Estimation With an Unknown Graph},
  author =       {Schubert, M{\'a}ty{\'a}s and Claassen, Tom and Magliacane, Sara},
  booktitle = 	 {Proceedings of The 28th International Conference on Artificial Intelligence and Statistics},
  pages = 	 {3340--3348},
  year = 	 {2025},
  editor = 	 {Li, Yingzhen and Mandt, Stephan and Agrawal, Shipra and Khan, Emtiyaz},
  volume = 	 {258},
  series = 	 {Proceedings of Machine Learning Research},
  month = 	 {03--05 May},
  publisher =    {PMLR},
  url = 	 {https://proceedings.mlr.press/v258/schubert25a.html},
}

@article{margaritis1999bayesian,
  title={Bayesian network induction via local neighborhoods},
  author={Margaritis, Dimitris and Thrun, Sebastian},
  journal={Advances in neural information processing systems},
  volume={12},
  year={1999}
}

@article{pellet2008using,
  title={Using markov blankets for causal structure learning.},
  author={Pellet, Jean-Philippe and Elisseeff, Andr{\'e}},
  journal={Journal of Machine Learning Research},
  volume={9},
  number={7},
  year={2008}
}

@article{wang2014discovering,
  title={Discovering and orienting the edges connected to a target variable in a DAG via a sequential local learning approach},
  author={Wang, Changzhang and Zhou, You and Zhao, Qiang and Geng, Zhi},
  journal={Computational statistics \& data analysis},
  volume={77},
  pages={252--266},
  year={2014},
  publisher={Elsevier}
}

@inproceedings{watson2022causal,
  title={Causal discovery under a confounder blanket},
  author={Watson, David S and Silva, Ricardo},
  booktitle={Uncertainty in Artificial Intelligence},
  pages={2096--2106},
  year={2022},
  organization={PMLR}
}

@inproceedings{gupta2023local,
  title={Local causal discovery for estimating causal effects},
  author={Gupta, Shantanu and Childers, David and Lipton, Zachary Chase},
  booktitle={Conference on Causal Learning and Reasoning},
  pages={408--447},
  year={2023},
  organization={PMLR}
}

@inproceedings{maasch2024local,
    title={Local Discovery by Partitioning: Polynomial-Time Causal Discovery Around Exposure-Outcome Pairs},
    author={Jacqueline R. M. A. Maasch and Weishen Pan and Shantanu Gupta and Volodymyr Kuleshov and Kyra Gan and Fei Wang},
    booktitle={The 40th Conference on Uncertainty in Artificial Intelligence},
    year={2024}
}

@article{zheng2026local,
    author={Zheng, Qingyuan and Liu, Yue and He, Yangbo},
    journal={ IEEE Transactions on Pattern Analysis \& Machine Intelligence },
    title={{ Local Causal Discovery with Background Knowledge }},
    year={2026},
    ISSN={1939-3539},
    pages={1-12},
    doi={10.1109/TPAMI.2026.3667409},
    url = {https://doi.ieeecomputersociety.org/10.1109/TPAMI.2026.3667409},
    publisher={IEEE Computer Society},
    address={Los Alamitos, CA, USA},
    month=feb
}

@inproceedings{schubert2026local,
  title     = {Local Causal Discovery for Statistically Efficient Causal Inference},
  author    = {M{\'a}ty{\'a}s Schubert and Tom Claassen and Sara Magliacane},
  booktitle = {The 29th International Conference on Artificial Intelligence and Statistics},
  year      = {2026},
  url       = {https://openreview.net/forum?id=FlWl20PFd7}
}

@article{witte2020efficient,
  author  = {Janine Witte and Leonard Henckel and Marloes H. Maathuis and Vanessa Didelez},
  title   = {On Efficient Adjustment in Causal Graphs},
  journal = {Journal of Machine Learning Research},
  year    = {2020},
  volume  = {21},
  number  = {246},
  pages   = {1--45},
  url     = {http://jmlr.org/papers/v21/20-175.html}
}

@article{fang2022local,
	author = {Zhuangyan Fang and Yue Liu and Zhi Geng and Shengyu Zhu and Yangbo He},
	doi = {https://doi.org/10.1016/j.artint.2022.103669},
	issn = {0004-3702},
	journal = {Artificial Intelligence},
	pages = {103669},
	title = {A local method for identifying causal relations under Markov equivalence},
	url = {https://www.sciencedirect.com/science/article/pii/S0004370222000091},
	volume = {305},
	year = {2022}}

@inproceedings{meek1995causal,
author = {Meek, Christopher},
title = {Causal inference and causal explanation with background knowledge},
year = {1995},
isbn = {1558603859},
publisher = {Morgan Kaufmann Publishers Inc.},
address = {San Francisco, CA, USA},
booktitle = {Proceedings of the Eleventh Conference on Uncertainty in Artificial Intelligence},
pages = {403–410},
numpages = {8},
location = {Montr\'{e}al, Qu\'{e}, Canada},
series = {UAI'95}
}

@inproceedings{perkovic2017interpreting,
  author       = {Emilija Perkovi\'c and
                  Markus Kalisch and
                  Marloes H. Maathuis},
  editor       = {Gal Elidan and
                  Kristian Kersting and
                  Alexander Ihler},
  title        = {Interpreting and Using CPDAGs With Background Knowledge},
  booktitle    = {Proceedings of the Thirty-Third Conference on Uncertainty in Artificial
                  Intelligence, {UAI} 2017, Sydney, Australia, August 11-15, 2017},
  publisher    = {{AUAI} Press},
  year         = {2017},
  url          = {http://auai.org/uai2017/proceedings/papers/120.pdf},
  bibsource    = {dblp computer science bibliography, https://dblp.org}
}

@article{WANG2023103964,
	author = {Tian-Zuo Wang and Tian Qin and Zhi-Hua Zhou},
	doi = {https://doi.org/10.1016/j.artint.2023.103964},
	issn = {0004-3702},
	journal = {Artificial Intelligence},
	pages = {103964},
	title = {Sound and complete causal identification with latent variables given local background knowledge},
	url = {https://www.sciencedirect.com/science/article/pii/S0004370223001108},
	volume = {322},
	year = {2023}}

@article{venkateswaran2024towards,
  title={Towards complete causal explanation with expert knowledge},
  author={Venkateswaran, Aparajithan and Perkovi{\'c}, Emilija},
  journal={arXiv preprint arXiv:2407.07338},
  year={2024}
}

@article{JMLR:v26:23-0624,
  author  = {Zhuangyan Fang and Ruiqi Zhao and Yue Liu and Yangbo He},
  title   = {On the Representation of Pairwise Causal Background Knowledge and Its Applications in Causal Inference},
  journal = {Journal of Machine Learning Research},
  year    = {2025},
  volume  = {26},
  number  = {229},
  pages   = {1--73},
  url     = {http://jmlr.org/papers/v26/23-0624.html}
}

@article{hiremath2025guess2graph,
  title={From Guess2Graph: When and How Can Unreliable Experts Safely Boost Causal Discovery in Finite Samples?},
  author={Hiremath, Sujai and Janzing, Dominik and Faller, Philipp and Bl{\"o}baum, Patrick and Kirschbaum, Elke and Kasiviswanathan, Shiva Prasad and Gan, Kyra},
  journal={arXiv preprint arXiv:2510.14488},
  year={2025}
}

@inproceedings{NEURIPS2025_b26dfe70,
	author = {Ejaz, Adiba and Bareinboim, Elias},
	booktitle = {Advances in Neural Information Processing Systems},
	doi = {10.52202/085713-4104},
	editor = {D. Belgrave and C. Zhang and H. Lin and R. Pascanu and P. Koniusz and M. Ghassemi and N. Chen},
	pages = {123068--123109},
	publisher = {Curran Associates, Inc.},
	title = {Less Greedy Equivalence Search},
	url = {https://proceedings.neurips.cc/paper_files/paper/2025/file/b26dfe705b549365ccbdbfdec74b1117-Paper-Conference.pdf},
	volume = {38, Main Conference},
	year = {2025}}

@Article{scutari2010bnlearn,
    title = {Learning Bayesian Networks with the {bnlearn} {R}
      Package},
    author = {Marco Scutari},
    journal = {Journal of Statistical Software},
    year = {2010},
    volume = {35},
    number = {3},
    pages = {1--22},
    doi = {10.18637/jss.v035.i03},
}

@Article{kalisch2012pcalg,
    title = {Causal Inference Using Graphical Models with the {R}
      Package {pcalg}},
    author = {Markus Kalisch and Martin M\"achler and Diego
      Colombo and Marloes H. Maathuis and Peter B\"uhlmann},
    journal = {Journal of Statistical Software},
    year = {2012},
    volume = {47},
    number = {11},
    pages = {1--26},
    doi = {10.18637/jss.v047.i11},
}

@article{ramsey2018tetrad,
  title={Tetrad—a toolbox for causal discovery},
  author={Ramsey, Joseph D and Zhang, Kun and Glymour, Madelyn and Romero, Ruben Sanchez and Huang, Biwei and Ebert-Uphoff, Imme and Samarasinghe, Savini and Barnes, Elizabeth A and Glymour, Clark},
  year={2018},
  journal={8th international workshop on climate informatics}
}

@article{zheng2024causal,
  title={Causal-learn: Causal discovery in python},
  author={Zheng, Yujia and Huang, Biwei and Chen, Wei and Ramsey, Joseph and Gong, Mingming and Cai, Ruichu and Shimizu, Shohei and Spirtes, Peter and Zhang, Kun},
  journal={Journal of Machine Learning Research},
  volume={25},
  number={60},
  pages={1--8},
  year={2024}
}

@article{ankan2024pgmpy,
  author  = {Ankur Ankan and Johannes Textor},
  title   = {pgmpy: A Python Toolkit for Bayesian Networks},
  journal = {Journal of Machine Learning Research},
  year    = {2024},
  volume  = {25},
  number  = {265},
  pages   = {1--8},
  url     = {http://jmlr.org/papers/v25/23-0487.html}
}

@techreport{hagberg2008exploring,
  title={Exploring network structure, dynamics, and function using NetworkX},
  author={Hagberg, Aric and Swart, Pieter and S Chult, Daniel},
  year={2008},
  institution={Los Alamos National Lab.(LANL), Los Alamos, NM (United States)}
}

@Article{csardi2006igraph,
    title = {The igraph software package for complex network research},
    author = {Gabor Csardi and Tamas Nepusz},
    journal = {InterJournal},
    volume = {Complex Systems},
    pages = {1695},
    year = {2006},
}

@InProceedings{pmlr-v108-andrews20a,
  title = 	 {On the Completeness of Causal Discovery in the Presence of Latent Confounding with Tiered Background Knowledge},
  author =       {Andrews, Bryan and Spirtes, Peter and Cooper, Gregory F.},
  booktitle = 	 {Proceedings of the Twenty Third International Conference on Artificial Intelligence and Statistics},
  pages = 	 {4002--4011},
  year = 	 {2020},
  editor = 	 {Chiappa, Silvia and Calandra, Roberto},
  volume = 	 {108},
  series = 	 {Proceedings of Machine Learning Research},
  month = 	 {26--28 Aug},
  publisher =    {PMLR},
  url = 	 {https://proceedings.mlr.press/v108/andrews20a.html}
}

@InProceedings{pmlr-v186-gonzales22a,
  title = 	 {A Hybrid Algorithm for Learning Causal Networks using Uncertain Experts’ Knowledge},
  author =       {Gonzales, Christophe and Journe, Axel and Mabrouk, Ahmed},
  booktitle = 	 {Proceedings of The 11th International Conference on Probabilistic Graphical Models},
  pages = 	 {241--252},
  year = 	 {2022},
  editor = 	 {Salmerón, Antonio and Rumí, Rafael},
  volume = 	 {186},
  series = 	 {Proceedings of Machine Learning Research},
  month = 	 {05--07 Oct},
  publisher =    {PMLR},
  url = 	 {https://proceedings.mlr.press/v186/gonzales22a.html}
}

@article{runge2019inferring,
  title={Inferring causation from time series in Earth system sciences},
  author={Runge, Jakob and Bathiany, Sebastian and Bollt, Erik and Camps-Valls, Gustau and Coumou, Dim and Deyle, Ethan and Glymour, Clark and Kretschmer, Marlene and Mahecha, Miguel D and Mu{\~n}oz-Mar{\'\i}, Jordi and others},
  journal={Nature communications},
  volume={10},
  number={1},
  pages={2553},
  year={2019},
  publisher={Nature Publishing Group UK London}
}

@inproceedings{bang2023we,
  title={Do we become wiser with time? On causal equivalence with tiered background knowledge},
  author={Bang, Christine W and Didelez, Vanessa},
  booktitle={Uncertainty in Artificial Intelligence},
  pages={119--129},
  year={2023},
  organization={PMLR}
}

@article{bang2024improving,
  title={Improving finite sample performance of causal discovery by exploiting temporal structure},
  author={Bang, Christine W and Witte, Janine and Foraita, Ronja and Didelez, Vanessa},
  journal={arXiv preprint arXiv:2406.19503},
  year={2024}
}

@article{larsen2025score,
  title={Score-Based Causal Discovery with Temporal Background Information},
  author={Larsen, Tobias Ellegaard and Ekstr{\o}m, Claus Thorn and Petersen, Anne Helby},
  journal={arXiv preprint arXiv:2502.06232},
  year={2025}
}

@article{bang2025constraint,
  title={Constraint-based causal discovery with tiered background knowledge and latent variables in single or overlapping datasets},
  author={Bang, Christine W and Didelez, Vanessa},
  journal={Proceedings of Machine Learning Research},
  volume={275},
  pages={1--31},
  year={2025}
}

@inproceedings{long2023causal,
  title={Causal Discovery with Language Models as Imperfect Experts},
  author={Long, Stephanie and Pich{\'e}, Alexandre and Zantedeschi, Valentina and Schuster, Tibor and Drouin, Alexandre},
  booktitle={ICML 2023 Workshop on Structured Probabilistic Inference $\{$$\backslash$\&$\}$ Generative Modeling},
  year={2023}
}

@inproceedings{wan2025large,
	author = {Wan, Guangya and Lu, Yunsheng and Wu, Yuqi and Hu, Mengxuan and Li, Sheng},
	booktitle = {Proceedings of the Thirty-Fourth International Joint Conference on Artificial Intelligence, {IJCAI-25}},
	doi = {10.24963/ijcai.2025/1186},
	editor = {James Kwok},
	month = {8},
	note = {Survey Track},
	pages = {10687--10695},
	publisher = {International Joint Conferences on Artificial Intelligence Organization},
	title = {Large Language Models for Causal Discovery: Current Landscape and Future Directions},
	url = {https://doi.org/10.24963/ijcai.2025/1186},
	year = {2025}
}

@inproceedings{kampani2024llm,
  title={LLM-initialized Differentiable Causal Discovery},
  author={Kampani, Shiv and Hidary, David and van der Poel, Constantijn and Ganahl, Martin and Miao, Brenda},
  booktitle={Causality and Large Models@ NeurIPS 2024},
  year={2024}
}

@inproceedings{feng2025reliability,
  title={On the reliability of large language models for causal discovery},
  author={Feng, Tao and Qu, Lizhen and Tandon, Niket and Li, Zhuang and Kang, Xiaoxi and Haffari, Gholamreza},
  booktitle={Proceedings of the 63rd Annual Meeting of the Association for Computational Linguistics (Volume 1: Long Papers)},
  pages={9565--9590},
  year={2025}
}

@article{takayama2025integrating,
  title={Integrating Large Language Models in Causal Discovery: A Statistical Causal Approach},
  author={Takayama, Masayuki and Okuda, Tadahisa and Pham, Thong and Ikenoue, Tatsuyoshi and Fukuma, Shingo and Shimizu, Shohei and Sannai, Akiyoshi},
  journal={Transactions on Machine Learning Research},
  year={2025}
}

@article{chen2025mitigating,
  title={Mitigating Prior Errors in Causal Structure Learning: A Resilient Approach Via Bayesian Networks},
  author={Chen, Lyuzhou and Ban, Taiyu and Wang, Xiangyu and Lyu, Derui and Chen, Huanhuan},
  journal={IEEE Transactions on Pattern Analysis and Machine Intelligence},
  year={2025},
  publisher={IEEE}
}

@ARTICLE{10872923,
  author={Wang, Xiangyu and Ban, Taiyu and Chen, Lyuzhou and Lyu, Derui and Zhu, Qinrui and Chen, Huanhuan},
  journal={IEEE Transactions on Knowledge and Data Engineering}, 
  title={Large-Scale Hierarchical Causal Discovery via Weak Prior Knowledge}, 
  year={2025},
  volume={37},
  number={5},
  pages={2695-2711},
  doi={10.1109/TKDE.2025.3537832}}

@article{shen2020challenges,
  title={Challenges and opportunities with causal discovery algorithms: application to Alzheimer’s pathophysiology},
  author={Shen, Xinpeng and Ma, Sisi and Vemuri, Prashanthi and Simon, Gyorgy},
  journal={Scientific reports},
  volume={10},
  number={1},
  pages={2975},
  year={2020},
  publisher={Nature Publishing Group UK London}
}

@ARTICLE{9095242,
  author={Wang, Wei and Hu, Gangqiang and Yuan, Bo and Ye, Shandong and Chen, Chao and Cui, Yayun and Zhang, Xi and Qian, Liting},
  journal={IEEE Access}, 
  title={Prior-Knowledge-Driven Local Causal Structure Learning and Its Application on Causal Discovery Between Type 2 Diabetes and Bone Mineral Density}, 
  year={2020},
  volume={8},
  number={},
  pages={108798-108810},
  doi={10.1109/ACCESS.2020.2994936}}

@article{foraita2024longitudinal,
	author = {Foraita, Ronja and Witte, Janine and B{\"o}rnhorst, Claudia and Gwozdz, Wencke and Pala, Valeria and Lissner, Lauren and Lauria, Fabio and Reisch, Lucia A. and Moln{\'a}r, D{\'e}nes and De Henauw, Stefaan and Moreno, Luis and Veidebaum, Toomas and Tornaritis, Michael and Pigeot, Iris and Didelez, Vanessa},
	date = {2024/03/21},
	doi = {10.1038/s41598-024-56721-y},
	id = {Foraita2024},
	isbn = {2045-2322},
	journal = {Scientific Reports},
	number = {1},
	pages = {6822},
	title = {A longitudinal causal graph analysis investigating modifiable risk factors and obesity in a European cohort of children and adolescents},
	url = {https://doi.org/10.1038/s41598-024-56721-y},
	volume = {14},
	year = {2024}}

@inproceedings{NEURIPS2019_5ee56059,
	author = {Greenewald, Kristjan and Katz, Dmitriy and Shanmugam, Karthikeyan and Magliacane, Sara and Kocaoglu, Murat and Boix Adsera, Enric and Bresler, Guy},
	booktitle = {Advances in Neural Information Processing Systems},
	editor = {H. Wallach and H. Larochelle and A. Beygelzimer and F. d\textquotesingle Alch\'{e}-Buc and E. Fox and R. Garnett},
	publisher = {Curran Associates, Inc.},
	title = {Sample Efficient Active Learning of Causal Trees},
	url = {https://proceedings.neurips.cc/paper_files/paper/2019/file/5ee5605917626676f6a285fa4c10f7b0-Paper.pdf},
	volume = {32},
	year = {2019}}

@article{mooij2020jci,
  author  = {Joris M. Mooij and Sara Magliacane and Tom Claassen},
  title   = {Joint Causal Inference from Multiple Contexts},
  journal = {Journal of Machine Learning Research},
  year    = {2020},
  volume  = {21},
  number  = {99},
  pages   = {1--108},
  url     = {http://jmlr.org/papers/v21/17-123.html}
}

@inproceedings{claassen2011logical,
  title={A logical characterization of constraint-based causal discovery},
  author={Claassen, Tom and Heskes, Tom},
  booktitle={Proceedings of the Twenty-Seventh Conference on Uncertainty in Artificial Intelligence},
  pages={135--144},
  year={2011}
}

@inproceedings{hyttinen2014,
title = "Constraint-based Causal Discovery: Conflict Resolution with Answer Set Programming",
author = "Antti Hyttinen and Frederick Eberhardt and Matti J{\"a}rvisalo",
year = "2014",
language = "English",
pages = "340--349",
editor = "Jin Tian and Zhang, \{Nevin L.\}",
booktitle = "Proceedings of the 30th Conference on Uncertainty in Artificial Intelligence (UAI 2014)",
publisher = "AUAI Press",
address = "International",
note = "Conference on Uncertainty in Artificial Intelligence ; Conference date: 01-01-1800",
}

@article{JMLR:v16:triantafillou15a,
  author  = {Sofia Triantafillou and Ioannis Tsamardinos},
  title   = {Constraint-based Causal Discovery from Multiple Interventions over Overlapping Variable Sets},
  journal = {Journal of Machine Learning Research},
  year    = {2015},
  volume  = {16},
  number  = {66},
  pages   = {2147--2205},
  url     = {http://jmlr.org/papers/v16/triantafillou15a.html}
}

@article{magliacane2016ancestral,
  title={Ancestral causal inference},
  author={Magliacane, Sara and Claassen, Tom and Mooij, Joris M},
  journal={Advances in Neural Information Processing Systems},
  volume={29},
  year={2016}
}

@inproceedings{wang2025polynomialdelay,
    title={Polynomial-Delay {MAG} Listing with Novel Locally Complete Orientation Rules},
    author={Tian-Zuo Wang and Wen-Bo Du and Zhi-Hua Zhou},
    booktitle={Forty-second International Conference on Machine Learning},
    year={2025},
    url={https://openreview.net/forum?id=70voOlSPos}
}

@article{colombo2014order,
  title={Order-independent constraint-based causal structure learning},
  author={Colombo, Diego and Maathuis, Marloes H},
  journal={The Journal of Machine Learning Research},
  volume={15},
  number={1},
  pages={3741--3782},
  year={2014},
  publisher={JMLR. org}
}

\newpage
\appendix
\onecolumn
\crefalias{section}{appendix}
\crefalias{subsection}{appendix}

\section{Markov blanket discovery with Background Knowledge} 
\label{sec:mb_discovery}

Markov blanket discovery algorithms input a target variable $T$ and return the set of nodes in their Markov blanket, denoted as $MB(T)$.
We can use background knowledge to initialize the MB of $T$ in the beginning of MB discovery algorithms with all variables adjacent to $T$ according to $\mathcal{B}$, as well as known spouses, i.e. all variables $V$ such that $T \to V'$ and $V \to V'$ are in $\mathcal{B}$, as described in \Cref{alg:poss-mb}.
This allows the algorithms to skip testing whether these variables are in the MB of $T$.
We implement this for the Grow-Shrink \citep{margaritis1999bayesian} algorithm in \Cref{alg:gs_bk} and Total Conditioning \citep{pellet2008using} in \Cref{alg:tc_bk}.

\begin{algorithm}
\caption{PossMB($\mathbf{V}, T, \mathcal{B})$}
\footnotesize{
\label{alg:poss-mb}
\begin{algorithmic}[1]
    \Require{Variables $\mathbf{V}$, target $T \in \mathbf{V}$, background knowledge $\mathcal{B}$}
    \State{$PossMB(T) \gets \{V \in \mathbf{V} | V - T \in \mathcal{B}\}$}
    \State{$PossMB(T) \gets PossMB(T) \cup \{V \in \mathbf{V} | V \to T \in \mathcal{B} \}$}
    \State{$PossMB(T) \gets PossMB(T) \cup \{V \in \mathbf{V} | V \gets T \in \mathcal{B}\}$}
    \State{$PossMB(T) \gets PossMB(T) \cup \{V \in \mathbf{V} | \exists V' \in \mathbf{V}: T \to V' \in \mathcal{B} \wedge V'\gets V \in \mathcal{B}\}$}
    \Ensure{Initial Markov blanket of $T: MB(T)$ given $\mathcal{B}$}
\end{algorithmic}
}
\end{algorithm}

\begin{algorithm}
\caption{Grow-Shrink-BK}
\footnotesize{
\label{alg:gs_bk}
\begin{algorithmic}[1]
    \Require{Variables $\mathbf{V}$, target $T \in \mathbf{V}$, background knowledge $\mathcal{B}$}
    \State{MB(T) $\gets$ PossMB($\mathbf{V}, T, \mathcal{B}$)}
    \While{$\exists V \in \mathbf{V} \setminus(\{T\} \cup MB(T)): T \dep V | MB(T)$}\label{line:gs:grow}
        \State{$MB(T) \gets MB(T) \cup \{V\}$}
    \EndWhile
    \While{$\exists V \in MB(T) : T \indep V | MB(T)$}
        \State{$MB(T) \gets MB(T) \setminus \{V\}$}
    \EndWhile
    \Ensure{Markov blanket of $T: MB(T)$}
\end{algorithmic}
}
\end{algorithm}

\begin{restatable}[]{theorem}{gsbk}
\label{thm:gs_bk}
Given oracle CI tests and consistent background knowledge $\mathcal{B}$, \Cref{alg:gs_bk} is sound and complete in finding the true Markov blanket of a target $T$.
\end{restatable}

\begin{proof}
Background knowledge is utilized in \Cref{alg:poss-mb}, where variables known by background knowledge to be in the Markov blanket of $T$ by definition (neighbors and spouses of $T$) are added to $MB(T)$.
These variables would also be added to $MB(T)$ by the original Grow-Shrink algorithm (Fig.~2 of \citep{margaritis1999bayesian}) starting at line~\ref{line:gs:grow}, as they are always dependent on $T$.
Since the Grow-Shrink algorithm does not specify an order for considering variables in the grow phase, we can assume an ordering where it starts with exactly the variables that are added to $MB(T)$ in the first four lines of \Cref{alg:gs_bk}.
Then from this point on, the execution of the two algorithms are equivalent.
Then, the soundness and completeness of \Cref{alg:gs_bk} is given by the soundness and completeness of the original Grow-Shrink algorithm as shown by \citet{margaritis1999bayesian}.
\end{proof}

\begin{algorithm}
\caption{Total Conditioning-BK}
\footnotesize{
\label{alg:tc_bk}
\begin{algorithmic}[1]
    \Require{Variables $\mathbf{V}$, target $T \in \mathbf{V}$, background knowledge $\mathcal{B}$}
    \State{MB(T) $\gets$ PossMB($\mathbf{V}, T, \mathcal{B}$)}
    
    \For{$V \in \mathbf{V} \setminus(\{T\} \cup MB(T))$}\label{line:tc:loop}
        \If{$T \dep V | \mathbf{V} \setminus \{T,V\}$}
            \State{$MB(T) \gets MB(T) \cup \{V\}$}
        \EndIf
    \EndFor
    \Ensure{Markov blanket of $T: MB(T)$}
\end{algorithmic}}
\end{algorithm}

\begin{restatable}[]{theorem}{tcbk}
\label{thm:tc_bk}
Given oracle CI tests and consistent background knowledge $\mathcal{B}$, \Cref{alg:tc_bk} is sound and complete in finding the true Markov blanket of a target $T$.
\end{restatable}

\begin{proof}
Background knowledge is utilized in \Cref{alg:poss-mb}, where variables known by background knowledge to be in the Markov blanket of $T$ by definition (neighbors and spouses of $T$) are added to $MB(T)$.
The rest of the algorithm adds all variables $V \in \mathbf{V}$ to $MB(T)$ if and only if $T \dep V | \mathbf{V} \setminus \{T,V\}$, which is a sound and complete criterion as shown by Property 7 of \citep{pellet2008using}.
\end{proof}

\section{Additional Algorithms}
\label{sec:algorithms}

\Cref{alg:meek} implements the orientation rules of \citet{meek1995causal}.

\begin{algorithm}[H]
\caption{Meek \citep{meek1995causal}}
\label{alg:meek}
\footnotesize
\begin{algorithmic}[1]
\Require{Partially directed acyclic graph $\hat{G}$}
\Repeat
    \If{$Z \to X - Y$ in $\hat{G}$ and $Z \notin Adj_{\hat{G}}(Y)$}
        \State Orient $X - Y$ as $X \to Y$ in $\hat{G}$
    \EndIf
    
    \If{$X \to Z \to Y$, and $X - Y$ in $\hat{G}$}
        \State Orient $X - Y$ as $X \to Y$ in $\hat{G}$
    \EndIf
    
    \If{$V \to Y \gets Z, V - X - Z$ and $X - Y$ in $\hat{G}$, and $V \notin Adj_{\hat{G}}(Z)$}
        \State Orient $X - Y$ as $X \to Y$ in $\hat{G}$
    \EndIf

    \If{$X - Y, X - V \to Z$ and $X - Z \to Y$ in $\hat{G}$, and $V \notin Adj_{\hat{G}}(Y)$}
        \State Orient $X - Y$ as $X \to Y$ in $\hat{G}$
    \EndIf
\Until{no edges can be oriented}
\Ensure{Completed partially directed acyclic graph $\hat{G}$}
\end{algorithmic}
\end{algorithm}

\Cref{alg:meek_mb_by_mb} shows the OrientUndirectedEdges procedure from the original MB-by-MB algorithm in \citep{wang2014discovering}, with the addition of the fourth Meek rule.

\begin{algorithm}[H]
\caption{OrientUndirectedEdges \citep{wang2014discovering}}
\label{alg:meek_mb_by_mb}
\footnotesize
\begin{algorithmic}[1]
\Require{Partially directed acyclic graph $\hat{G}$}
\Repeat
    \If{$Z \to X - Y$ in $\hat{G}$ and $\exists\ sepset(Z,Y)$ s.t. $X \in sepset(Z,Y)$}
        \State Orient $X - Y$ as $X \to Y$ in $\hat{G}$
    \EndIf
    \If{$X \to Z \to Y$ and $X - Y$ in $\hat G$}
        \State Orient $X - Y$ as $X \to Y$ in $\hat{G}$
    \EndIf
    \If{$X - Y$ and $X - V \to Z$ and $X - Z \to Y$ in $\hat{G}$ and $\exists\ sepset(X,Y)$ s.t. $V \in sepset(X,Y)$}
        \State Orient $V - Z$ as $V \to Z$ in $\hat{G}$
    \EndIf
    \If{$X - Y$ and $X - V \to Z$ and $X - Z \to Y$ in $\hat{G}$, and $V \notin Adj_{\hat{G}}(Y)$}
        \State Orient $X - Y$ as $X \to Y$ in $\hat{G}$
    \EndIf
\Until{no edges can be oriented}
\Ensure{Completed partially directed acyclic graph $\hat{G}$}
\end{algorithmic}
\end{algorithm}

Algorithms~\ref{alg:pc-v} and \ref{alg:rfci-v} correspond to orienting v-structures as done in the PC algorithm \citep{spirtes2000causation} and the RFCI algorithm \citep{colombo2012learning} algorithms. Both algorithms use the meta-symbol $*$ as a wildcard to denote both edge tails and arrowheads. This allows edges to be oriented as bi-directed.

\begin{algorithm}
\centering
\caption{VstructPC: Orienting v-structures in the PC algorithm \citep{spirtes2000causation}}
\label{alg:pc-v}
\begin{algorithmic}[1]
    \Require{Skeleton $\hat{U}$, separating sets $sepset$}
    \State{$\hat{G} \gets \hat{U}$}
    \For{$X - Z - Y$ in $\hat{G}$ such that $X \notin Adj_{\hat{G}}(Y)$}
        \If{$Z \notin sepset(X,Y)$}
            \State Orient $X \starleft-\starright Z \starleft-\starright Y$ as $X \starleft\to Z \gets\starright Y$ in $\hat{G}$
        \EndIf
    \EndFor
    \Ensure{Partially oriented DAG $\hat{G}$}
\end{algorithmic}
\end{algorithm}

\begin{algorithm}
\centering
\caption{VstructRFCI: Orienting v-structures in the RFCI algorithm \citep{colombo2012learning}}
\label{alg:rfci-v}
\begin{algorithmic}[1]
    \Require{Skeleton $\hat{U}$, separating sets $sepset$}
    \State $M \gets \{(X,Z,Y) \text{ such that } X - Z - Y$ in $\hat{\mathbf{E}}$ \text{ and } $X \notin Adj_{\hat{U}}(Y)\}$
    \State $L \gets \{\}$
    \Repeat
        \State $X,Z,Y \gets$ choose an unshielded triple from $M$
        \If{$Z \notin sepset(X,Y)$}\label{line:rfciv:sepset}
            \If{$X \dep Z | sepset(X,Y)$ and $Z \dep Y | sepset(X,Y)$}\label{line:rfciv:deptest}
                \State $L \gets L \cup \{(X,Y,Z)\}$
            \Else
                \For{$V \in \{X,Y\}$}
                    \If{$V \indep Z | sepset(X,Y)$}
                        \State $\mathbf{S} \gets sepset(X,Y)$
                        \State done $\gets$ False
                        \While{not done}
                            \State done $\gets$ True
                            \For{$S \in \mathbf{S}$}
                                \If{$V \indep Z | \mathbf{S} \setminus \{S\}$}
                                    \State $\mathbf{S} \gets \mathbf{S} \setminus \{S\}$
                                    \State done $\gets$ False
                                    \State \textbf{break}
                                \EndIf
                            \EndFor
                        \EndWhile
                        \State $sepset(V,Z) \gets sepset(Z,V) \gets \mathbf{S}$
                        \State $M \gets M \cup$ all triangles $(\min(V,Z),\cdot,\max(V,Z))$ in $\hat{U}$
                        \State $M \gets M \setminus$ all triples in $M$ of the form $(V,Z,\cdot), (Z,V,\cdot), (\cdot,V,Z)$ and $(\cdot,Z,V)$
                        \State $L \gets L \setminus$ all triples in $L$ of the form $(V,Z,\cdot), (Z,V,\cdot), (\cdot,V,Z)$ and $(\cdot,Z,V)$
                        \State Delete the edge $V - Z$ from $\hat{U}$
                    \EndIf
                \EndFor
            \EndIf
        \EndIf
    \Until{$M$ is empty}
    \State$\hat{G} \gets \hat{U}$
    \For{$(X, Z, Y) \in L$}
        \State Orient $X \starleft-\starright Z \starleft-\starright Y$ as $X \starleft\to Z \gets\starright Y$ in $\hat{G}$
    \EndFor
\Ensure{Partially oriented DAG $\hat{G}$ and separating sets $sepset$}
\end{algorithmic}
\end{algorithm}

We implement LocalPC-BK in \Cref{alg:localpc} according to Sec~A.1.1 in \citep{schubert2026local} by conditioning on all subsets of the Markov blanket when trying to separate variables.

\begin{algorithm}
\centering
\caption{LocalPC-BK}
\label{alg:localpc}
\begin{algorithmic}[1]
    \Require{Target $T$, Markov blanket $MB(T)$ and background knowledge $\mathcal{B}$}
    \State{$Ne(T) \gets MB(T)$}
    \State{$sepset \gets \emptyset$}
    \For{$i \in 0,\dots,|MB(T)|:$}
        \For{$V \in Ne(X)$ s.t. $V *\kern-3pt-\kern-3pt* \notin \mathcal{B}$}\label{line:ldecc:known_nb}
            \For{$\mathbf{S} \subseteq MB(X) \setminus\{V\}$ s.t. $|\mathbf{S}| = i$}
                \If{$V \indep T | \mathbf{S}$}
                    \State{$Ne(T) \gets Ne(T) \setminus \{V\}$}
                    \State{$sepset(\{T,V\}) \gets \mathbf{S}$}
                \EndIf
            \EndFor
        \EndFor
    \EndFor
    \Ensure{Neighbors $Ne(T)$ and separating sets $sepset$}
\end{algorithmic}
\end{algorithm}

\section{Proofs}
\label{sec:proofs}

\subsection{Proof for \texorpdfstring{\Cref{thm:pc_bk}}{}}
\label{sec:proofs:pc}

\pcbk*
\begin{proof}
Given a true DAG $D$ and consistent background knowledge $\mathcal{B}$, let $G^*$ be the true MPDAG and let $\hat{G}$ be the graph discovered by \Cref{alg:pc_bk} using oracle CI tests.
We now prove that $\hat{G}$ is identical to $G^*$.

We first prove that the skeleton (i.e., adjacencies) of $\hat{G}$ is identical to the skeleton of $G^*$.
The adjacencies of $\hat{G}$ are established in two consecutive procedures, using \Cref{alg:skeleton_bk} at line~\ref{line:pc:skeleton} and \Cref{alg:missing_sepsets} at line~\ref{line:pc:find_missing_sepsets}.
\Cref{alg:skeleton_bk} is the same as the skeleton search of the standard PC algorithm with two differences:
\begin{compactenum}
    \item At line~\ref{line:skel:skip}, all CI tests for pairs of $(X,Y) \in \mathcal{B}$ are skipped, and thus kept adjacent. 
    \item At line~\ref{line:skel:limit}, the candidate sets to separate $X$ from $Y$ are the $PossPa_{\hat{G}}(X, \mathcal{B}) = Adj_{\hat{G}}(X) \setminus (\{V \in \mathbf{V} | X \to V \in \mathcal{B} \vee X \gap V \in \mathcal{B} \})$.
\end{compactenum}

The first difference impacts all pairs $(X,Y) \in \mathcal{B}$.
If $X-Y \in \mathcal{B}$ or $X \to Y \in \mathcal{B}$, then $X$ and $Y$ are adjacent in $G^*$, thus leaving them adjacent without any CI tests keeps their adjacency correct.
The gaps $X \gap Y \in \mathcal{B}$ are made non-adjacent in $\hat{G}$ later on by \Cref{alg:missing_sepsets} at line~\ref{line:pc:find_missing_sepsets}.

The second difference impacts all other pairs $(X,Y) \notin \mathcal{B}$. 
In particular, this filters from the candidate sets to separate $X$ from some other variable $Y$ the known children of $X$ or variables known to be non-adjacent to $X$.
As shown by Lemma 3.3.9 in \cite{spirtes2000causation}, if $X$ and $Y$ are not adjacent, then they can be d-separated by the parents of $X$ or $Y$ in the underlying causal graph.
Since we usually do not know the true parents during the skeleton search, in standard PC we test subsets of the currently adjacent nodes, which will be a superset of the true parents.
In our setting, we assume that the background knowledge $\mathcal{B}$ is consistent, so this allows us to filter these nodes and reduce the number of separating sets to be tested. Thus, if $(X,Y) \notin \mathcal{B}$, and $X$ and $Y$ are still adjacent in $\hat{G}$ by line~\ref{line:pc:skel_end} then $X$ and $Y$ are also adjacent in $G^*$.

This means that by line~\ref{line:pc:skel_end}, the adjacencies of $\hat{G}$ and $G^*$ are the same except for $X \gap Y \in \mathcal{B}$ which are still adjacent in $\hat{G}$ and not adjacent in $G^*$.
These superfluous adjacencies are taken care of by \Cref{alg:missing_sepsets} at line~\ref{line:pc:find_missing_sepsets}, which finds a corresponding separating set for each of them.
Similarly to \Cref{alg:skeleton_bk}, the candidate separating sets do not involve known children or non-adjacencies, which is correct for the same reasoning as described earlier.
Thus, this procedure finds a separating set for every gap $X \gap Y \in \mathcal{B}$, which are then also made non-adjacent in $\hat{G}$.

In case \Cref{alg:missing_sepsets} does not find a corresponding separating set for a known gap $X \gap Y \in \mathcal{B}$, either due to incorrect CI test results or inconsistent background knowledge, then it simply ignores the known gap and keeps the edge $X - Y$ in $\hat{G}$.
This ensures that all non-adjacencies have a corresponding separating set, which is crucial for correct v-structure orientation.
However, this situation should never arise given oracle CI tests and consistent background knowledge.

This means that by line~\ref{line:pc:missing_end}, the skeleton of $\hat{G}$ and $G^*$ and all non-adjacencies have a corresponding separating set.
At line~\ref{line:pc:vstruc}, we orient v-structures the same way as standard PC.
Since we found a separating corresponding for every non-adjacency, $sepset(\{X,Y\})$ is always well-defined.
Then, the correctness of orienting v-structures is given by Lemmas 5.1.2 and 5.1.3 in \citep{spirtes2000causation}.

Finally, we maximally orient $\hat{G}$ using \Cref{alg:meek}, which is the same procedure as described by Algorithm 1 in \citep{meek1995causal} and \citep{perkovic2017interpreting}, except that we do not check for cases when $\mathcal{B}$ and $\hat{G}$ disagree, since this can never happen assuming oracle CI tests and consistent background knowledge $\mathcal{B}$.

Thus, the final output $\hat{G}$ is identical to the true MPDAG $G^*$.
\end{proof}

\subsection{Proof for \texorpdfstring{\Cref{thm:snap_k_bk}}{} and \texorpdfstring{\Cref{cor:snap_inf_bk}}{}}
\label{sec:proofs:snap}

A crucial step of SNAP($k$)-BK (\Cref{alg:snap_k_bk}) is pruning variables at every iteration.
We prune variables at line~\ref{line:snap:prune} that have no possibly directed paths to the targets $\mathbf{T}$, i.e., we keep their possible ancestors.
This step is identical to the original algorithm.
We show why it is not enough to keep only the b-possible ancestors of the targets at lower orders in \Cref{example:snap_bpossanc_not_enough}.

\begin{example}\label{example:snap_bpossanc_not_enough}

\begin{figure}[ht]
\begin{subfigure}{.38\linewidth}
    \centering
    \includegraphics[width=.8\linewidth]{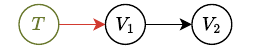}
    \caption{DAG with $\mathcal{B} = \{T \to V_1\}$.}
    \label{fig:snap_bk_in_loop_dag}
\end{subfigure}
\begin{subfigure}{.20\linewidth}
    \centering
    \includegraphics[width=\linewidth]{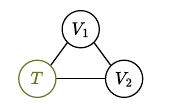}
    \caption{Skeleton.}
    \label{fig:snap_bk_in_loop_skel}
\end{subfigure}
\begin{subfigure}{.20\linewidth}
    \centering
    \includegraphics[width=\linewidth]{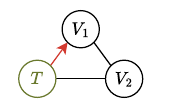}
    \caption{Oriented graph.}
    \label{fig:snap_bk_in_loop_mpdag}
\end{subfigure}
\begin{subfigure}{.20\linewidth}
    \centering
    \includegraphics[width=\linewidth]{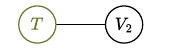}
    \caption{Pruned graph.}
    \label{fig:snap_bk_in_loop_pruned}
\end{subfigure}
\caption{
\Cref{fig:snap_bk_in_loop_dag}: Example DAG with target $T$ where utilizing background knowledge $\mathcal{B}$, denoted in red, and naively pruning variables that are not b-possible ancestors of $T$ can lead to incorrect results in SNAP($k$)-BK.
\Cref{fig:snap_bk_in_loop_skel}: The skeleton discovered at order 0.
\Cref{fig:snap_bk_in_loop_mpdag}: The graph after orienting v-structures and orienting edges according to $\mathcal{B}$.
\Cref{fig:snap_bk_in_loop_pruned}: The graph after pruning the variables that are not b-possible ancestors of $T$.
}
\label{fig:snap_bpossanc_not_enough}
\end{figure}

Consider the DAG in \Cref{fig:snap_bk_in_loop_dag} with target $T$ and background knowledge $\mathcal{B} = \{T \to V_1\}$.
The corresponding true MPDAG is identical to the DAG, and the induced subgraph of the MPDAG over the b-possible ancestors of $T$ is the graph with the single node of $T$.
However, applying Algorithm 1 of \citep{perkovic2017interpreting} at every iteration of SNAP would yield an incorrect MPDAG as follows.
\begin{compactenum}
    \item The skeleton step at order 0 produces the graph shown in \Cref{fig:snap_bk_in_loop_skel}, since all variables are marginally dependent.
    \item The graph after orienting v-structures (none) and Algorithm 1 of \citep{perkovic2017interpreting} is shown in \Cref{fig:snap_bk_in_loop_mpdag}.
    \item In this graph, only $V_2$ is a b-possible ancestor of $T$, hence $V_1$ is pruned, resulting in the graph shown in \Cref{fig:snap_bk_in_loop_pruned}.
    \item At this point, $T$ and $V_2$ can never be separated, resulting in an \textbf{incorrect} final MPDAG shown in \Cref{fig:snap_bk_in_loop_pruned}.
\end{compactenum}
\end{example}

\Cref{example:snap_bpossanc_not_enough} shows why keeping only b-possible ancestors of the targets is not enough at lower orders.
Instead, similarly to possibly ancestral sets in \citep{schubert2025snap}, we define b-possibly ancestral sets as $\mathbf{V}^* \subseteq \mathbf{V}$ such that $\text{b-PossAn}_G(\mathbf{V}^*) \subseteq \mathbf{V}^*$, and prune variables that are not in this set.
\Cref{lem:possdirpath} shows that the b-possibly ancestral set is exactly the same as the ancestral set, and thus the pruning step at line~\ref{line:snap:prune} should remain the same.

\begin{restatable}[]{lemma}{possdirpath}
\label{lem:possdirpath}
Let $G$ be a full MPDAG over variables $\mathbf{V}$ and Let $\mathbf{V}^* \subseteq \mathbf{V}$ be a b-possibly ancestral set of nodes, i.e. $\text{b-PossAn}_G(\mathbf{V}^*) \subseteq \mathbf{V}^*$. If $V \in \mathbf{V}^*$ then all nodes with a possibly directed path to $V$ is also in $\mathbf{V}^*$.
\end{restatable}

\begin{proof}
Let $V^0 = V \in \mathbf{V}^*$.
Since $\text{b-PossAn}_G(\mathbf{V}^*) \subseteq \mathbf{V}^*$, then $\text{b-PossAn}_G(V) \subseteq \mathbf{V}^*$.
This means that all of its parents $V^1 \to V$ and siblings $V^1 - V$ are also elements of $\mathbf{V}^*$.
Since $V^1 \in \mathbf{V}^*$, then all of its parents $V^2 \to V^1$ and siblings $V^2 - V^1$ are also elements of $\mathbf{V}^*$. We have to continue adding all elements to $\mathbf{V}^*$ like this until all nodes with possible directed paths $V^n -\kern-4pt* V^{n-1} -\kern-4pt* \cdots -\kern-4pt* V^1 -\kern-4pt* V^0$ are in $\mathbf{V}^*$. Hence, all nodes with possibly directed paths to $V$ are elements of $\mathbf{V}^*$.
\end{proof}

\Cref{lem:induced_mpdag} shows that unlike b-possible ancestors, b-possibly ancestral sets form a sound and complete objective, i.e., the induced subgraph of a full MPDAG over a b-possibly ancestral set retains all edges and orientations.
This is analogous to Lem.~3.1 in \citep{schubert2025snap} for possibly ancestral sets and CPDAGs.

\begin{restatable}[]{lemma}{leminducedmpdag}
\label{lem:induced_mpdag}
Let $G$ be a full MPDAG over variables $\mathbf{V}$ with background knowledge $\mathcal{B}$. 
Let $\mathbf{V}^* \subseteq \mathbf{V}$ be a b-possibly ancestral set of nodes, i.e. $\text{b-PossAn}_G(\mathbf{V}^*) \subseteq \mathbf{V}^*$.
Let $G|_{\mathbf{V}^*}$ be the induced subgraph of $G$ over $\mathbf{V}^*$ and let $G(\mathbf{V}^*)$ be the MPDAG over variables $\mathbf{V}^*$ with the same background knowledge $\mathcal{B}$. Then we have $G|_{\mathbf{V}^*} = G(\mathbf{V}^*)$.
\end{restatable}

\begin{proof}
    Let $D$ be the true underlying causal DAG, and $C$ the corresponding CPDAG of the MPDAG $G$.
    $\mathbf{V}^*$ is a b-possibly ancestral set of nodes, i.e. for every $V \in \mathbf{V}^*$ it holds that $\text{b-PossAn}_G(V) \subseteq \mathbf{V}^*$.
    By definition, $\text{An}_D(V) \subseteq \text{b-PossAn}_G(V)$.
    Hence, $\mathbf{V}^*$ is an ancestral set in $D$.
    Then, \citet{lauritzen1996graphical} (Proposition 3.22) and \citet{guo2023variable} (Lemma D.1) show that if $C|_{\mathbf{V}^*}$ is the induced subgraph of $C$ over $\mathbf{V}^*$ and $C(\mathbf{V}^*)$ is the CPDAG over variables $\mathbf{V}^*$, then it holds that $C|_{\mathbf{V}^*} = C(\mathbf{V}^*)$.
    The MPDAG corresponding to $C(\mathbf{V}^*)$ with background knowledge $\mathcal{B}$ (between pairs in $\mathbf{V}^*$) is exactly $G(\mathbf{V}^*)$.

    Now we show that the MPDAG corresponding to $C|_{\mathbf{V}^*}$ with background knowledge $\mathcal{B}$, denoted as $G'$, equals $G|_{\mathbf{V}^*}$.

    Since $G$ is a MPDAG corresponding to the CPDAG $C$ and $G'$ is a MPDAG corresponding to CPDAG $C|_{\mathbf{V}^*}$, the adjacencies in $G|_{\mathbf{V}^*}$ and $G'$ have to be identical, and every edge oriented in $C|_{\mathbf{V}^*}$ is oriented the same in $G|_{\mathbf{V}^*}$ and $G'$.
    Thus, $G|_{\mathbf{V}^*}$ and $G'$ can only differ in edge orientations due to background knowledge.
    Since $G|_{\mathbf{V}^*}$ and $G'$ are closed under orientation by definition correspond to the same background knowledge, $G|_{\mathbf{V}^*}$ and $G'$ cannot have contradicting orientations, i.e., $X \to Y$ in one and $X \gets Y$ in the other.
    Furthermore, since $G|_{\mathbf{V}^*}$ is the induced subgraph of the full MPDAG $G$ over variables $V \supseteq \mathbf{V}^*$, and thus its edges can be oriented by at least as much or more background knowledge in $\mathcal{B}$, it also cannot happen that the edge is oriented in $G'$ but not in $G|_{\mathbf{V}^*}$.
    
    Hence, if there is a differing edge between $G|_{\mathbf{V}^*}$ and $G'$, it has to be oriented in $G|_{\mathbf{V}^*}$ but not in $G'$, due to some orientation background knowledge in $\mathcal{B}$ that is not in $\mathcal{B}$.
    This orientation background knowledge has to be between variables where at least one is not in $\mathbf{V}^*$, which then propagated through $G$ via orientation rules shown in \Cref{alg:meek} and oriented an edge in $G|_{\mathbf{V}^*}$, but it could not appear in $G'$.
    In the following, we show that this cannot happen, because if an edge $X \to Y$ such that $X,Y \in \mathbf{V}^*$ is oriented via orientation rules, then all other variables appearing in the orientation also have to be in $\mathbf{V}^*$.
    Note, that \Cref{lem:possdirpath}, if $X,Y \in \mathbf{V}^*$, then every node with a possibly directed path to $X$ or $Y$ is also in $\mathbf{V}^*$.
    \begin{compactenum}
        \item The first rule involves the structure $Z \to X - Y$.
        Then, because $Z \to X$, it follows that $Z \in \mathbf{V}^*$.
        \item The second rule involves the structure $X \to Z \to Y$.
        Then, because $Z \to Y$, it follows that $Z \in \mathbf{V}^*$.
        \item  The third rule involves the structures $V \to Y \gets Z, V - X - Z$ and $X - Y$.
        Then, due to $V \to Y \gets Z$, it follows that $V,Z \in \mathbf{V}^*$.
        \item The fourth rule involves the structures  $X - Y, X - V \to Z$ and $X - Z \to Y$.
        Then, because $X - V$ and $Z \to Y$, it follows that $V,Z \in \mathbf{V}^*$.
    \end{compactenum}
    Thus, no edge in $G|_{\mathbf{V}^*}$ could have been oriented by orientations between edges involving variables not in $\mathbf{V}^*$.
    This means that $G|_{\mathbf{V}^*} = G'$.

    We have shown earlier that $C|_{\mathbf{V}^*} = C(\mathbf{V}^*)$, and that
    the MPDAG corresponding to $C(\mathbf{V}^*)$ with background knowledge $\mathcal{B}$ is $G(\mathbf{V}^*)$.
    If two CPDAGs are equal, then their corresponding MPDAG with the same background knowledge $\mathcal{B}$ is also equal.
    Thus, $G(\mathbf{V}^*) = G'$.
    Then, it follows that $G|_{\mathbf{V}^*} = G(\mathbf{V}^*)$.
\end{proof}

With the objective of recovering $G|_{\mathbf{V}^*}$ with possibly ancestral set $\mathbf{V}^*$ established, we use the following lemmas to prove \Cref{thm:snap_k_bk}.
Lem.~B.1. in \citep{schubert2025snap} is directly transferrable to SNAP($k$)-BK.

\begin{lemma}[Lemma~B.1 in \citep{schubert2025snap}]
\label{lem:skeleton}
Given oracle conditional independence tests, at any iteration $i = 0,..,k$ of \Cref{alg:snap_k_bk}, the undirected graph $\hat{U}$ is a supergraph of $U|_{\hat{\mathbf{V}}}$, the induced subgraph of the true skeleton $U$ over $\hat{\mathbf{V}}$.

\begin{proof}
In \Cref{alg:snap_k_bk} we only remove edges in $\hat{U}$ between two nodes $X, Y \in \hat{\mathbf{V}}$ for which we can find a separating set.
By faithfulness, we assume that these nodes are therefore also non-adjacent in the ground truth skeleton $U$.
Hence, the resulting skeleton $\hat{U}$ over the variables $\hat{\mathbf{V}}$ is a supergraph of the induced subgraph of the true skeleton, denoted by  $U|_{\hat{\mathbf{V}}}$.
\end{proof}
\end{lemma}

Next, we show that no variable from the possibly ancestral set containing the targets is wrongly eliminated at any order $i$, starting with v-structures.
The only difference between SNAP($k$) and SNAP($k$)-BK for orienting v-structures is that SNAP($k$)-BK already applies the RFCI v-structure orientation rules (\Cref{alg:rfci-v}) starting at iteration $i=1$ instead of $i=2$.
\Cref{example:snap_bk_vstruct_counter} shows how orienting v-structures with the PC v-structure orientation rules (\Cref{alg:pc-v}) can lead to wrong results.

\begin{example}\label{example:snap_bk_vstruct_counter}
\begin{figure}[ht]
\begin{subfigure}{.5\linewidth}
    \centering
    \includegraphics[width=.8\linewidth]{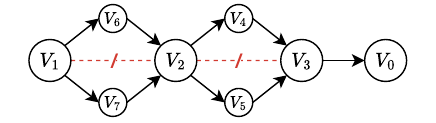}
    \caption{DAG with $\mathcal{B} = \{V_1 \gap V_2, V_2 \gap V_3\}$.}
    \label{fig:snap_gaps}
\end{subfigure}
\begin{subfigure}{.5\linewidth}
    \centering
    \includegraphics[width=.8\linewidth]{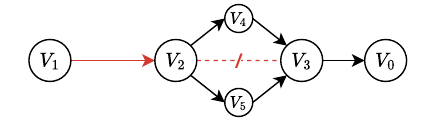}
    \caption{DAG with $\mathcal{B} = \{V_1 \to V_2, V_2 \gap V_3\}$.}
    \label{fig:snap_gap_orient}
\end{subfigure}
\caption{Example DAGs where utilizing $\mathcal{B}$, denoted in red,  during skeleton search can lead to incorrect results in SNAP. In both cases, the skeleton step up to order 1 never finds $V_1 \indep V_3 | V_2$ due to the utilization of background knowledge, keeping  $V_1$ and $V_3$ adjacent. On the other hand, it may find $V_0 \indep V_1 | V_2$ before $V_0 \indep V_1 | V_3$. In that case, at order 1, SNAP would orient a v-structure $V_1 \to V_3 \gets V_0$, which contradicts $V_3 \to V_0$ in the true DAGs.
}
\label{fig:snap_bk_vstruct_counter}
\end{figure}

Consider the DAGs in \Cref{fig:snap_bk_vstruct_counter} with background knowledge $\mathcal{B}$ denoted in red.
A naive implementation of SNAP($k$)-BK that uses the PC v-structure orientation rules at order $i=1$ may conclude that $V_3$ is a non-ancestor of $V_0$ at, even though $V_3 \to V_0$ in both DAGs.

The incorrect naive implementation would proceed as follows:
\begin{compactenum}
    \item At order 0, no variables can be made marginally independent.
    \item At order 1, the skeleton step may find $V_0 \indep V_1 | V_2$ before $V_0 \indep V_1 | V_3$, resulting in $sepset(V_0,V_1) = \{V_2\}$.
    \item $\mathcal{B}$ indicates that $V_2 \notin \text{PossPa}(V_1)$ and $V_2 \notin \text{PossPa}(V_3)$, and thus $V_2$ is not required to separate $V_1$ from $V_3$ and $V_1 \indep V_3 | V_2$ is never tested, rendering $V_1$ and $V_3$ still adjacent after the skeleton step at order 1.
    \item At order 1, the v-structure step orients $V_1 \to V_3 \gets V_0$, since $V_3 \notin sepset(V_0,V_1) = \{V_2\}$.
    \item At order 1, SNAP($k$)-BK \textbf{incorrectly} concludes that $V_3$ is a definite non-ancestor of $V_0$.
\end{compactenum}
\end{example}

The reason why using the PC v-structure orientation rules at order $i=1$ in SNAP($k$)-BK can fail is because the case of order $i=1$ in Lem. B.4. of \citep{schubert2025snap} does not hold anymore.

We now show in \Cref{lem:snap_v_struc} that correctly orienting v-structures using PC v-structure orientation rules only at order $i=0$ at line~\ref{line:snap:pc_v} and using RFCI v-structure orientation rules at line~\ref{line:snap:rfci_v} otherwise, will never contradict ancestry in the true causal DAG $D$, i.e., they orient $X \to Z \gets Y$ only if $Z \notin An_D(X)$ (and $Z \notin An_D(Y)$).

\begin{lemma}[Analogous to Lemmas~B.2 and B.4 in \citep{schubert2025snap}]\label{lem:snap_v_struc}
    Given the true MPDAG $G$ with background knowledge $\mathcal{B}$, at any iteration $i = 0,\dots,k$ of SNAP($k$)-BK (\Cref{alg:snap_k_bk}), lines \ref{line:snap:pc_v} and \ref{line:snap:rfci_v} orient  $X \starleft\to Z \gets\starright Y$ only if $Z \notin \text{b-PossAn}_G(X)$.

\begin{proof}
    Our proof follows the general structure of Lemmas~B.2 and B.4 in \citep{schubert2025snap}.
    We first show that at every iteration, $X \starleft\to Z \gets\starright Y$ is oriented only if the following three conditions hold: $Z \notin sepset(X,Y)$, $X \dep Z | sepset(X,Y)$ and $Z \dep Y | sepset(X,Y)$.

    At iteration $i=0$ SNAP($k$)-BK orients v-structures using \texttt{VstructPC} (\Cref{alg:pc-v}), which orients unshielded triples as $X \to Z \gets Y$ if and only if $Z \notin sepset(X,Y) = \emptyset$.
    At this point, every pair of variables that may be separated (not known as adjacent by $\mathcal{B}$) has been tested for independence with the empty conditioning set during the skeleton step.
    Thus, if there is still an edge between $X$ and $Z$ and between $Z$ and $Y$, then they are dependent given the empty set, or in other words, $X \dep Z | \emptyset$ and $Y \dep Z | \emptyset$ always hold.

    At every iteration $i > 0$, SNAP($k$)-BK orients v-structures using \texttt{VstructRFCI} (\Cref{alg:rfci-v}), which explicitly tests 
    the three conditions of $Z \notin sepset(X,Y)$, $X \dep Z | sepset(X,Y)$ and $Z \dep Y | sepset(X,Y)$ at lines \ref{line:rfciv:sepset} and \ref{line:rfciv:deptest}. 

    Then, the following proof by contradiction is directly reproduced from Lem~B.2 in \citep{schubert2025snap}.
    For the sake of contradiction, let us assume that there exists a DAG $D$ in the MEC of $G$ such that $Z \in An_D(X)$.
    Then, there is a directed path $Z \to \dots \to X$ in $D$.
    Since $Z \dep Y | sepset(X,Y)$, there exists at least one open path between $Y$ and $Z$ not blocked by $sepset(X,Y)$ in $D$.
    However, since $X \indep Y | sepset(X,Y)$, there has to be a node $W \in sepset(X,Y)$ on the directed path from $Z \to \dots  W \dots \to X$ that blocks this directed path in $D$.
    Otherwise, the open path between $Y$ and $Z$ and the directed path from $Z$ to $X$ would not be blocked by $sepset(X,Y)$.
    However, we also know that $X \dep Z | sepset(X,Y)$ implying that there is some other path between $X$ and $Z$ that is not blocked by $sepset(X,Y)$.
    Thus, there is a path between $X$ and $Z$ and between $Z$ and $Y$ that are not blocked by $sepset(X,Y)$ in  $D$.
    In this case, $X \indep Y | sepset(X,Y)$ can only hold if $Z$ is a collider when connecting these two paths, because $Z \notin sepset(X,Y)$.
    However, this path is then unblocked by $W$ since it is a descendant of $Z$ and it is in $sepset(X,Y)$.
    It follows that $X \indep Y | sepset(X,Y)$ cannot hold in $D$, where $Z \in An_{D}(X)$, hence we get a contradiction.
    Since no such DAG $D$ in the MEC of $G$ can exist where $Z \in An_{D}(X)$, this proves that $Z \notin \text{b-PossAn}_G(X)$.
\end{proof}
\end{lemma}

Using the lemmas above, we can prove the main result of \Cref{thm:snap_k_bk} and consequently \Cref{cor:snap_inf_bk}.

\snapkbk*
\begin{proof}
    The proof is analogous to the proof of Lem~B.2 in \citep{schubert2025snap}.
    When pruning variables at line~\ref{line:snap:prune}, all edges are oriented either by v-structures, which are shown to be sound by \Cref{lem:snap_v_struc}, or by consistent orientation background knowledge $\mathcal{B}$.
    From this and \Cref{lem:skeleton} it follows that if there is a possibly directed path in the true MPDAG $G$ from some $V \in \hat{\mathbf{V}}$ to $V' \in \hat{\mathbf{V}}$, then a possibly directed path with the same sequence of nodes also exists in $\hat{G}$ and it is also possibly directed.
    Since $\hat{\mathbf{V}}$ is constructed by taking all nodes with a possibly directed path to $\mathbf{T}$, then by \Cref{lem:possdirpath} it holds that $\hat{\mathbf{V}}$ forms a b-possibly ancestral set that contains $\mathbf{T}$.
\end{proof}

\snapinfbk*
\begin{proof}
    We show that the graph $\hat{G}$ by the end of SNAP($k$)-BK running with $k = |\mathbf{V}|-2$ has the same skeleton and v-structures as the induced sub-graph of the true MPDAG $G|_{\hat{\mathbf{V}}}$.
    This part of the proof is analogous to the proof of Lem.~B.5 in \citep{schubert2025snap}.
    \Cref{lem:skeleton} shows that the discovered skeleton $\hat{U}$ is always a supergraph of the skeleton of $G|_{\hat{\mathbf{V}}}$.
    Furthermore, since SNAP$k$-BK only removes an edge if a corresponding separating set is found, then by faithfulness, every node is adjacent to its parents.
    Furthermore, since the background knowledge $\mathcal{B}$ is consistent, the parents of each node $V$ are a subset of its possible parents $PossPa_{\hat{\mathbf{G}}}(V, \mathcal{B})$.
    Then, by line~\ref{line:snap:loop_end} with $k = |\mathbf{V}|-2$, all subsets of the parents have been tried as conditioning sets when trying to separate a variable from another.
    Hence, by faithfulness, if two nodes are still adjacent in $\hat{G}$ then they are also adjacent in $G|_{\hat{\mathbf{V}}}$.
    From this and \Cref{lem:skeleton} it follows that the skeleton of $\hat{G}$ and $G|_{\hat{\mathbf{V}}}$ are identical.

    Then, their unshielded triples are also identical and by \Cref{lem:snap_v_struc} it follows that their v-structures are also identical.

    Given $\hat{G}$ obtained from SNAP($k$)-BK with $k = |\mathbf{V}|-2$, the final step of the proof is analogous to the proof of Thm.~3.2 in \citep{schubert2025snap}.
    If $\hat{G}$ and $G|_{\hat{\mathbf{V}}}$ have the same skeleton and v-structures, then they are Markov equivalent.
    Then, as shown by \citet{perkovic2017interpreting}, applying consistent orientation background knowledge in $\mathcal{B}$ and orientation rules in \Cref{alg:meek} on $\hat{G}$ yields exactly the true MPDAG $G|_{\hat{\mathbf{V}}}$.
    The final output graph is obtained at line~\ref{line:snap:final_prune}, where the final b-possibly ancestral set is obtained according to the $\hat{G} = G|_{\hat{\mathbf{V}}}$.
    By \Cref{lem:induced_mpdag}, $\hat{G}$ is an MPDAG.
\end{proof}

\subsection{Proof for \texorpdfstring{\Cref{thm:mb_by_mb_bk}}{}}
\label{sec:proofs:mb_by_mb}

Before proving \Cref{thm:mb_by_mb_bk}, we provide a detailed an example for why a naive application of PC-BK (\Cref{alg:pc_bk}) results in wrong local structures discovered in MB-by-MB-BK

\begin{example}\label{example:mb_by_mb}
Consider the example DAG/MPDAG shown in \Cref{fig:mb_by_mb_pc_dag} with background knowledge $\mathcal{B} = \{T \to V_2, V_2 \gap X\}$ denoted in red.
Naively using PC-BK when when discovering the local structure over the Markov blanket of $T$ can incorrectly identify $X$ as adjacent to $T$, as shown in \Cref{fig:mb_by_mb_naive}.
In particular, the MB-by-MB-BK algorithm would then proceed as follows:

\begin{compactenum}
    \item In the beginning of the algorithm $\text{WaitList = [T]}$ and $\hat{G} = (\{T, X, V_1, V_2, V_3, V_4\}, \emptyset)$.
    \item In the first iteration, at line~\ref{line:mb_by_mb:gs}, MB-by-MB-BK correctly correctly finds $MB(T) = \{X,V_1,V_2,V_3\}$ 
    \item Then, at line~\ref{line:mb_by_mb:pc}, MB-by-MB proceeds to discover the local structure $L_T$ over $MB^+(T)$ naively using PC-BK:
    \begin{compactenum}
        \item When trying to separate nodes from $T$, PC-BK does not find $T \indep X | \{V_2\}$ because $V_2 \notin \text{PossPa}_{L_T}(T,\mathcal{B})$ due to $T \to V_2 \in \mathcal{B}$.
        \item When trying to separate nodes from $X$, PC-BK also does not find $X \indep T | \{V_2\}$ because $V_2 \notin \text{PossPa}_{L_T}(X,\mathcal{B})$ due to $X \gap V_2 \in \mathcal{B}$.
        \item The only other way to separate $T$ and $X$ would be $T \indep X | \{V_4\}$, but $V_4 \notin MB^+(T)$ and thus it is not considered in the local structure.
        \item For the same reason, $V_3 \indep X | \{V_4\}$ is also not found.
        \item At the end of the skeleton search by line~\ref{line:pc:missing_end}, $T$ and $X$ (as well as $V_3$ and $X$) are still \textbf{incorrectly} adjacent.
        \item At line~\ref{line:pc:vstruc}, PC-BK orients the v-structure $T \to X \gets V_3$  (along with $T \to V_1 \gets V_3$),
    \end{compactenum}
    \item MB-by-MB-BK updates $\hat{G}$ with edges and v-structures containing $T$ in $L_T$ at line~\ref{line:mb_by_mb:update} including $T \to X \gets V_3$.
    \item However, the edge between $X$ and $V_1$ in $L_T$ is discarded as it is not part of a v-structure, due to the extra edge between $T$ and $X$.
    \item MB-by-MB-BK applies background knowledge $T \to V_2$ at line~\ref{line:mb_by_mb:orient}, resulting in the graph shown in \Cref{fig:mb_by_mb_naive}
    \item As all neighbors of $T$ are oriented, MB-by-MB-BK terminates and \textbf{incorrectly} returns $X$ as a child of $T$.
\end{compactenum}

Instead, if PC-BK uses the adjusted definition of possible parents defined in \Cref{eq:mb_by_mb_posspa}, then the MB-by-MB-BK algorithm would proceed as follows:
\begin{compactenum}
    \item In the beginning of the algorithm $\text{WaiList = [T]}$ and $\hat{G} = (\{T, X, V_1, V_2, V_3, V_4\}, \emptyset)$.
    \item In the first iteration, at line~\ref{line:mb_by_mb:gs}, MB-by-MB-BK correctly correctly finds $MB(T) = \{X,V_1,V_2,V_3\}$ 
    \item Then, at line~\ref{line:mb_by_mb:pc}, MB-by-MB proceeds to discover the local structure $L_T$ over $MB^+(T)$ using PC-BK:
    \begin{compactenum}
        \item When trying to separate nodes from $T$, PC-BK does not find $T \indep X | \{V_2\}$ because $V_2 \notin \text{PossPa}_{\hat{G}}^*(T,\mathcal{B})$ due to $T \to V_2 \in \mathcal{B}$.
        \item When trying to separate nodes from $X$ at $i=1$ of the skeleton search of PC-BK, at this point $X$ and $V_2$ are adjacent, and so $V_2 \in \text{PossPa}^*_{\hat{G}}(X,\mathcal{B}) = Adj_{\hat{G}}(X)$. Thus PC-BK can \textbf{correctly} find $X \indep T | \{V_2\}$.
        \item $V_3 \indep X | \{V_4\}$ is again not found, but in this case it is not an issue.
        \item At the end of the skeleton search by line~\ref{line:pc:missing_end}, $T$ and $X$ are \textbf{correctly} non-adjacent.
        \item At line~\ref{line:pc:vstruc}, PC-BK only orients v-structure $T \to V_1 \gets V_3$.
    \end{compactenum}
    \item MB-by-MB updates its global graph with edges and v-structures containing $T$ at line~\ref{line:mb_by_mb:update}, including $T \to V_1 \gets X$.
    \item MB-by-MB-BK applies background knowledge $T \to V_2$ at line~\ref{line:mb_by_mb:orient}, resulting in the graph shown in \Cref{fig:mb_by_mb_correct}.
    \item As all neighbors of $T$ are oriented, MB-by-MB-BK terminates and \textbf{correctly} returns $X$ as non-adjacent of $T$.
\end{compactenum}

\end{example}

\mbbybbk*

\begin{proof}
    \Cref{alg:mb_by_mb_bk} utilizes background knowledge in Markov blanket discovery, learning the local structure over the Markov blanket, and orienting edges.
    Markov blanket discovery is performed by GS-BK (\Cref{alg:gs_bk}) at line~\ref{line:mb_by_mb:gs}, which is sound and complete by \Cref{thm:gs_bk}.
    
    The local structure over Markov blankets is learned at line~\ref{line:mb_by_mb:pc} using PC-BK (\Cref{alg:pc_bk}) with the updated definition of possible parents given in \Cref{eq:mb_by_mb_posspa}.
    An example for why this adjustment is necessary is shown in \Cref{example:mb_by_mb}.
    In particular, the adjustment definition in \Cref{eq:mb_by_mb_posspa} make sure that the limiting of candidate separating sets in \Cref{alg:skeleton_bk} and \Cref{alg:missing_sepsets} is only applied for the variable $X$ when discovering the local structure over $MB(X)$.
    Limiting the candidate separating for $X$ is sound, because all of its neighbors are in its Markov blanket, including its possible parents. 
    However, this might be not true for any other node in $MB(X)$, hence we do not limit their separating sets (as shown by \Cref{example:mb_by_mb}).

    Hence both the Markov blankets as well as the adjacencies and v-structures learned in the local structures are identical to the ones MB-by-MB would learn without background knowledge.
    
    Finally, orientation background knowledge is applied on the learned graph at line~\ref{line:mb_by_mb:orient}.
    This is the same step as line 3 of Algorithm 1 of \citep{zheng2026local}, another approach for MB-by-MB with background knowledge.
    As the other steps of the two algorithms are the same the original MB-by-MB algorithm \citep{wang2014discovering}, \Cref{alg:mb_by_mb_bk} is equivalent to Algorithm 1 of \citep{zheng2026local} in terms of orientations.
    Thus, the soundness and completeness of \Cref{alg:mb_by_mb_bk} follows from Lem.~11 of \citep{zheng2026local}.
\end{proof}

\subsection{Example of when LDECC-BK (\texorpdfstring{\Cref{sec:ldecc}}{}) is not complete}

The following example shows how LDECC-BK fails to identify a parent of a target even though it is oriented in the true MPDAG.

\begin{example}\label{example:ldecc_not_complete}
    Consider the true DAG $X \to Y \to T$ with \ac{BK} $\mathcal{B} = \{X \to Y\}$.
    The MPDAG of this DAG equals to the true DAG, since after applying the orientation $X \to Y$ to the corresponding CPDAG $X - Y - T$, as $X \to Y - T$, the edge $Y - T$ is oriented by the first Meek rule into $Y \to T$, resulting in a fully oriented MPDAG $X \to Y \to T$.

    Applying LDECC-BK in \Cref{alg:ldecc_bk} on the example above with target $T$ however cannot identify $Y$ as a parent of $T$ because it has no way of propagating orientation \ac{BK} outside of the Markov blanket of $T$.
    In particular, the algorithm executes as follows.
    \begin{compactenum}
        \item At line~\ref{line:ldecc:gs}, LDECC-BK finds $MB(T) = \{Y\}$.
        \item At line~\ref{line:ldecc:localpc}, LDECC-BK finds $Ne(T) = \{Y\}$.
        \item At line~\ref{line:ldecc:uccchildren}, LDECC-BK identifies $Ch(T) = \{\}$.
        \item The loop at line~\ref{line:ldecc:fullpc} does not find more independences due to $X *\kern-4pt-\kern-4pt* Y) \in \mathcal{B}$ and all other CI tests including target $T$.
        \item Finally, LDECC-BK \textbf{incorrectly} returns $Y$ as a sibling of $T$.
    \end{compactenum}
    Note, that LDECC-BK would also return $Y$ as a sibling of $T$ if it did not skip CI tests for $X$ and $Y$.
    In that case, it would simply find $X \dep Y | \emptyset$ and $X \dep Y | \{T\}$, which do not allow it to orient any edges.
\end{example}

\subsection{Proof for \texorpdfstring{\Cref{thm:load_bk}}{}}
\label{sec:proofs:load}

To prove \Cref{thm:load_bk}, we first adapt Def. D.1 and Lem 4.4 in \citep{schubert2026local} to MPDAGs.

\begin{definition}[Projection for MPDAGs with single treatment and outcome given $\mathbf{F}$, analogous to Def. D.1 in \citep{schubert2026local}]
    Let $G$ be a MPDAG with nodes $\mathbf{V}$ and background knowledge $\mathcal{B}$, and let $T, O \in \mathbf{V}$, and $\mathbf{F} \subset \mathbf{V} \setminus \{T,O\}$.    
    A projection $\hat{G}^{T,O}$ of $G$ is a graph with nodes $\mathbf{V} \setminus \mathbf{F}$ and edges as follows. For distinct nodes $W_i, W_j \in \mathbf{V} \setminus \mathbf{F}$,
    \begin{compactenum}
        \item $\hat{G}^{T,O}$ contains a directed edge $W_i \to W_j$ if and only if $G$ contains a directed path $W_i \to \cdots \to W_j$ on which all non-endpoint nodes are in $\mathbf{F}$,
        \item $\hat{G}^{T,O}$ contains a bi-directed edge $Wi \leftrightarrow Wj$ if and only if $G$ contains a path, with at least one non-endpoint node, of the form $W_i \gets \cdots \to W_j$ on which all non-endpoints are non-colliders and in $\mathbf{F}$,
        \item $\hat{G}^{T,O}$ contains an undirected edge $W_i - W_j$ if and only if $G$ contains $W_i - W_j$.
    \end{compactenum}
\end{definition}

The following lemma, analogous to Lem. 4.4 in \citep{schubert2026local}, shows that the modified forbidden projection in LOAD-BK results in an MPDAG in which the optimal adjustment set are the parents of the outcome.

\begin{restatable}[Analogous to Lem. 4.4 in \citep{schubert2026local}]{lemma}{possde}
\label{lemma:possde}
    For treatment $T$ and outcome $O$ in an amenable MPDAG $G$ with background knowledge $\mathcal{B}$, let $\mathbf{D} = PossDe_G(T)) \setminus \{T,O\}$ and $G^{T,O}$ be the modified forbidden projection with nodes $\mathbf{V}\setminus \mathbf{D}$ and background knowledge on adjacencies and orientations over $\mathbf{V}\setminus \mathbf{D}$ in $\mathcal{B}$. 
    We show that $G^{T,O}$ is a MPDAG and the  optimal adjustment set $Oset_G(X,Y)$ is given by
    $$
    Oset_G(T,O) = Pa_{G^{T,O}}(O) \setminus \{T\}.
    $$
\end{restatable}

\begin{proof}
    We first note, that the background knowledge on adjacencies and orientations over $\mathbf{V}\setminus \mathbf{D}$ in $\mathcal{B}$ is valid under the modified projection.
    This holds because the modified forbidden projection does not remove any adjacencies between nodes that are not projected out, hence their direct causal relations remain the same as in the original MPDAG $G$ over all variables.

    Then, the only difference between the modified forbidden projection of LOAD-BK and LOAD arises from orientation \ac{BK}, which may orient some additional undirected neighbors of $T$ and parents, thus reducing the set of projected variables $\mathbf{D}$.

    Now, the same proof applies as for em. 4.4 in \citep{schubert2026local}, which we also present here for completeness.

    As shown in Prop.~22 by \citet{witte2020efficient} projecting out the forbidden nodes $\mathbf{F}= PossDe_G(PossCn_G(T,O)) \setminus \{T, O\}$, where $PossCn_G(T,O)$ are the nodes on possibly directed paths from $T$ to $O$, excluding $T$, from $G$ results in an MPDAG $\tilde{G}^{T,O}$, where optimal adjustment set are the parents of $O$ except $T$ in $\tilde{G}^{T,O}$.

    We can further project out the remaining variables in $\tilde{G}^{T,O}$, $\Delta = PossDe_{G}(T) \setminus PossDe_{G}(PossCn_{G}(T,O)) \cup \{T, O\}$.
    These will be still the remaining possible descendants of $T$ in this new graph, i.e., $PossDe_{\tilde{G}^{T,O}}(T)$, which are definite non-ancestors of $O$ and its possible ancestors, e.g., its parents, in both graphs, since otherwise they would be possible causal nodes in the original graph $G$ and removed in the first projection. As such, following \citet{schubert2025snap}, they can be safely marginalized out in $G^{T,O}$ without changing any orientation in the graph regarding $O$, its possible ancestors, e.g., parents, i.e., without changing the parent set of $O$, and hence the optimal adjustment set.

    To prove that $G^{T,O}$ is still a MPDAG, we first show that it does not contain bi-directed edges, and then show that we do not add any other edge to the remaining variables with respect to the forbidden projection $\tilde{G}^{T,O}$.

    In order for a bi-directed edge to appear in $G^{T,O}$, there have to exist two nodes $W_i, W_j \notin \mathbf{D}$ such that $W_i \gets \cdots \to W_j$ in $G$ and all non-endpoint nodes on this path are non-colliders and in $\mathbf{D}$. This means that both $W_i$ and $W_j$ have to be possible descendants of a node in $\mathbf{D}$. By transitivity, possible descendants of nodes in $\mathbf{D}$ are also in $\mathbf{D}$, except for $T$ and $O$, which are then the only possible candidates for $W_i$ and $W_j$.
    On the other hand, there cannot exists such a path $W_i=T \gets \cdots \to O=W_j$ in $G$ such that all non-endpoint nodes on this path are possible descendants of $T$ due to the definition of $\mathbf{D}$, since this would be a contradiction based on the orientation of the path.

We now show that the modified projection does not add any other directed or undirected edge compared to the original forbidden projection $\tilde{G}^{T,O}$.
A fully undirected path or (possibly) directed path from $T$ to other variables $V \in \mathcal{V}\setminus \{\Delta \cup O, T\}$ that passes through some variable $D \in \Delta$ would imply that $V$ is also in $\Delta$ by definition of possible descendants, which is a contradiction. Moreover, it is not possible to have a  directed path from $V$ to $T$ through some $D \in \Delta$, since this would contradict the fact that $\Delta$ are possible descendants of $T$.

We will therefore consider the case when there is a possibly directed path $p=(V, \dots, D_1, \dots, D_k, \dots, T)$ in $\mathcal{G}$ for $V \in \mathcal{V}\setminus \{\Delta \cup \{O, T\}\}$ where all non-endpoint nodes $D_1, \dots, D_k$ are in $\Delta$. The subpath $p'$ from $D_1 - D_2 \dots - D_k - T$ has to undirected, since it cannot be possibly directed from any $D_i$ to $T$ by definition of $\Delta$ and it cannot be possibly directed in the other direction by assumption. Additionally, the subpath $p''$ from $V \dots D_1$ cannot be completely undirected, otherwise $V$ would be part of $\Delta$. In order for $p'$ to stay undirected, we need that the variable $V_j$ on $p''$ closest to $D_1$ that has an orientation is part of a shielded triple $V_i \to V_j - D_1$ with additionally $V_i - D_1$ or $V_i \to D_1$ , otherwise the orientation will propagate further due to Meek's rules \citep{meek1995causal}, but we cannot use $V_i -D_1$ , since this would mean that $V_i \in \Delta$, because it is connected by an undirected path to $T$, which is a contradiction. So the only possible shield is $V_i \to D_1$, but then since we have $V_i  \to D_1 - D_2$ this would imply that there needs to be another shield $V_i \to D_2$, on which we can apply the same argument on each node $D_i$, until we would get a shield $V_i \to T$. This means that in the projection we would not need to add any edge from $V_i$ to $T$. A simplified case is when we only have three variables $V \to D - T$, which would require $V - T$, which is again a contradiction, or $V \to T$ which would mean we do not need to add a new edge after the projection.
The only possible path from $T$ to other variables $V \in \mathcal{V}\setminus \{\Delta \cup O, T\}$ that passes through some variable $D \in \Delta$ would have $D$ as a collider, which means we can safely remove this variable and still have a valid MPDAG, which means we do not need to add any edge to the original $\tilde{G}^{T,O}$.
\end{proof}

\loadbk*
\begin{proof}
    \Cref{thm:mb_by_mb_bk} it follows that every step of LOAD-BK uses sound and complete information according to the true MPDAG.
    Then, the soundness and completeness of Steps 1 and 2 follows from Lem 4.1. and Cor 4.1  of \citep{schubert2026local}, while Step 4 follows from Thm. 2 by \citep{fang2022local}.
    Finally, \Cref{lemma:possde} shows that step 4 of LOAD-BK is also sound and complete.
\end{proof}

\section{Limitations and Extensions}
\label{sec:limits_and_extens}
In this work we focus on background knowledge about direct causal relationships between pairs of variables, and how it can be utilized during causal discovery.
However, our approach is also compatible with any type of post-processing approach that is appropriate to the output structures, e.g. MPDAG for PC-BK/SNAP($\infty$)-BK or local structures for MB-by-MB-BK/LDECC-BK and LOAD-BK.

\subsection{Ancestral relations.}
\label{sec:ancestral_bk}
Other types of \ac{BK}, such as  ancestral relationships, often have implications in terms of direct relations that we can still exploit in our methods.
For example, if $X$ is an ancestor of $Y$, this obviously implies that $Y$ is not a direct cause of $X$. 
Additionally, if we know that $X$ and $Y$ are adjacent, then in this case we can also infer that $X$ is a direct cause of $Y$.
Finally, if $X$ and $Y$ are not ancestors of each other then we know that there is a gap between them.

Utilizing ancestral \ac{BK} in a more complete way during discovery is not trivial, as \citet{JMLR:v26:23-0624} showed that ancestral \ac{BK} cannot even be fully represented by MPDAGs.

\section{Experimental Details}
\label{sec:experimental_details}

For our experiments, we used the following libraries: igraph \citep{csardi2006igraph} (GNU GPL version 2 or later), networkx \citep{hagberg2008exploring} (3-Clause BSD License), bnlearn \citep{scutari2010bnlearn} (MIT License), causal-learn \citep{zheng2024causal} (MIT License).
In particular, we used the CI test implementations from the causal-learn package \citep{zheng2024causal}.

All experiments were performed on AMD Rome CPUs, using 48 CPU cores and 84 GiB of memory.
We let each experiment run for at most 24 hours.
If an experiment over 100 seeds did not finish in the given time or did not fit in the given memory, then we do not report any results for it.

\section{Results on Structural Hamming Distance}
\label{sec:shd}

The various methods we extend recover different types of subgraphs with different sets of nodes or even just local structures, such as LOAD-BK and LDECC$^+$-BK.
Because of this it would not be fair or even possible to compare all algorithms using structural metrics. 
Instead, we generally compare all methods based on the downstream task we are interested in, i.e., causal effect estimation for a set of target variables, which we measure through the intervention distance.

\begin{figure*}[t]
    \centering
    \begin{subfigure}[b]{0.49\linewidth}
        \centering
        \caption*{$\quad$ Fisher-Z tests}
        \includegraphics[width=.7\linewidth]{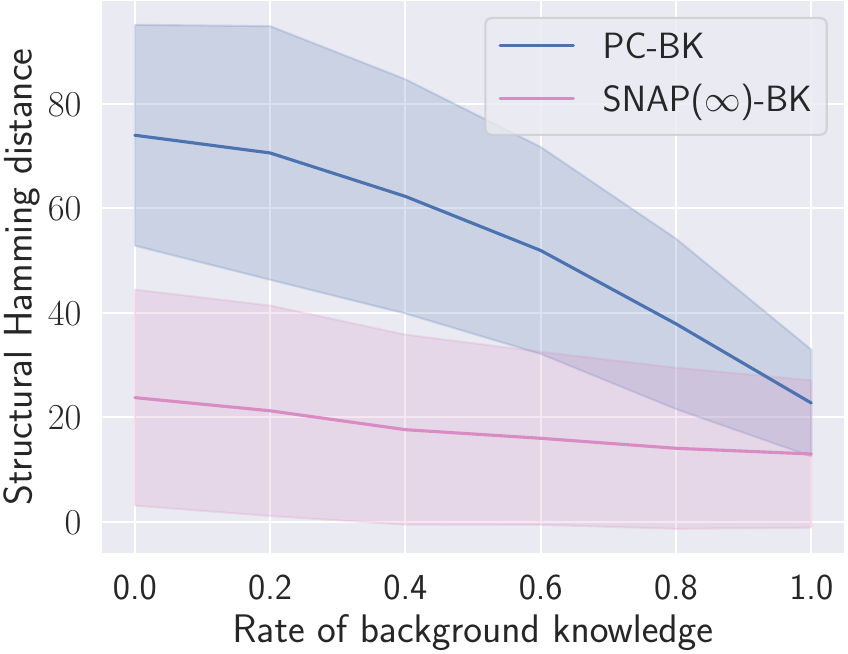}
    \end{subfigure}
    \begin{subfigure}[b]{0.49\linewidth}
        \centering
        \caption*{$G^2$ tests}
        \includegraphics[width=.7\linewidth]{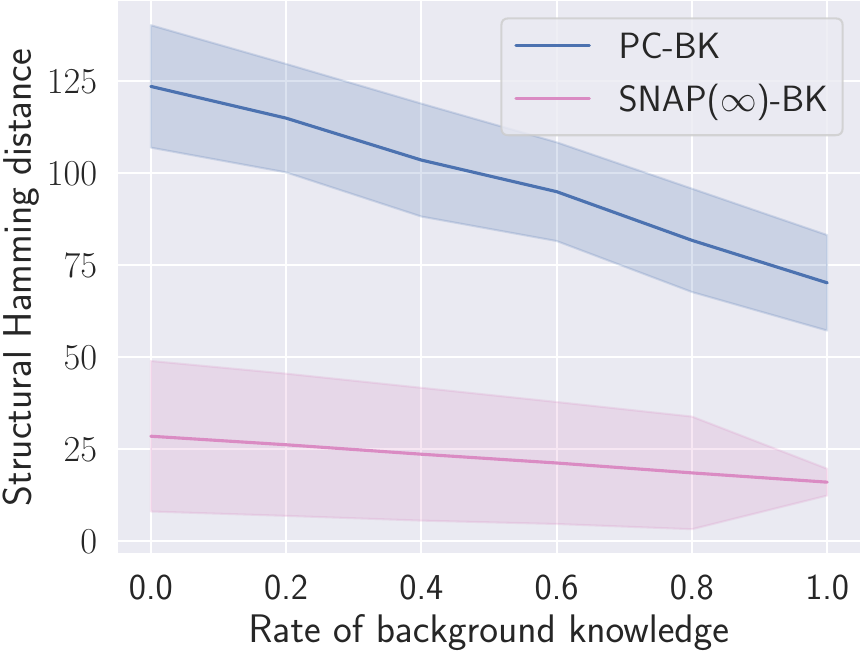}
    \end{subfigure}
    \caption{SHD for PC-BK and SNAP($\infty$)-BK over rate of \ac{BK} between variable pairs.
    The shadow area denotes standard deviation.}
    \label{fig:shd}
\end{figure*}

However, as PC-BK and SNAP($\infty$)-BK do recover MPDAGs, for completeness in this section we report their SHD with 100 nodes and otherwise under the same finite data settings as in \Cref{fig:over_bk}.
In \Cref{fig:shd} we measure the SHD compared to the true MPDAG for PC-BK, and compared to the induced subgraph of the true MPDAG over the b-possible ancestors of the targets in the true MPDAG for SNAP($\infty$)-BK, so these results are not directly comparable across the methods, but these results can be used to evaluate how the SHD changes based on the \ac{BK} rate.
As expected, our results show a clear trend of decreasing (improving) SHD with increasing rate of \ac{BK} for both methods in both data settings.

\section{Comparison to Post-Processing}
\label{sec:post_proc}

\begin{figure*}[ht]
    \centering
    \includegraphics[width=.8\linewidth]{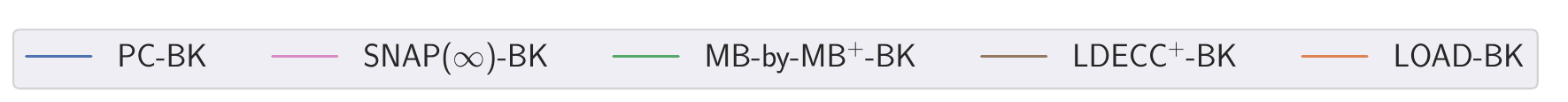}
    \begin{subfigure}[b]{\linewidth}
        \begin{subfigure}[b]{0.32\linewidth}
            \centering
            \caption*{ $\quad$ d-separation tests}
            \includegraphics[width=\linewidth]{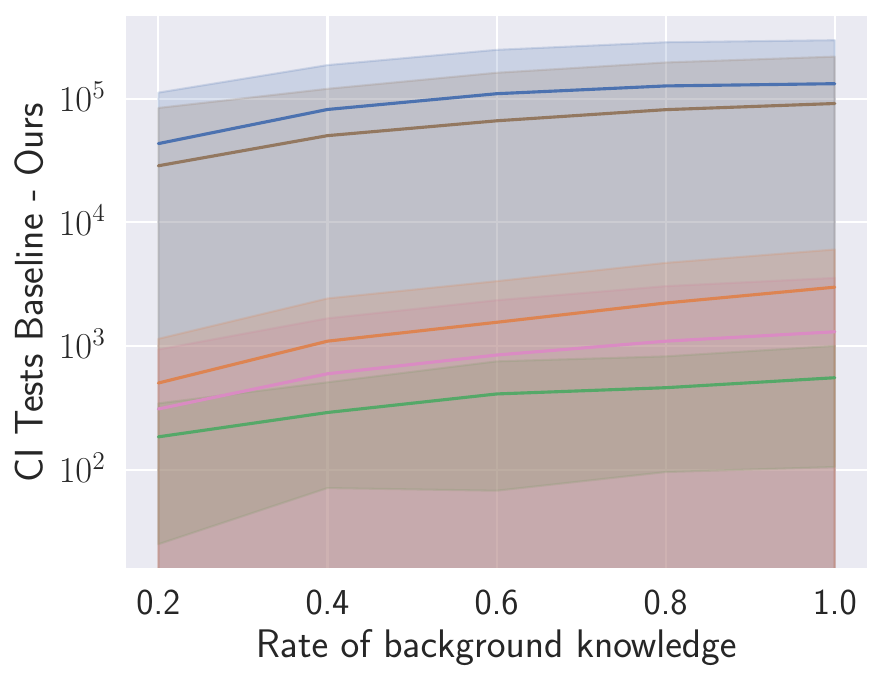}
            \includegraphics[width=\linewidth]{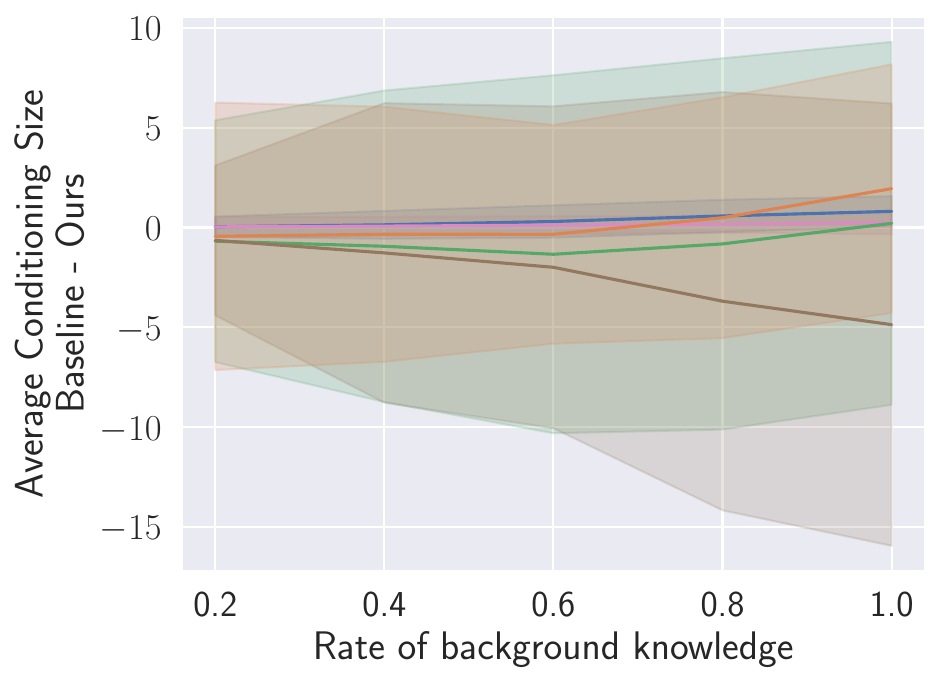}
            \includegraphics[width=\linewidth]{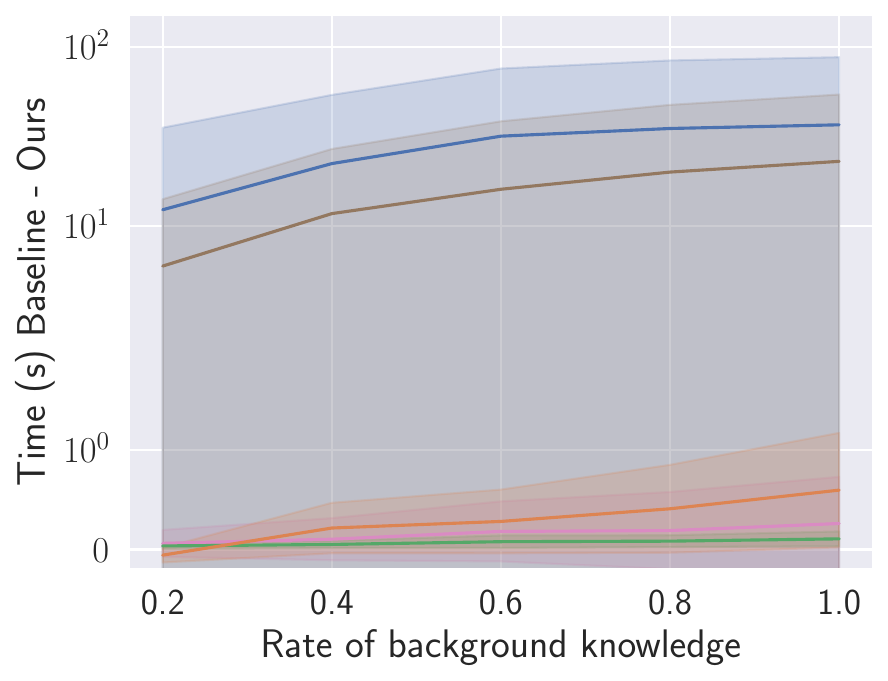}
            \includegraphics[width=\linewidth]{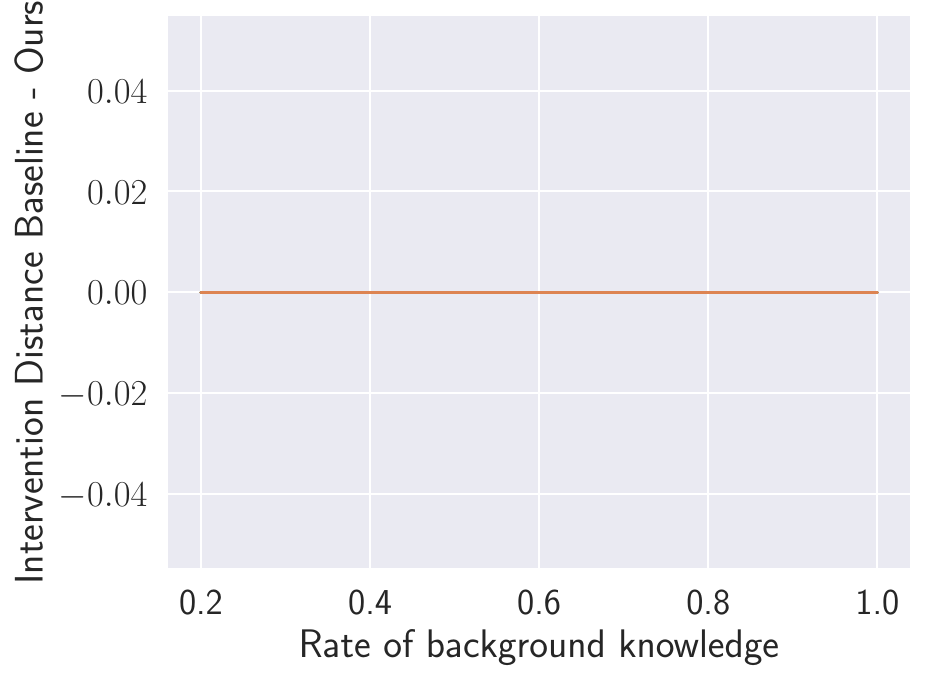}
        \end{subfigure}
        \begin{subfigure}[b]{0.32\linewidth}
            \centering
            \caption*{$\quad$ Fisher-Z tests}
            \includegraphics[width=\linewidth]{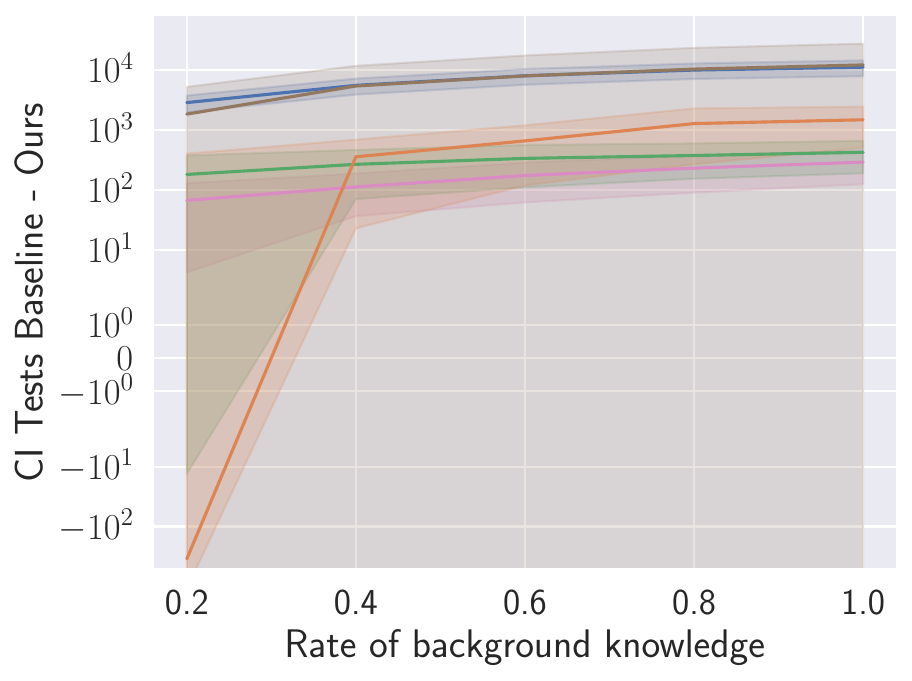}
            \includegraphics[width=\linewidth]{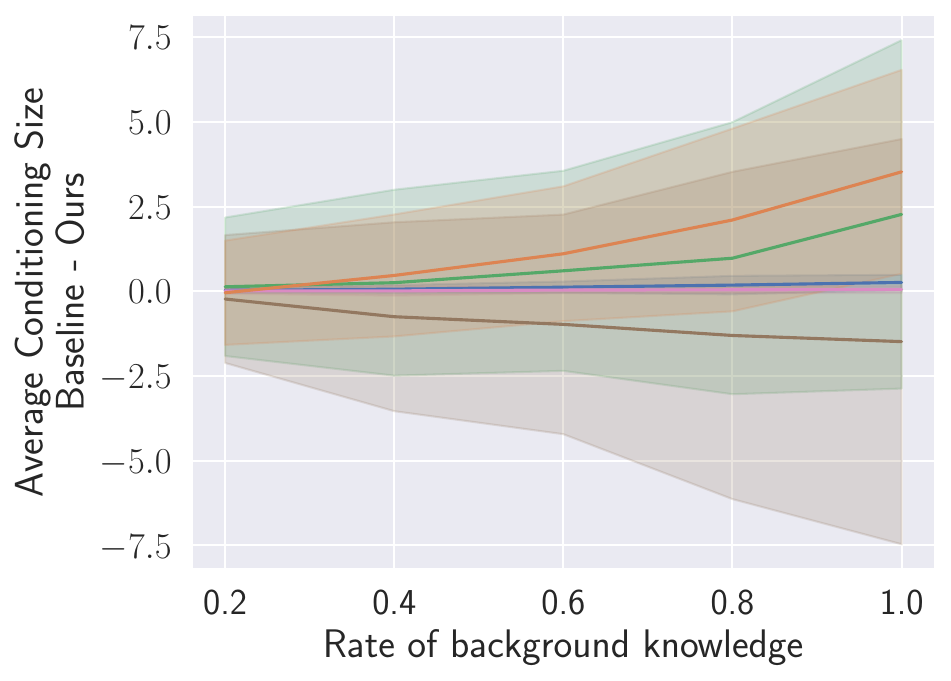}
            \includegraphics[width=\linewidth]{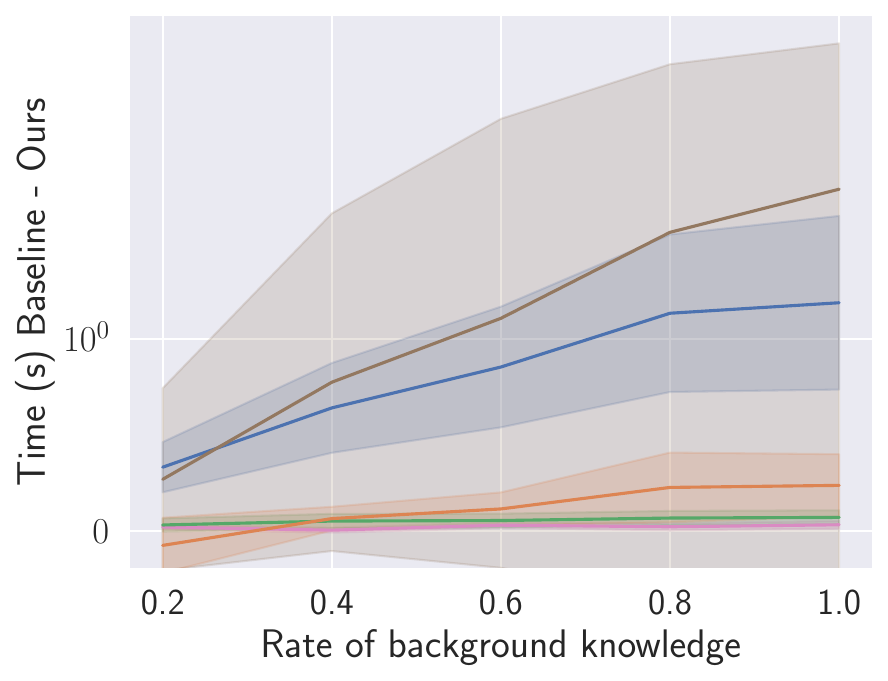}
            \includegraphics[width=\linewidth]{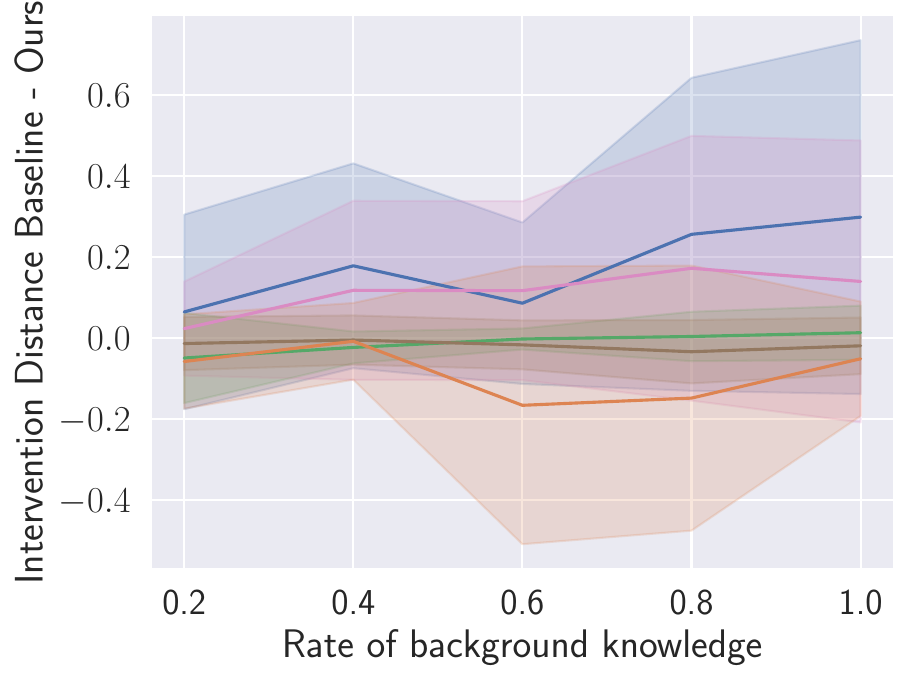}
        \end{subfigure}
        \begin{subfigure}[b]{0.32\linewidth}
            \centering
            \caption*{$G^2$ tests}
            \includegraphics[width=\linewidth]{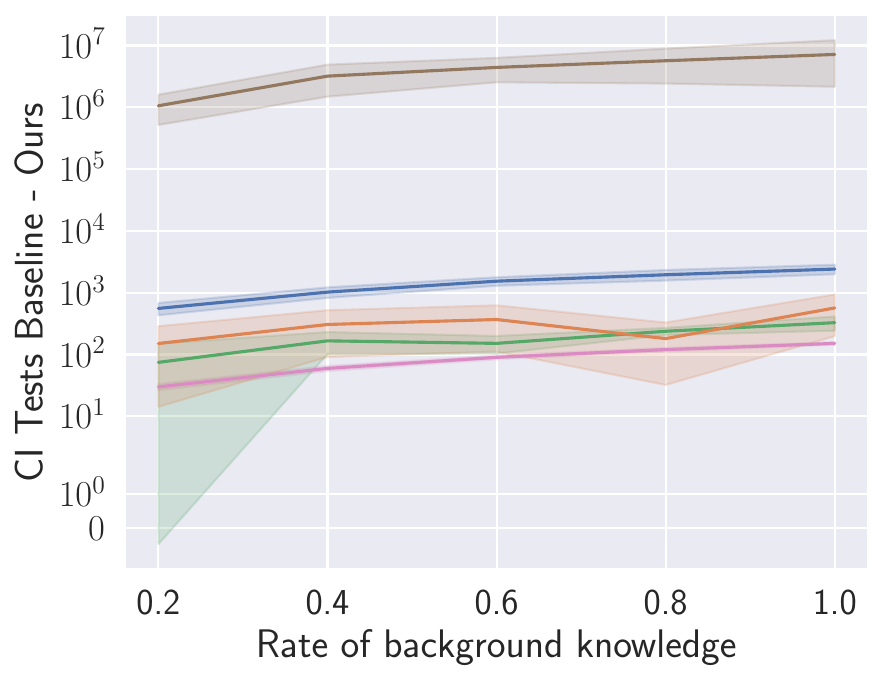}
            \includegraphics[width=\linewidth]{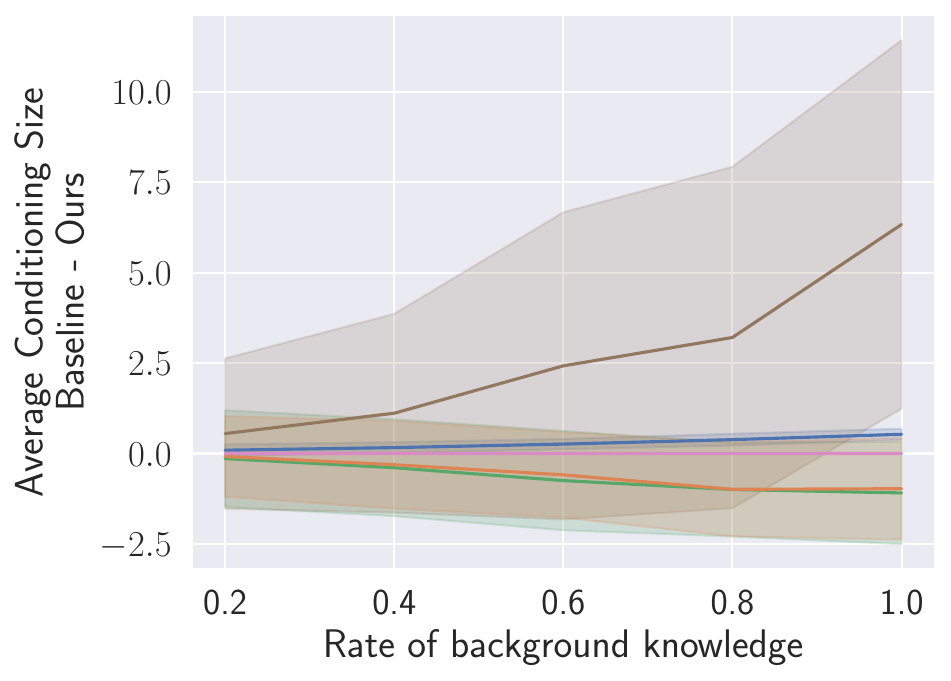}
            \includegraphics[width=\linewidth]{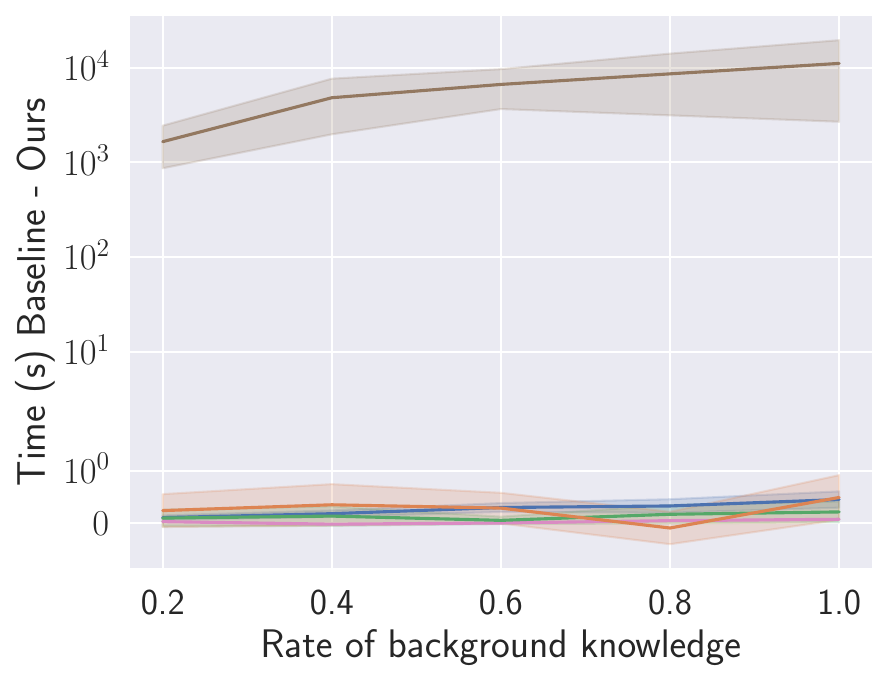}
            \includegraphics[width=\linewidth]{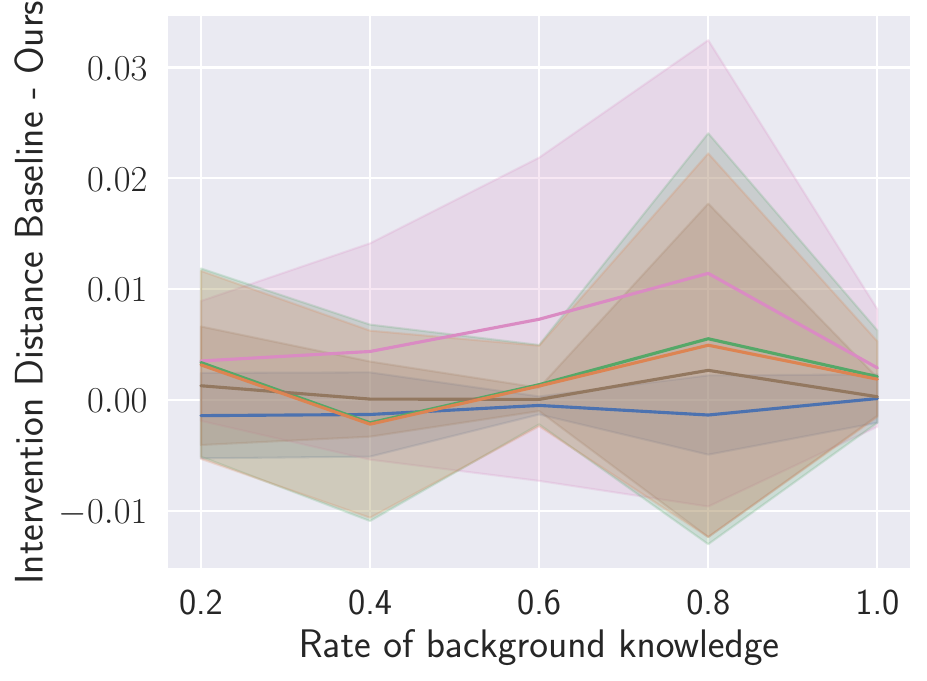}
        \end{subfigure}
    \end{subfigure}
    \caption{Results for showing the \emph{improvement} of our approach compared to the post-processing baselines described in \Cref{sec:post_proc}, quantified as the difference between the baselines and our approach, in the same setting as \Cref{fig:over_bk} but on 100 nodes.
    The shadow area denotes standard deviation. 
    Higher is better for each metric.}
    \label{fig:post_proc}
\end{figure*}

While we propose methods for integrating the \ac{BK} as soon as possible, so either before or during the learning process, our work is compatible with other approaches for \ac{BK} in postprocessing, such as \citep{perkovic2017interpreting} or \citep{zheng2026local}. So in principle one could use different sets of \ac{BK} for our approaches and for postprocessing. If we consider perfect \ac{BK} and oracle d-separation tests, the outputs should be the same, as shown by our theoretical results. In terms of pre-processing, we already integrate as much as possible of the knowledge before starting the learning procedure and show that doing so naively (like in the current implementations described in \Cref{sec:gaps_code}) might create errors.

In general, the computational efficiency of using \ac{BK} in postprocessing would be very similar, if not in some cases identical, to the results shown in \Cref{fig:over_bk} for a \ac{BK} rate of zero.
For PC and SNAP($\infty$), integrating \ac{BK} in postprocessing with the approach by \citet{perkovic2017interpreting} gives exactly the same CI tests as these results, since there are no CI tests in the postprocessing. The execution time is also dominated by the tests, so this is quite similar. For local methods like MB-by-MB and LOAD, we can follow \citep{zheng2026local} for postprocessing to orient the local structures, which also gives similar results. For LDECC, we can simply add or remove parents and children to its output based on the \ac{BK}. While this strategy is sound, it may identify fewer children than our approach, as using orientations during discovery can orient additional children at line 19 of LDECC (Fig. 5 in \citep{gupta2023local}).

To make this point even clearer, we provide empirical results, showing the improvement of our methods in comparison with using \ac{BK} in postprocessing.
In \Cref{fig:post_proc} we report the \emph{improvement} of our approach compared to the aforementioned baselines, quantified as the difference between the baselines and our approach, in the same setting as \Cref{fig:over_bk}, where higher is better. The results show that as the rate of \ac{BK} increases our methods increasingly reduce the number of CI tests compared to post-processing. Similarly, they increasingly reduce the average conditioning set size for PC-BK and SNAP($\infty$)-BK, while for MB-by-MB$^+$-BK, LDECC$^+$-BK and LOAD-BK this depends on the data setting and the rate of \ac{BK}. Similar to the number of tests, the running time is also consistently reduced. While our main focus is on computational efficiency, the intervention distance in finite sample data seems to improve for PC-BK and SNAP($\infty$)-BK in Fisher-Z tests, while the trend is less clear for other methods and settings.

Additionally, in \Cref{fig:post_proc_noisy} we also compare how sensitive the two stages of applying \ac{BK} are to noisy data by measuring the difference in intervention distance over different rates of \ac{BK} error (as in \Cref{fig:bk_over_error}).
This figure shows that the intervention distance gets worse with a higher rate of \ac{BK} error for both types of methods (first two rows), but as the rate of \ac{BK} error increases, for $G^2$ tests there seems to be a higher improvement in intervention distance for our methods vs postprocessing (last row), while for Fisher-Z this is less clear.

\begin{figure*}[t]
    \centering
    \includegraphics[width=.8\linewidth]{experiments/legend.pdf}
    \begin{subfigure}[b]{\linewidth}
        \centering
        \begin{subfigure}[b]{0.35\linewidth}
            \centering
            \caption*{$\quad$ Fisher-Z tests}
            \includegraphics[width=\linewidth]{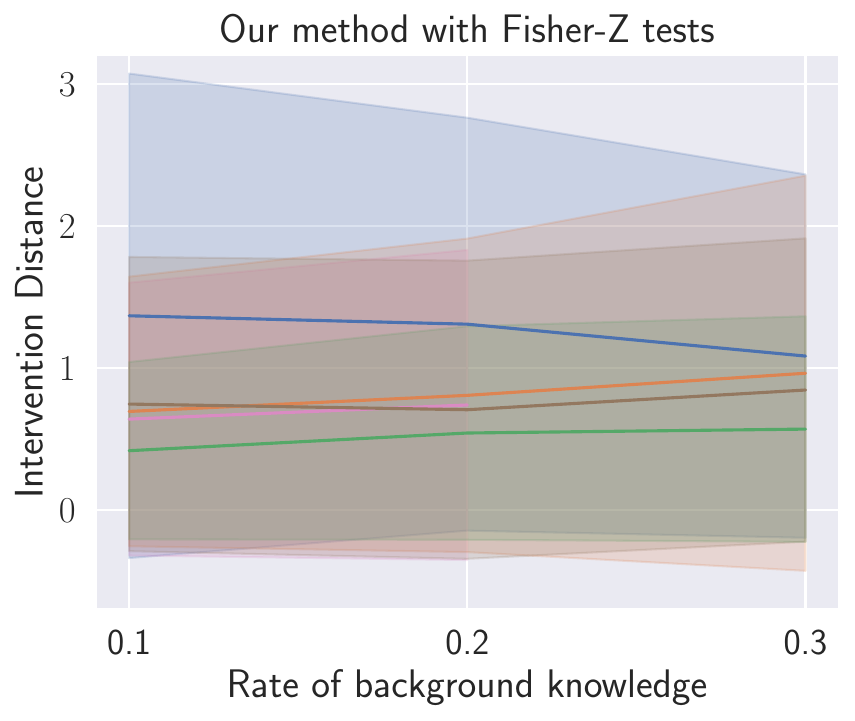}
            \includegraphics[width=\linewidth]{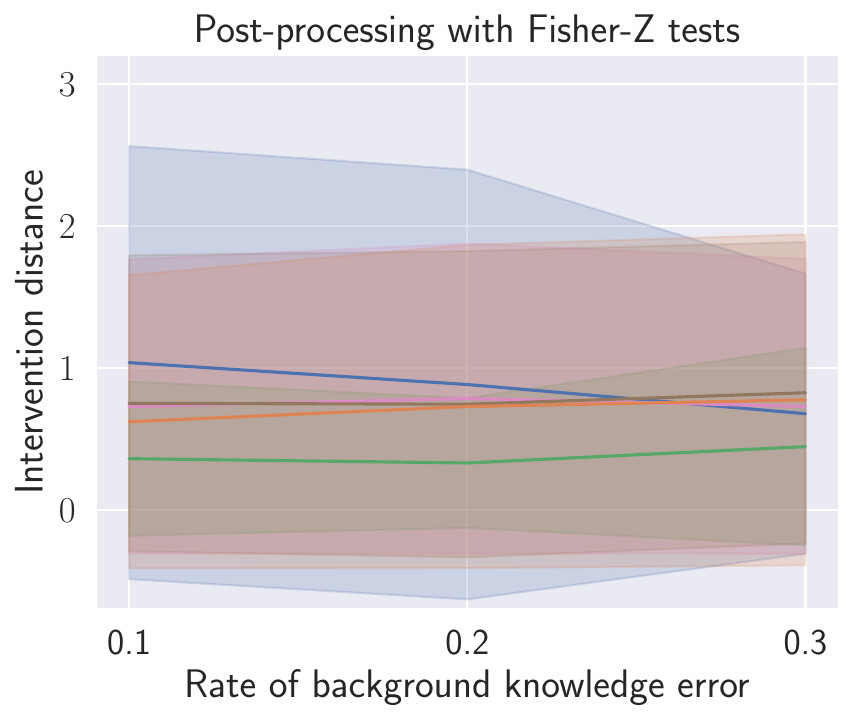}
            \includegraphics[width=\linewidth]{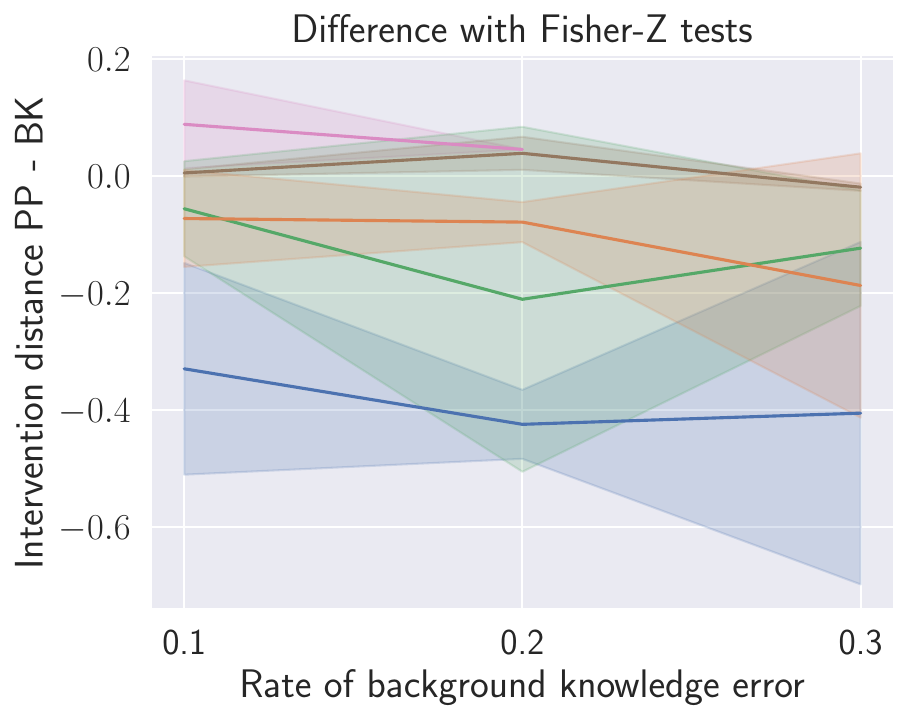}
        \end{subfigure}
        \hspace{8mm}
        \begin{subfigure}[b]{0.37\linewidth}
            \centering
            \caption*{$G^2$ tests}
            \includegraphics[width=\linewidth]{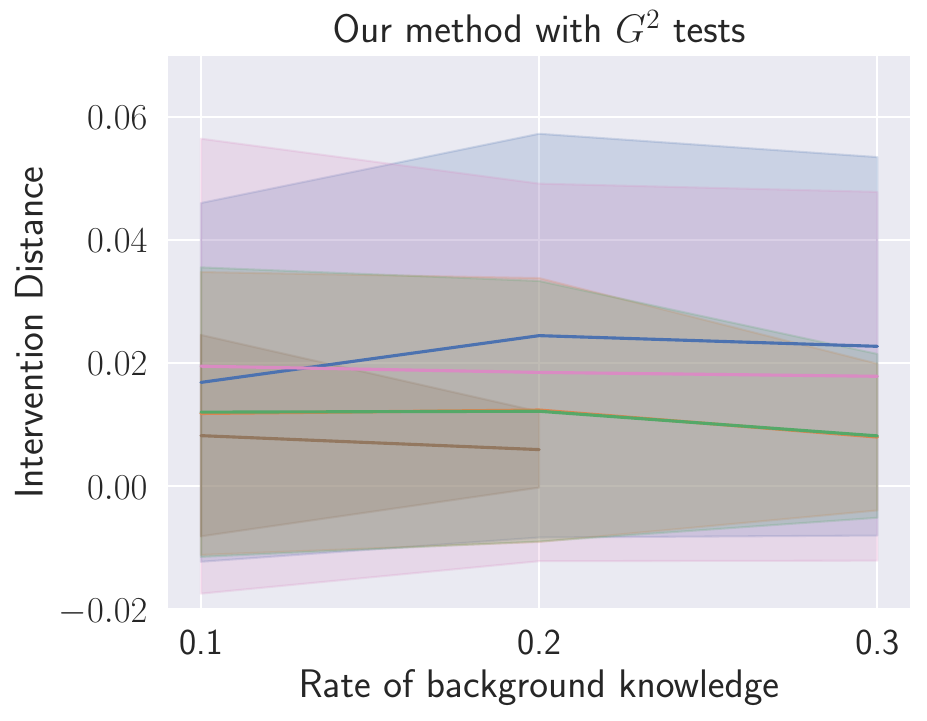}
            \includegraphics[width=\linewidth]{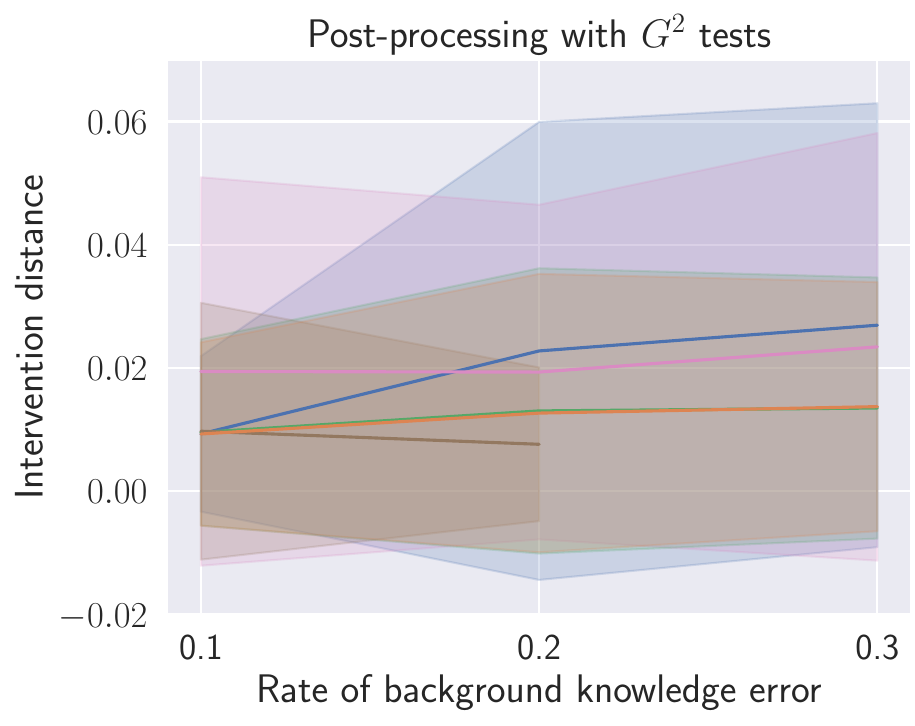}
            \includegraphics[width=\linewidth]{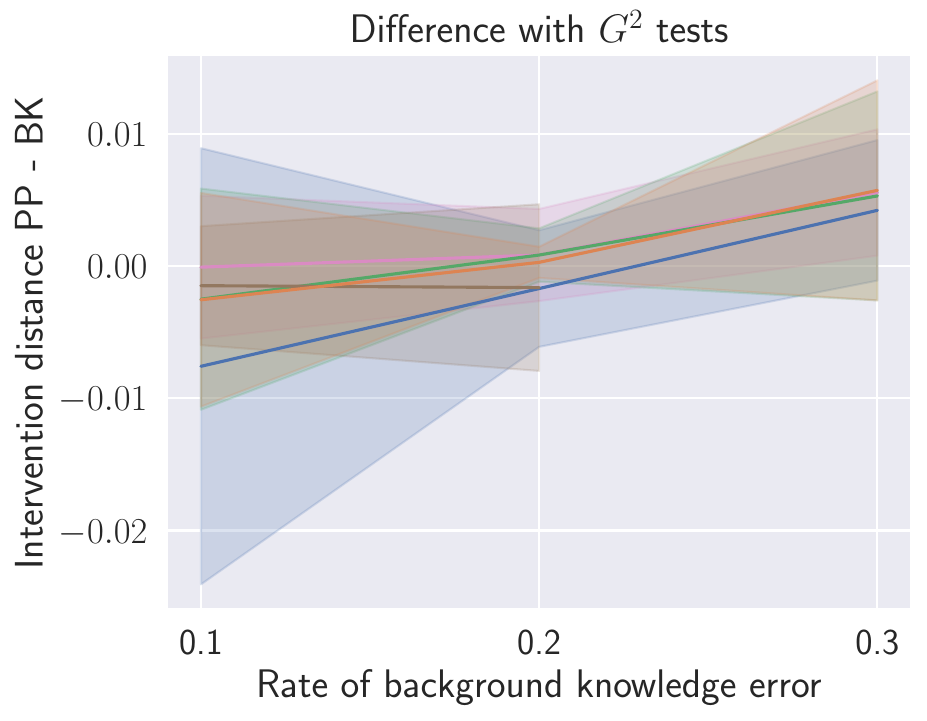}
        \end{subfigure}
    \end{subfigure}
    \caption{
    Results comparing our approach to post-processing baselines described in \Cref{sec:post_proc} under different amounts of background knowledge error in the same setting as \Cref{fig:over_bk} but on 100 nodes.
    The shadow area denotes standard deviation. 
    Higher is better for each metric.}
    \label{fig:post_proc_noisy}
\end{figure*}

\section{Comparison to Unsound Methods}
\label{sec:unsound}

As discussed in \Cref{sec:method}, current implementations of \ac{BK} in PC are \emph{not sound} and can only handle a \emph{subset} of pairwise \ac{BK} during discovery.
In particular:
\begin{compactitem}
    \item \texttt{pcalg} \citep{kalisch2012pcalg} can only integrate \ac{BK} on gaps and adjacencies, but \emph{not orientations}, where they become fixed in the skeleton.
    Fixing gaps directly leads to incorrect v-structure orientations, i.e., the \emph{soundness issues} we discuss in \Cref{sec:method}.
    \item \texttt{causal-learn} \citep{zheng2024causal} can only integrate gaps and orientations, i.e., it cannot handle \ac{BK} that two variables are adjacent with unknown orientation. Similarly to \texttt{pcalg}, gaps and adjacencies get fixed in the skeleton, which leads to the \emph{same soundness issues}. The orientations are used to obtain an MPDAG in the orientation phase, \emph{after} all CI tests have been performed, so they do not improve the computational efficiency of the method.
    \item \texttt{tetrad} \citep{ramsey2018tetrad}, similarly \texttt{causal-learn}, can only handle gaps and orientations which get fixed in the skeleton and lead to the \emph{same soundness issues}. Nonetheless, \texttt{tetrad} is the only implementation that utilizes \ac{BK} during graph learning by avoiding possible-parents in separating sets, which we also discuss in \Cref{sec:method}.
\end{compactitem}

\begin{figure}[t]
\centering
\includegraphics[width=.9\linewidth]{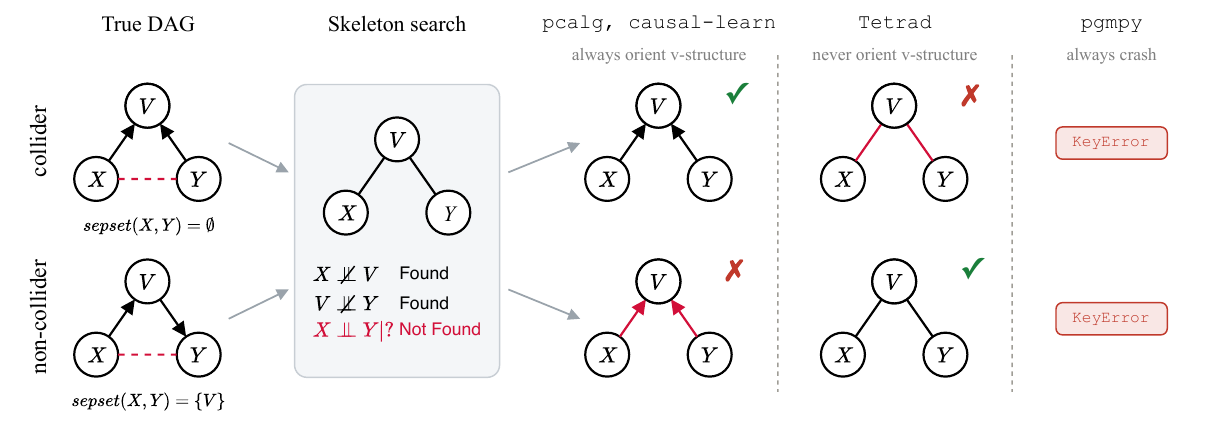}
\caption{
Given BK $\mathcal{B} = \{X \gap Y\}$, current implementations of PC \citep{kalisch2012pcalg, ramsey2018tetrad, zheng2024causal, ankan2024pgmpy} completely skip tests to find a separating set for $X$ and $Y$, resulting in incorrect orientations or crashes.
}
\label{fig:unsound_pc_example_long}
\end{figure}

These soundness issues are based on the intuitive but wrong assumption that \ac{BK} should always take priority against CI tests.
\Cref{fig:unsound_pc_example_long} provides minimal examples on how these implementations fail.
Additionally, none of these implementations use all of the \ac{BK} we consider (gaps, adjacencies and orientations) during the causal discovery process, which we also prove that we do in a \emph{sound and complete} way.
We also note that using order-invariant versions of PC \citep{colombo2014order}  do not solve this problem, since the issue does not stem from the order of the CI tests.

\subsection{Code Examples for Handling Gaps}
\label{sec:gaps_code}

The following examples shows how, for the true DAG $X \to Y \to Z$ with consistent background knowledge $\mathcal{B} = \{X \gap Z\}$ and oracle CI tests, the PC algorithm is implemented by different causal discovery libraries.

In the \texttt{pcalg} library \citep{kalisch2012pcalg}, PC returns an incorrect CPDAG $X \to Y \gets Z$ instead of the correct CPDAG $X - Y - Z$.
\begin{minted}{R}
library(pcalg)

# Define true DAG as X -> Y -> Z
dag <- new("graphNEL",
    nodes = c("X", "Y", "Z"),
    edgeL = list("X" = "Y", "Y" = "Z", "Z" = character(0)),
    edgemode = "directed"
)
# Define background knowledge with forbidden edge between X and Z
fixedGaps <- matrix(FALSE, nrow = 3, ncol = 3)
fixedGaps[1, 3] <- TRUE
fixedGaps[3, 1] <- TRUE
# Run PC algorithm with background knowledge
pc_res <- pc(
    suffStat = list(g = dag, jp = RBGL::johnson.all.pairs.sp(dag)),
    indepTest = pcalg::dsepTest,
    alpha = 0.05,
    labels = c("X", "Y", "Z"),
    fixedGaps = fixedGaps
)
# Testing if result is X -- Y -- Z fails!
stopifnot(
    "Y" %in% graph::adj(pc_res@graph, "X") &&
    "X" %in% graph::adj(pc_res@graph, "Y") &&
    "Y" %in% graph::adj(pc_res@graph, "Z") &&
    "Z" %in% graph::adj(pc_res@graph, "Y")
)
\end{minted}

The following code snippet shows the same behavior in the \texttt{causal-learn} library \citep{zheng2024causal}.
\begin{minted}{python}
from causallearn.graph.GraphClass import CausalGraph
from causallearn.search.ConstraintBased.PC import pc
from causallearn.utils.PCUtils.BackgroundKnowledge import BackgroundKnowledge
import networkx as nx
import numpy as np

# Define true DAG as 0 -> 1 -> 2
dag = nx.DiGraph([(0, 1), (1, 2)])
# Define background knowledge with forbidden edge between 0 and 2
nodes = CausalGraph(3).G.get_nodes()
bk = BackgroundKnowledge()
bk.add_forbidden_by_node(nodes[0], nodes[2])
bk.add_forbidden_by_node(nodes[2], nodes[0])
# Run PC algorithm with background knowledge
pc_res = pc(
    data=np.zeros((10, 3)),
    alpha=0.05,
    indep_test="d_separation",
    true_dag=dag,
    background_knowledge=bk,
)
# Testing if result is 0 -- 1 -- 2 fails!
assert pc_res.G.is_undirected_from_to(
    nodes[0], nodes[1]
) and pc_res.G.is_undirected_from_to(nodes[1], nodes[2])
\end{minted}

In the \texttt{pgmpy} library \citep{ankan2024pgmpy}, PC fails (i.e., raises an exception) if the \ac{BK} is utilized during skeleton search (\texttt{enforce\_expert\_knowledge=True}) due to the missing separating set $sepset(X,Z)$.
Otherwise, PC returns the correct CPDAG $X - Y - Z$. 
\begin{minted}{python}
import networkx as nx
from pgmpy.base import DAG
from pgmpy.estimators import ExpertKnowledge, PC
import pandas as pd


# Define true DAG as X -> Y -> Z
dag = DAG([("X", "Y"), ("Y", "Z")])
data = pd.DataFrame({"X": [0] * 10, "Y": [0] * 10, "Z": [0] * 10})
# Define background knowledge with forbidden edge between X and Z
bk = ExpertKnowledge(forbidden_edges=[("X", "Z"), ("Z", "X")])
# Run PC algorithm with background knowledge
ci_test = lambda x, y, S, **kwargs: float(
    nx.is_d_separator(dag, x, y, set(S) if S else set())
)
pc_res = PC(data).estimate(
    ci_test=ci_test, expert_knowledge=bk, enforce_expert_knowledge=True
)
\end{minted}

Unlike the previous implementations of PC, the \texttt{tetrad} library \citep{ramsey2018tetrad} instead fails to orient v-structures $X \to Y \gets Z$ when the consistent \ac{BK} $\mathcal{B} = \{X \gap Z\}$ is provided.
\begin{minted}{java}
package edu.cmu.tetrad;

import edu.cmu.tetrad.data.Knowledge;
import edu.cmu.tetrad.graph.*;
import edu.cmu.tetrad.search.Pc;
import edu.cmu.tetrad.search.test.IndependenceTest;
import edu.cmu.tetrad.search.test.MsepTest;

public class Main {
    public static void main(String[] args) {
        // Define true DAG as X -> Y <- Z
        Graph dag = GraphUtils.convert("X-->Y,Z-->Y");
        // Define background knowledge with forbidden edge between X and Z
        Knowledge knowledge = new Knowledge();
        knowledge.setForbidden("X", "Z");
        knowledge.setForbidden("Z", "X");
        // Run PC algorithm with background knowledge
        IndependenceTest ciTest = new MsepTest(dag);
        Pc pc = new Pc(ciTest);
        pc.setKnowledge(knowledge);
        Graph pcRes = null;
        try {
            pcRes = pc.search();
        } catch (InterruptedException e) {
            throw new RuntimeException(e);
        }
        // Testing if result is X -> Y <- Z fails!
        if (!pcRes.equals(dag)) {
            throw new AssertionError();
        }
    }
}
\end{minted}

\subsection{Empirical comparison}

Regardless of these issues, in \Cref{fig:over_bk_pre_proc} we provide an experimental comparison to \texttt{pcalg} (with postprocessing to get orientations) and \texttt{causal-learn}, showing their computational performance, but also SHD in comparison to our methods.
We use the same setting as in \Cref{fig:over_bk}, but only on 50 nodes because on larger graphs the running time exceeded our limit of 24 hours for each graph.
This happened for \texttt{causal-learn} with any rate of \ac{BK}, and most graphs using \texttt{pcalg} with lower rates of \ac{BK} (partly due to the implementations of these libraries).
We include stable PC implementations for completeness.
Since all PC variants output an MPDAG, we can also compare them using SHD.
We do not report results for \texttt{causal-learn} with \ac{BK} rate 0.4-1.0 due to the time limit.

\begin{figure*}[ht]
    \centering
    \includegraphics[width=\linewidth]{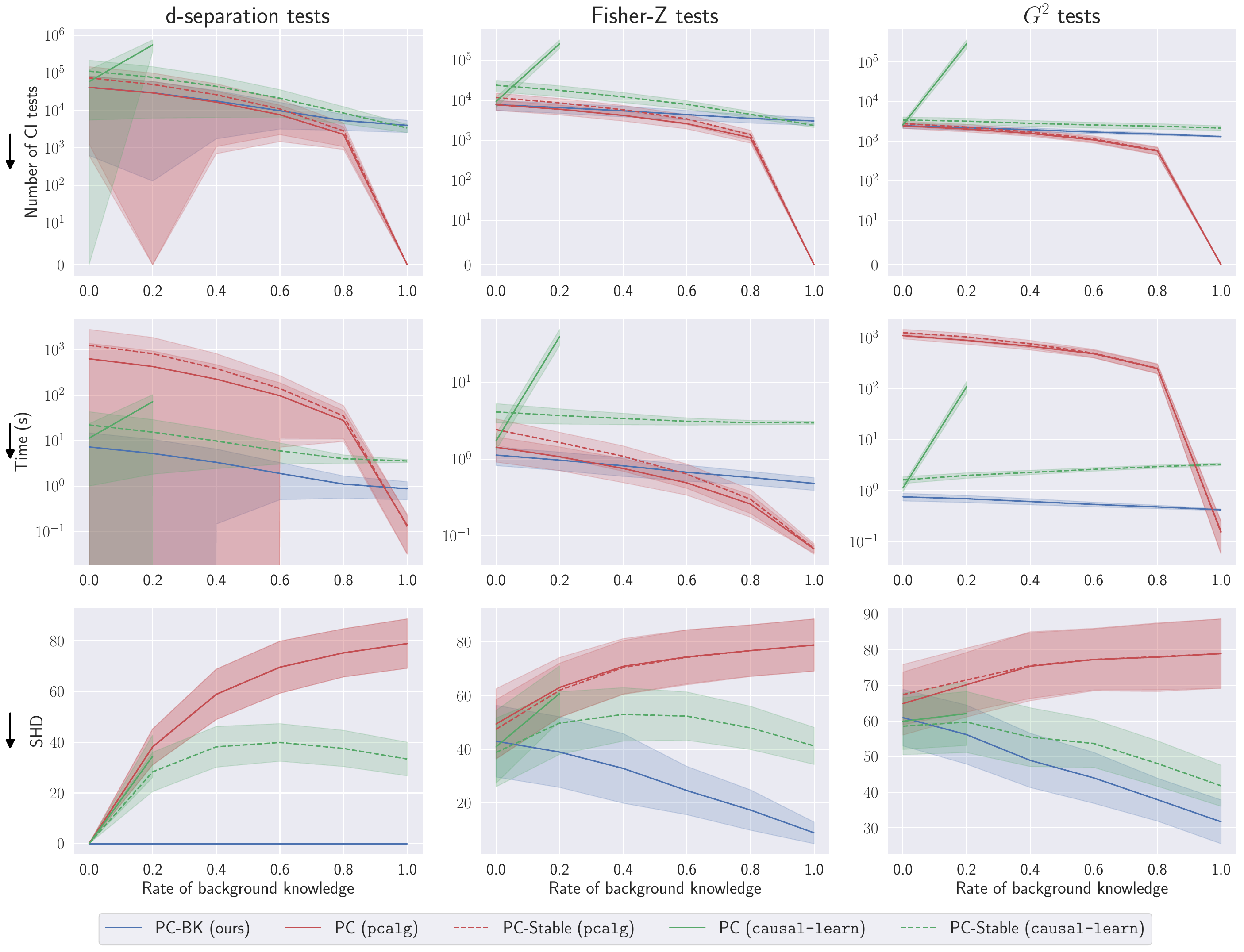}
    \caption{Results over rate of \ac{BK} between variable pairs, comparing our approach to pre-processing approach for PC available in \texttt{pcalg} and \texttt{causal-learn}.
    The shadow area denotes standard deviation. 
    }
    \label{fig:over_bk_pre_proc}
\end{figure*}

Our results show the advantages of our approach.
Since PC from \texttt{pcalg} skips all CI tests corresponding to known gaps, it can perform slightly fewer CI tests compared to our approach.
However, this comes with a \emph{much higher SHD} with any rate of \ac{BK}.
Crucially, \emph{all baselines have a non-zero SHD for d-separation} with \ac{BK} due to incorrectly handling gaps.

\section{Types of Background Knowledge}
\label{sec:bk_types_abl}
In this section we study the importance of different types of background knowledge, i.e., adjacencies, orientations and gaps.
We show our results for each algorithm separately.

Our results for PC-BK in \Cref{fig:bk_type_pc} show that all types of \ac{BK} are beneficial.
On almost all figures we can see the same trend of more \ac{BK} implying better performance both in terms of speed and quality.
The only exceptions are on binary data with $G^2$ CI tests, where more known gaps \ac{BK} increase computation time and known adjacencies increase intervention distance.

Results for SNAP($\infty$)-BK are shown in \Cref{fig:bk_type_snap}.
The number of CI tests steadily decreases over all \ac{BK} types and data domains except for gaps on binary data with $G^2$ CI tests.
As expected, orientation \ac{BK} consistently helps the most to push down intervention distance.
While running time generally decreases when utilizing all types of \ac{BK} at once over all domains, the \ac{BK} types separately show less consistency.

The results for MB-by-MB$^+$-BK in \Cref{fig:bk_type_mb_by_mb} show that all types of \ac{BK} over all domains, except for gaps, decrease computational costs.
As usual, orientation \ac{BK} helps the most with intervention distance, where results are less regular for other types of \ac{BK}.

LDECC$^+$-BK, shown in \Cref{fig:bk_type_ldecc}, shows the most clear trends over all metrics with increasing amount of \ac{BK} over all \ac{BK} types, metrics and data domains.

Results for LOAD-BK in \Cref{fig:bk_type_load} are very similar to MB-by-MB$^+$-BK, the local causal discovery sub-routine it uses.

\begin{figure*}[t]
    \centering
    \large{PC-BK algorithm}
    \includegraphics[width=\linewidth]{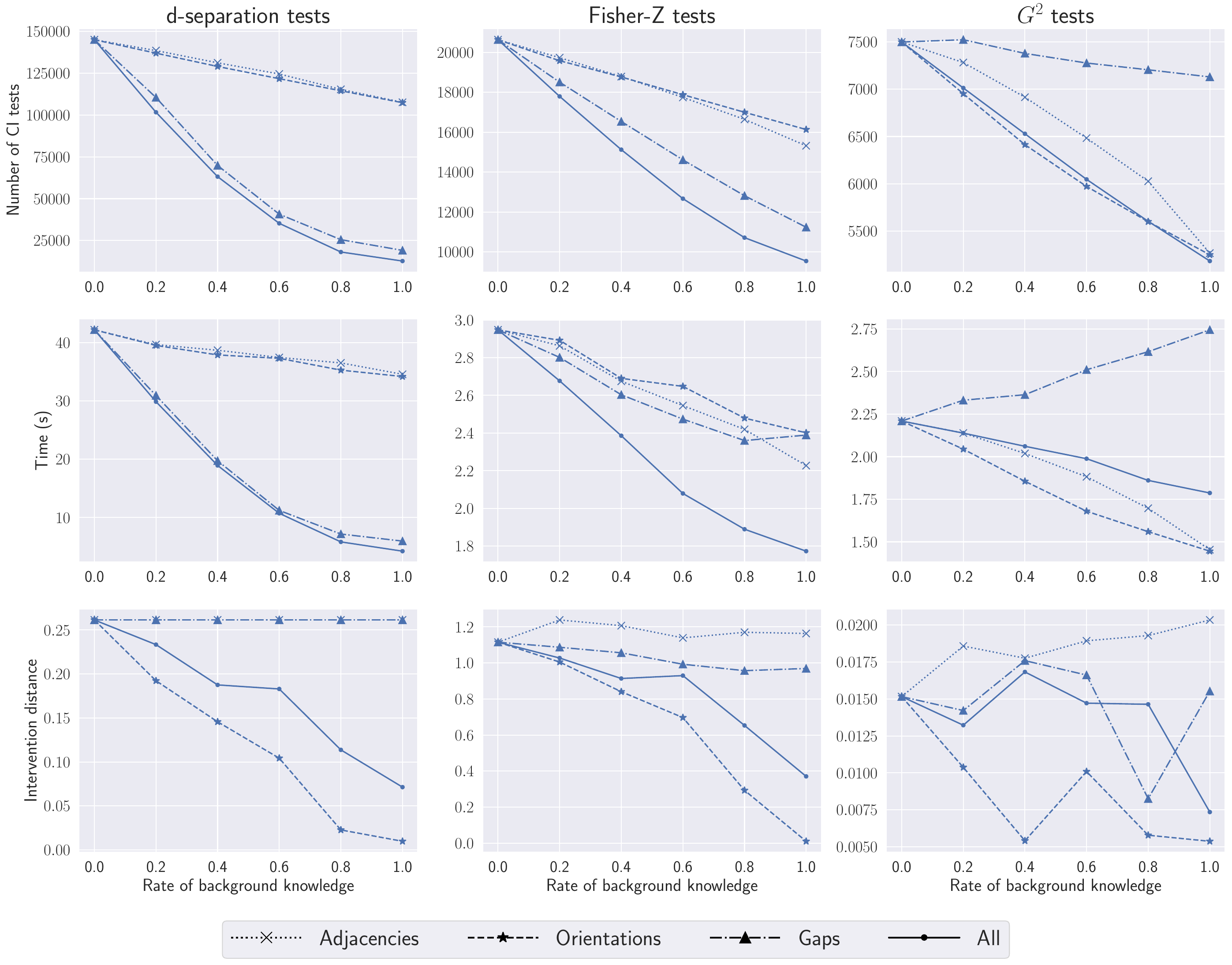}
    \caption{Results for the PC-BK algorithm over rates of different types background knowledge between variable pairs.}
    \label{fig:bk_type_pc}
\end{figure*}

\begin{figure*}[ht]
    \centering
    \large{SNAP($\infty$)-BK algorithm}
    \includegraphics[width=\linewidth]{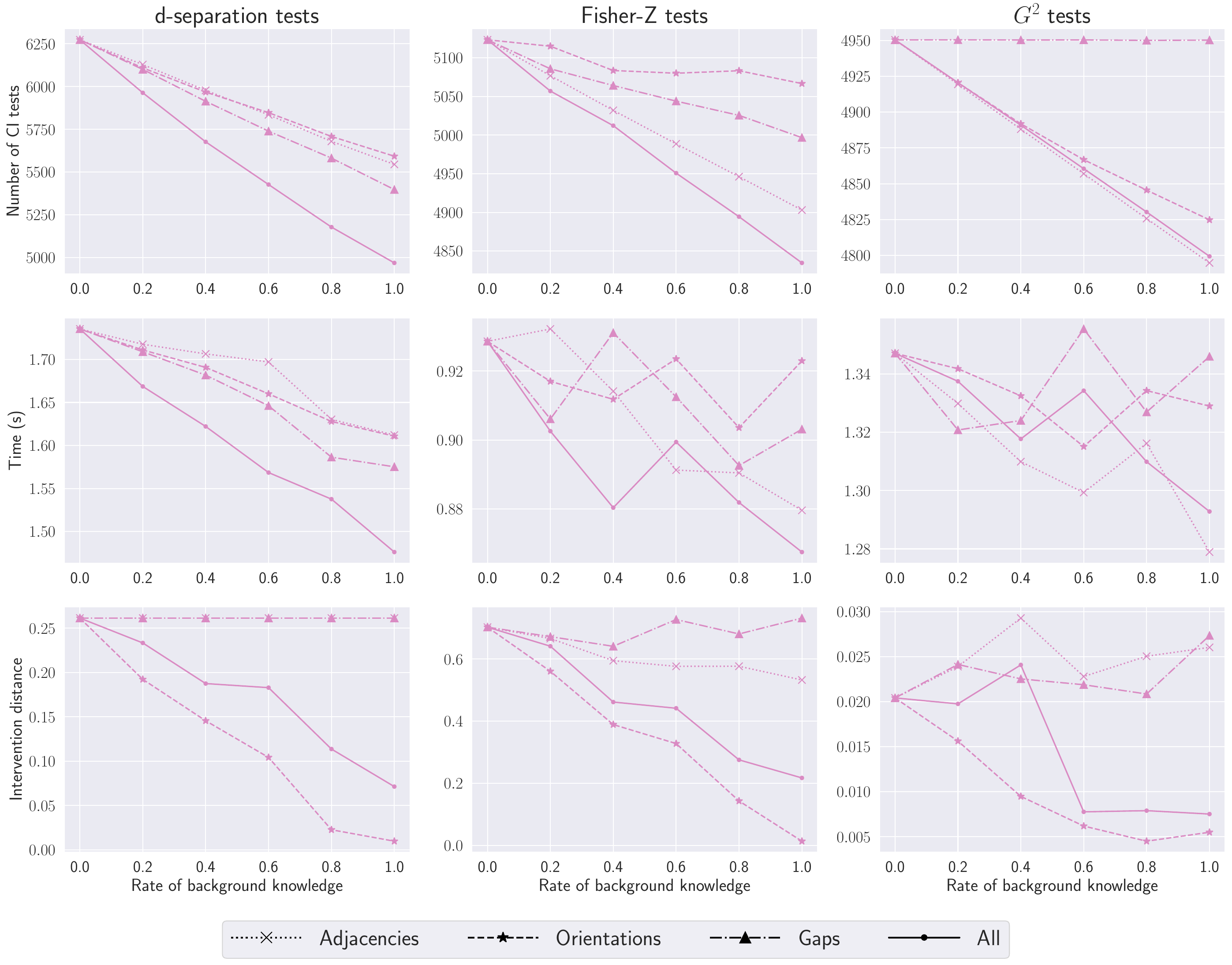}
    \caption{Results for the SNAP($\infty$)-BK algorithm over rates of different types background knowledge between variable pairs.}
    \label{fig:bk_type_snap}
\end{figure*}

\begin{figure*}[ht]
    \centering
    \large{MB-by-MB$^+$-BK algorithm}
    \includegraphics[width=\linewidth]{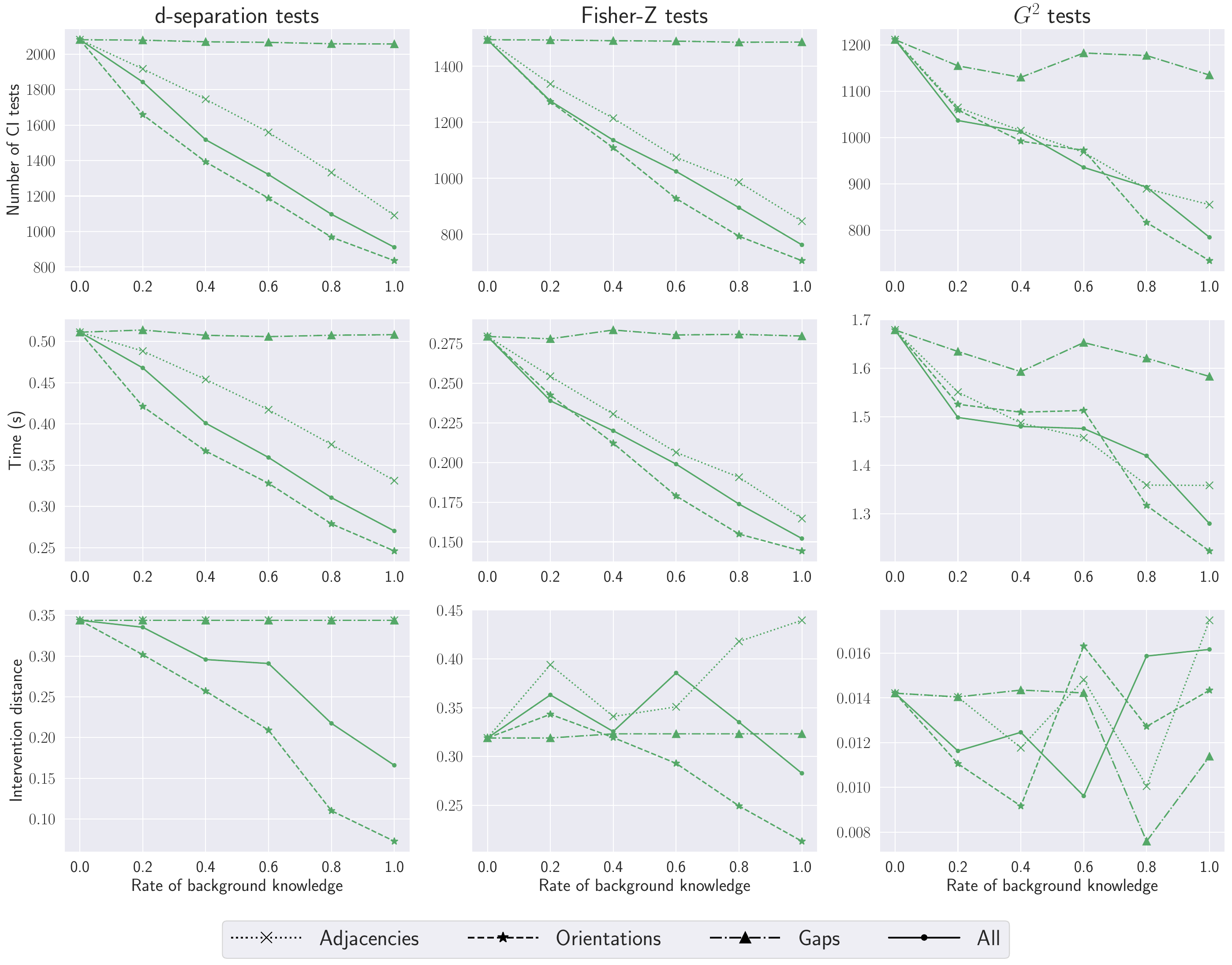}
    \caption{Results for the MB-by-MB$^+$-BK algorithm over rates of different types background knowledge between variable pairs.}
    \label{fig:bk_type_mb_by_mb}
\end{figure*}

\begin{figure*}[ht]
    \centering
    \large{LDECC$^+$-BK algorithm}
    \includegraphics[width=\linewidth]{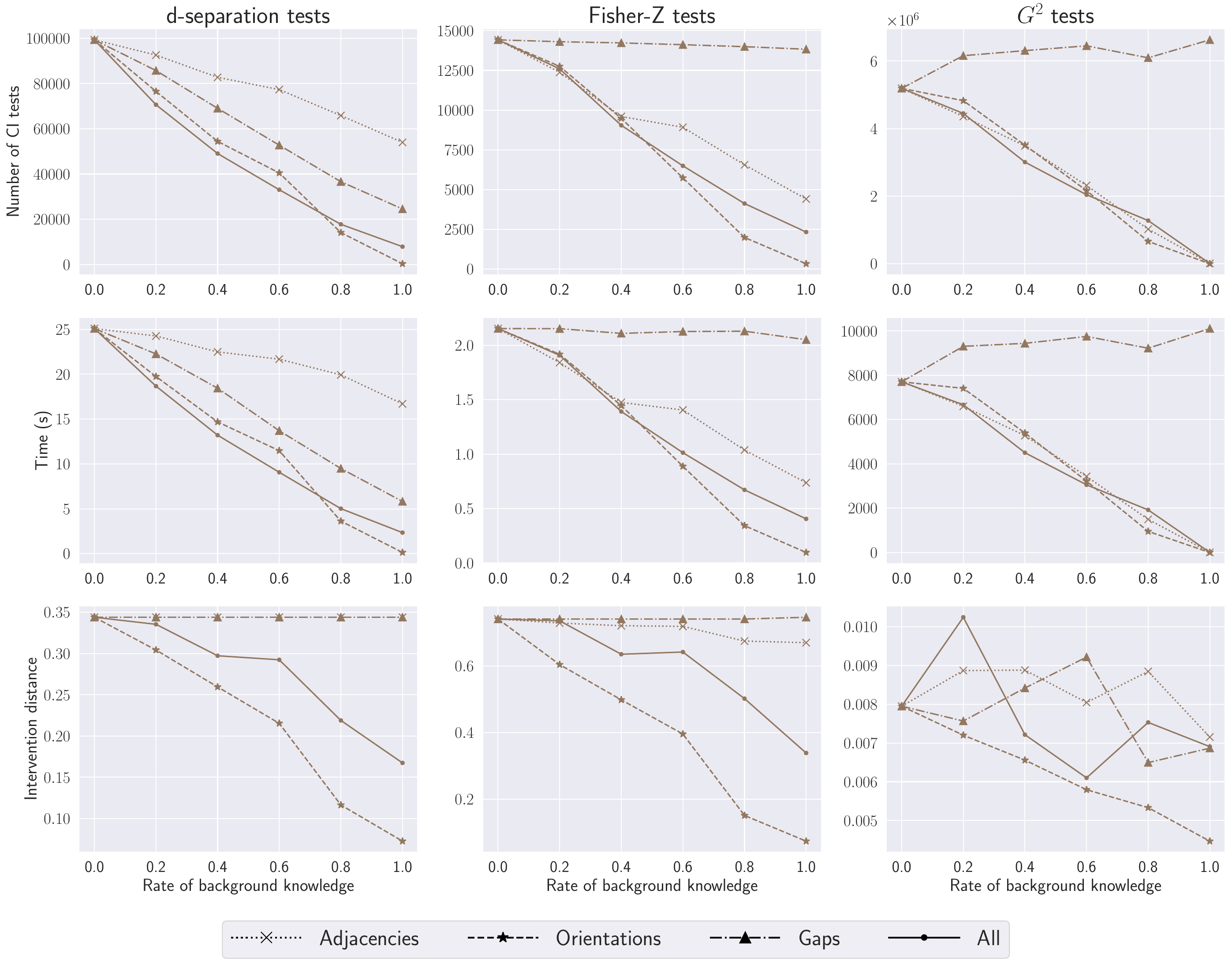}
    \caption{Results for the LDECC$^+$-BK algorithm over rates of different types background knowledge between variable pairs.}
    \label{fig:bk_type_ldecc}
\end{figure*}

\begin{figure*}[ht]
    \centering
    \large{LOAD-BK algorithm}
    \includegraphics[width=\linewidth]{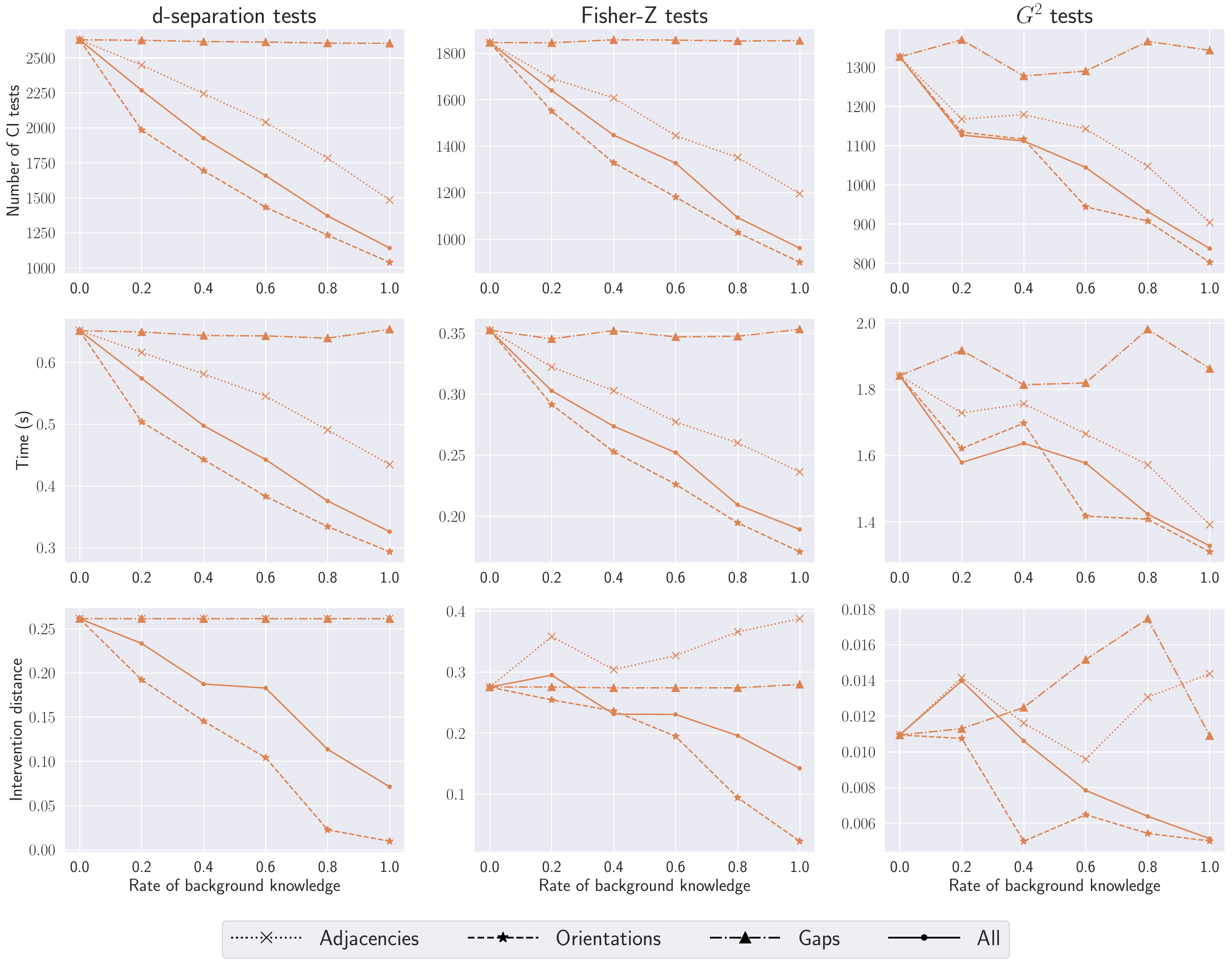}
    \caption{Results for the LOAD-BK algorithm over rates of different types background knowledge between variable pairs.}
    \label{fig:bk_type_load}
\end{figure*}

\section{Imperfect Background Knowledge}
\label{sec:bk_error}
In this section we experiment with imperfect \ac{BK} by deliberately introduce errors into it.
For a given pair $(X,Y) \in \mathcal{B}$, we generate an erroneous \ac{BK} as follows.
If $X$ and $Y$ are adjacent, then, we $\mathcal{B}$ will either indicate a gap between them, or an orientation that contradicts the true DAG, with uniform probability.
If $X$ and $Y$ are non-adjacent, then $\mathcal{B}$ will indicate that they are adjacent or oriented in one of two directions, with uniform probability over all three options.
Throughout our experiments we select half the amount of gaps to have errors so that the errors between \ac{BK} types are more balanced.

In \Cref{fig:fixed_error} we conduct experiments with 100 nodes and otherwise the same settings as for the main results shown in \Cref{fig:over_bk} but with a fixed rate of errors by select $10\%$ of adjacencies and orientations, and $5\%$ of gaps to have errors.
Our results show that different algorithms display different amounts of sensitivity to \ac{BK} errors.
Even with errors, the computational costs of MB-by-MB$^+$-BK, LDECC$^+$-BK and LOAD-BK are decreasing with more \ac{BK}, along with all domains, along with PC-BK with d-separation CI tests.
On the other hand, SNAP($\infty$)-BK is much more sensitive to errors, resulting in prohibitive execution times for higher rate of \ac{BK} across all domains.
With a higher rate of \ac{BK} and a constant rate of errors, the number of incorrect orientations grows.
Hence, most algorithms generally exhibit a growing intervention distance with larger variance.
Interestingly, different algorithms achieve different intervention distances with the exact same \ac{BK} errors and oracle d-separation CI tests, further demonstrating that error sensitivity is highly algorithm dependent.
Nonetheless, the achieved intervention distances stay reasonably low across algorithms and domains, with local methods MB-by-MB$^+$-BK, LDECC$^+$-BK and LOAD-BK, staying the most robust.
In general, the performance of most algorithms, except for SNAP($\infty$)-BK, degrades gradually rather than sharply, demonstrating that they are resilient against higher rates of \ac{BK} error.

\begin{figure*}[ht]
    \centering
    \includegraphics[width=\linewidth]{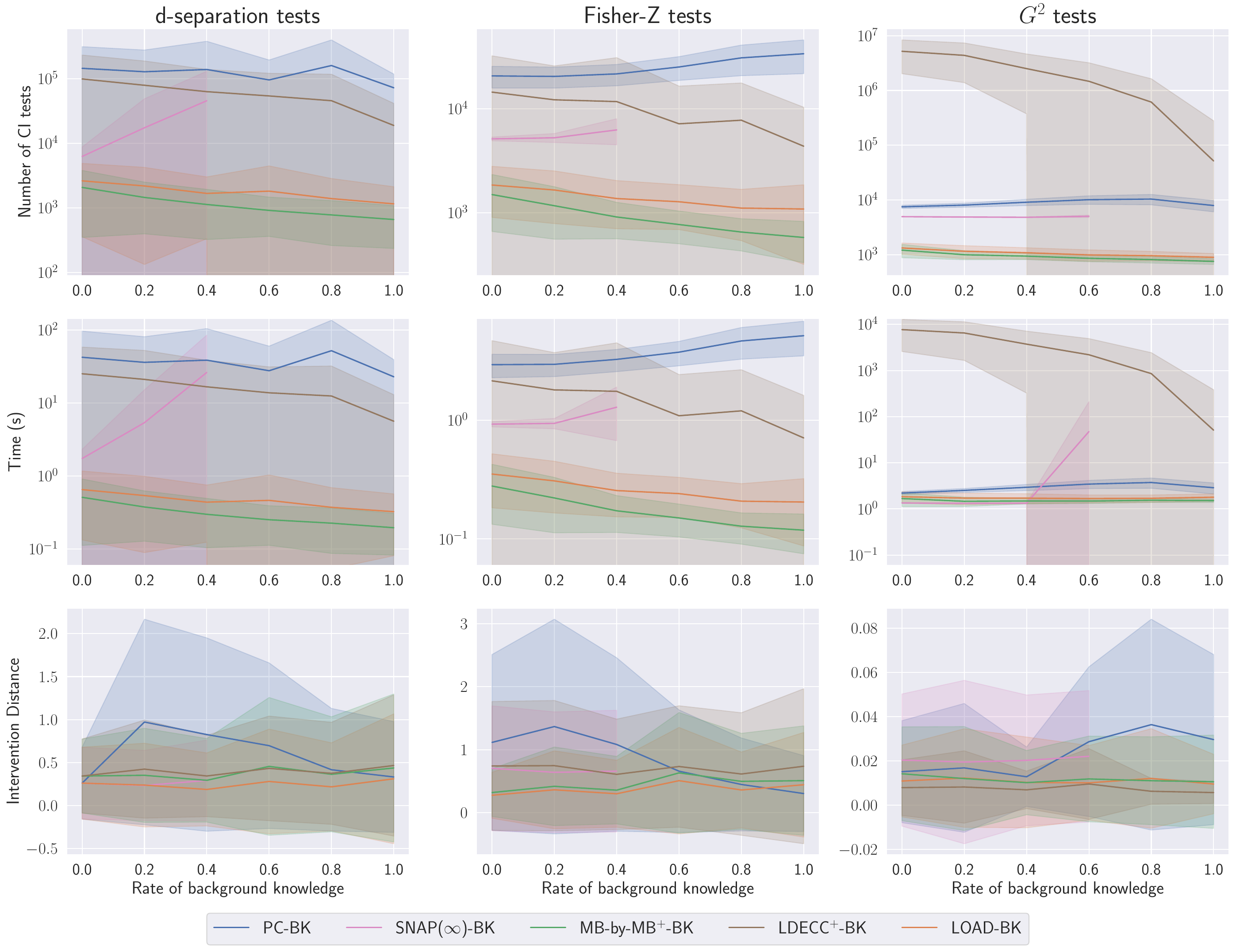}
    \caption{Results over rate of \ac{BK} between variable pairs, where \ac{BK} contains errors at a fixed rate of $10\%$ for adjacencies and orientations, and $5\%$ for gaps.
    The shadow area denotes standard deviation.}
    \label{fig:fixed_error}
\end{figure*}

In our experiments shown in \Cref{fig:bk_over_error}, we instead keep the \ac{BK} rate fixed at $0.2$ and change the error rate between $0.0$ to $0.3$, corresponding to errors in $30\%$ of adjacencies and orientations, and $15\%$ of gaps.
As expected, the number of CI tests and computation time required by most algorithms increase, but not substantially.
One exception is SNAP($\infty$)-BK which, similarly to the results with fixed error shown in \Cref{fig:fixed_error}, is much more sensitive to errors compared to other algorithms.
The other exception is MB-by-MB$^+$-BK, which gets slightly cheaper as the rate of errors grow.
As expected, intervention distance generally grows with higher error rate due to incorrect orientations.
PC-BK is a notable exception, as its intervention distance seems to drop for a \ac{BK} error of $0.3$ both with d-separation and Fisher-Z CI tests.
However, the variance of the intervention distance achieved by PC-BK is so large that this trend may be only due to noise.

\begin{figure*}[ht]
    \centering
    \includegraphics[width=\linewidth]{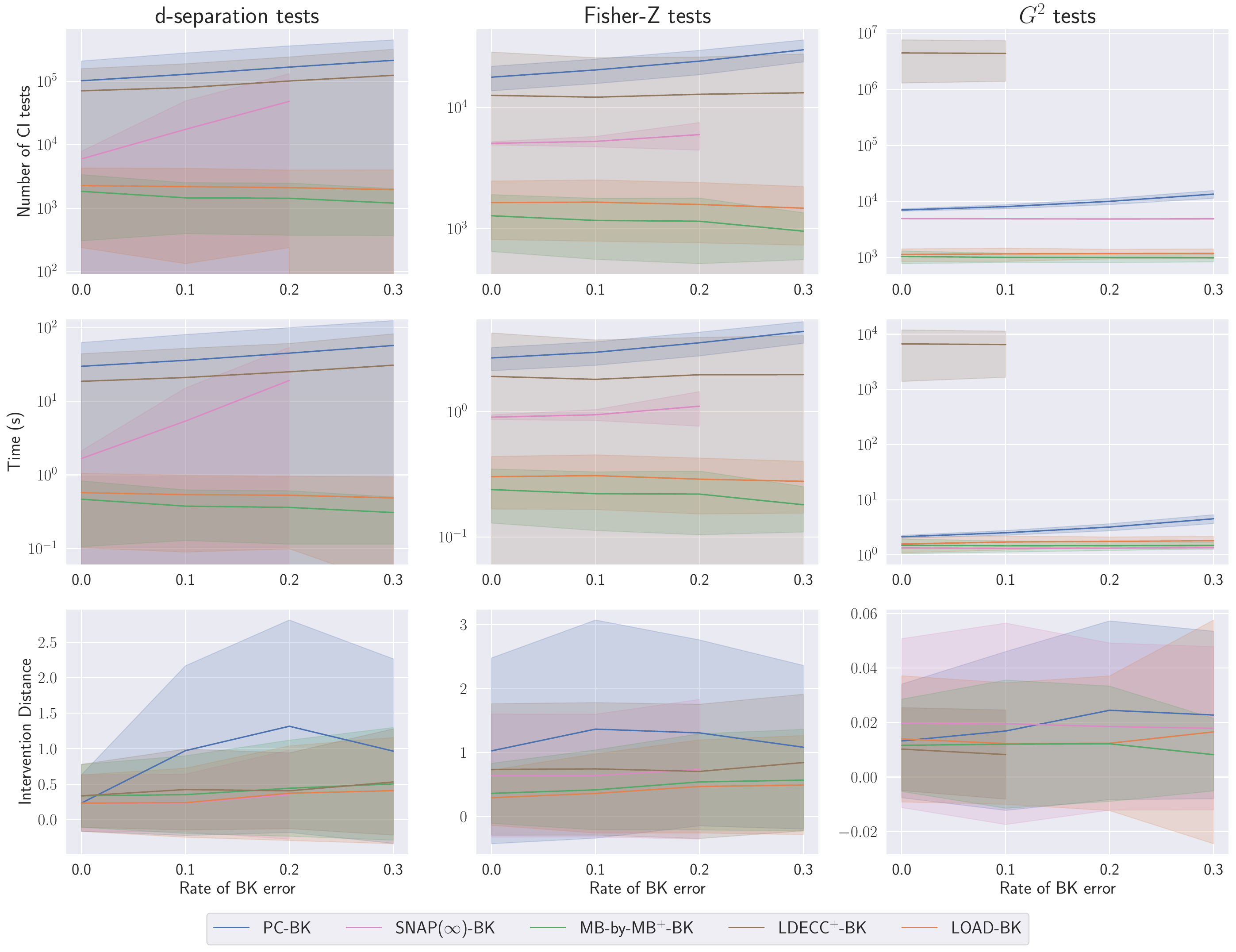}
    \caption{Results over rate of \ac{BK} error, with \ac{BK} rate kept fixed at 0.2.
    The error rate for gaps is always a half of the reported number (e.g., at a rate of 0.1, \ac{BK} contains errors for $10\%$ of adjacencies and orientations, and for $5\%$ of gaps) to keep the errors balanced between error types.
    The shadow area denotes standard deviation.}
    \label{fig:bk_over_error}
\end{figure*}

\section{Realistic Data}
\label{sec:semi-synth}

We perform experiments on realistic semi-synthetic data sampled according to MAGIC-NIAB and ANDES causal graphs from the \texttt{bnlearn} package \citep{scutari2010bnlearn} using the \texttt{pgmpy} library \citep{ankan2024pgmpy}.
We use the same experimental setup as described in \Cref{sec:experiments}.
Our results are shown in \Cref{fig:semi_synthetic}.

The MAGIC-NIAB network is causal graph with 44 nodes and 66 edges that generates linear Gaussian data, hence we use Fisher-Z CI tests for all algorithms.
Results for MAGIC-NIAB are shown in the first column, where more \ac{BK} improves all metrics, especially the number of CI tests and computation time.
For intervention distance, a difference between no \ac{BK} and a \ac{BK} rate of 0.2 leads to a much smaller variance.

The ECOLI70 is another causal graph generating linear Gaussian data over 46 variables and 70 edges.
Our results in the second column of \Cref{fig:semi_synthetic} show that more \ac{BK} improves all metrics, especially in terms of computational efficiency, where the improvements are exponential (shown in log-scale).

The largest linear Gaussian causal graph is ARTH150 with 107 variables 150 edges.
Without any background knowledge, the original PC and LDECC algorithms (corresponding to PC-BK and LDECC$^+$-BK with \ac{BK} rate of 0.0) took prohibitively long to run.
Utilizing some amount of \ac{BK} makes it possible to apply these algorithms on ARTH150, demonstrating that the utilization of \ac{BK} can be necessary is some scenarios.
Computational costs improve for all other algorithms as well (it is difficult to see for SNAP($\infty$) due to scale).

The ANDES network generates discrete binary data over 223	variables with 338 direct causal relations.
We were not able to run LDECC$^+$-BK due to memory issues.
The results for the ANDES network, shown in the fourth column, are less straightforward.
The computational effort of all methods decrease with more \ac{BK} except for SNAP($\infty$)-BK, which decreases only a tiny amount.
The intervention distance remains negligible throughout rates of \ac{BK}.

\begin{figure*}[ht]
    \centering
    \includegraphics[width=\linewidth]{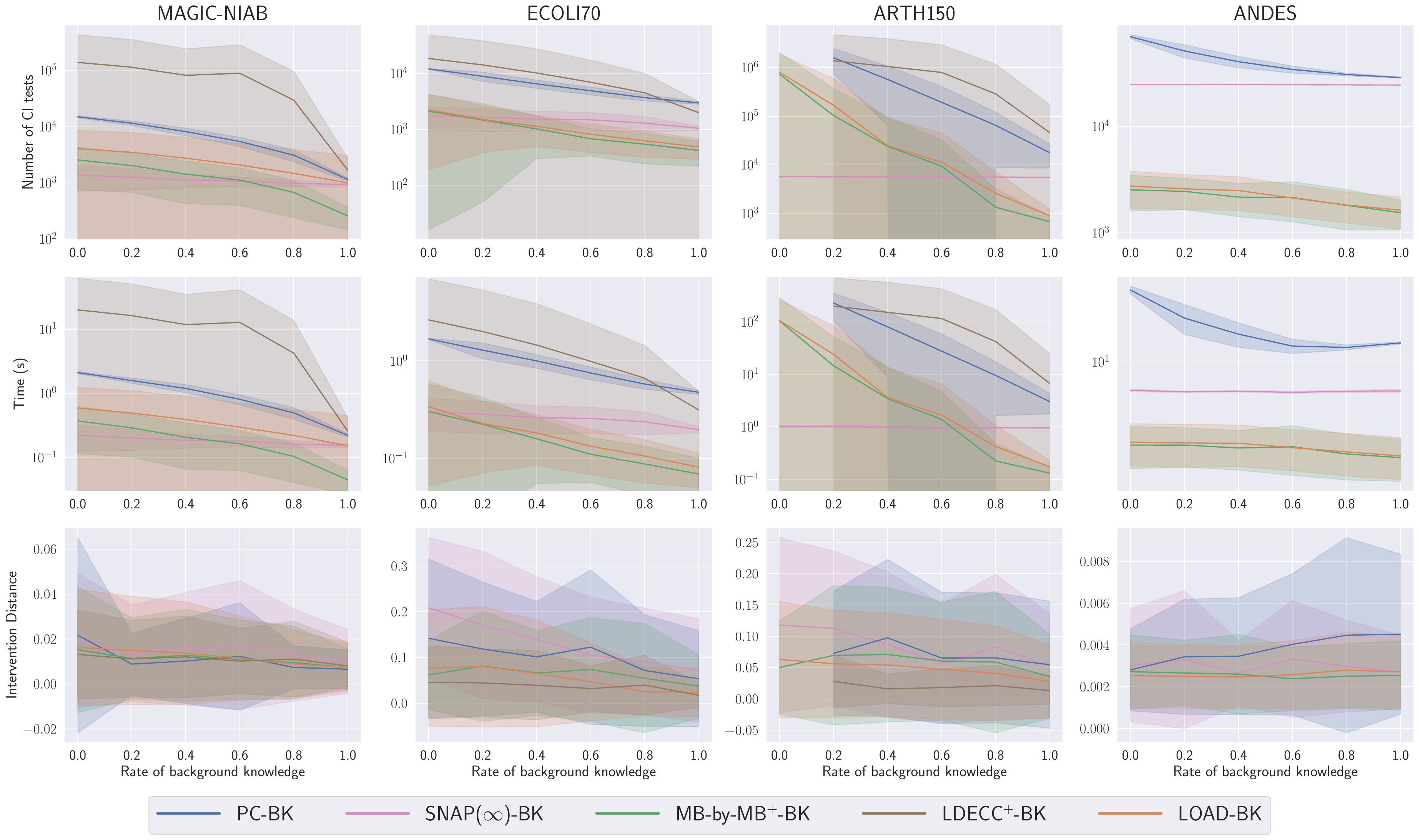}
    \caption{Results on semi-synthetic data according to various causal graphs from the \texttt{bnlearn} package \citep{scutari2010bnlearn}.}
    \label{fig:semi_synthetic}
\end{figure*}

\clearpage
\section*{NeurIPS Paper Checklist}

\begin{enumerate}

\item {\bf Claims}
    \item[] Question: Do the main claims made in the abstract and introduction accurately reflect the paper's contributions and scope?
    \item[] Answer: \answerYes{} 
    \item[] Justification: All claims made in the abstract and the introduction are supported by theoretical results with proofs and empirical evidence including several ablations.

\item {\bf Limitations}
    \item[] Question: Does the paper discuss the limitations of the work performed by the authors?
    \item[] Answer: \answerYes{} 
    \item[] Justification: We state limitations of our work in the conclusions and extensively discuss them in \Cref{sec:limits_and_extens}.

\item {\bf Theory assumptions and proofs}
    \item[] Question: For each theoretical result, does the paper provide the full set of assumptions and a complete (and correct) proof?
    \item[] Answer: \answerYes{} 
    \item[] Justification: We clearly state our assumptions in \Cref{sec:background} and provide proofs for all theorems in \Cref{sec:proofs}, which we always reference in the main paper.

    \item {\bf Experimental result reproducibility}
    \item[] Question: Does the paper fully disclose all the information needed to reproduce the main experimental results of the paper to the extent that it affects the main claims and/or conclusions of the paper (regardless of whether the code and data are provided or not)?
    \item[] Answer: \answerYes{} 
    \item[] Justification: We explicitly state the parameters of our experiments in \Cref{sec:experiments} and \Cref{sec:experimental_details}, and provide several pseudocode for every evaluated algorithm in \Cref{sec:method}, \Cref{sec:mb_discovery} and \Cref{sec:algorithms}. Additionally, we provide the source code to replicate all results in the supplementary material.

\item {\bf Open access to data and code}
    \item[] Question: Does the paper provide open access to the data and code, with sufficient instructions to faithfully reproduce the main experimental results, as described in supplemental material?
    \item[] Answer: \answerYes{} 
    \item[] Justification: Our evaluation data is synthetic in nature and can be reproduced using our source code as well as the cited open source library utilized for semi-synthetic datasets.

\item {\bf Experimental setting/details}
    \item[] Question: Does the paper specify all the training and test details (e.g., data splits, hyperparameters, how they were chosen, type of optimizer) necessary to understand the results?
    \item[] Answer: \answerYes{} 
    \item[] Justification: We specify all learning parameters in \Cref{sec:experiments} as well as in the provided source code.

\item {\bf Experiment statistical significance}
    \item[] Question: Does the paper report error bars suitably and correctly defined or other appropriate information about the statistical significance of the experiments?
    \item[] Answer: \answerYes{} 
    \item[] Justification: All experimental results include standard deviations in both figures and tables, except for a single ablation study in \Cref{sec:bk_types_abl}.

\item {\bf Experiments compute resources}
    \item[] Question: For each experiment, does the paper provide sufficient information on the computer resources (type of compute workers, memory, time of execution) needed to reproduce the experiments?
    \item[] Answer: \answerYes{} 
    \item[] Justification: We provide the exact computer resources used to run the experiments in \Cref{sec:experimental_details}
    
\item {\bf Code of ethics}
    \item[] Question: Does the research conducted in the paper conform, in every respect, with the NeurIPS Code of Ethics \url{https://neurips.cc/public/EthicsGuidelines}?
    \item[] Answer: \answerYes{} 
    \item[] Justification: The paper conform, in every respect, with the NeurIPS Code of Ethics.

\item {\bf Broader impacts}
    \item[] Question: Does the paper discuss both potential positive societal impacts and negative societal impacts of the work performed?
    \item[] Answer: \answerNA{} 
    \item[] Justification: The paper contains foundational research. It is not tied to particular applications and there is no direct path to any negative applications.
    
\item {\bf Safeguards}
    \item[] Question: Does the paper describe safeguards that have been put in place for responsible release of data or models that have a high risk for misuse (e.g., pre-trained language models, image generators, or scraped datasets)?
    \item[] Answer: \answerNA{} 
    \item[] Justification: The paper contains foundational research with synthetically generated data, and poses no such risks.

\item {\bf Licenses for existing assets}
    \item[] Question: Are the creators or original owners of assets (e.g., code, data, models), used in the paper, properly credited and are the license and terms of use explicitly mentioned and properly respected?
    \item[] Answer: \answerYes{} 
    \item[] Justification: We cite each library along with their corresponding license that are used in the experiments in \Cref{sec:experimental_details}.

\item {\bf New assets}
    \item[] Question: Are new assets introduced in the paper well documented and is the documentation provided alongside the assets?
    \item[] Answer: \answerYes{} 
    \item[] Justification: The source code is documented and included in the submission in an anonymized zip file.

\item {\bf Crowdsourcing and research with human subjects}
    \item[] Question: For crowdsourcing experiments and research with human subjects, does the paper include the full text of instructions given to participants and screenshots, if applicable, as well as details about compensation (if any)? 
    \item[] Answer: \answerNA{} 
    \item[] Justification: This paper does not involve crowdsourcing nor research with human subjects.

\item {\bf Institutional review board (IRB) approvals or equivalent for research with human subjects}
    \item[] Question: Does the paper describe potential risks incurred by study participants, whether such risks were disclosed to the subjects, and whether Institutional Review Board (IRB) approvals (or an equivalent approval/review based on the requirements of your country or institution) were obtained?
    \item[] Answer: \answerNA{} 
    \item[] Justification: This paper does not involve crowdsourcing nor research with human subjects.

\item {\bf Declaration of LLM usage}
    \item[] Question: Does the paper describe the usage of LLMs if it is an important, original, or non-standard component of the core methods in this research? Note that if the LLM is used only for writing, editing, or formatting purposes and does \emph{not} impact the core methodology, scientific rigor, or originality of the research, declaration is not required.
    \item[] Answer: \answerNA{} 
    \item[] Justification: The core method development in this research does not involve LLMs as any important, original, or non-standard components.

\end{enumerate}

\end{document}